\documentclass{article}

\usepackage{microtype}
\usepackage{graphicx}
\usepackage{subfigure}
\usepackage{booktabs}

\usepackage{hyperref}

\usepackage[accepted]{icml2024}

\usepackage{amsmath}
\usepackage{amssymb}
\usepackage{mathtools}
\usepackage{amsthm}
\usepackage{bbm}
\usepackage{soul}
\usepackage{multirow}

\usepackage[capitalize,noabbrev]{cleveref}
\usepackage{threeparttable}

\usepackage{nicefrac}
\usepackage{thm-restate}
\usepackage{tikz}

\newtheorem{innercustomthm}{Theorem}
\newenvironment{customthm}[1]
  {\renewcommand\theinnercustomthm{#1}\innercustomthm}
  {\endinnercustomthm}
  
\makeatletter
\let\save@mathaccent\mathaccent
\newcommand*\if@single[3]{%
  \setbox0\hbox{${\mathaccent"0362{#1}}^H$}%
  \setbox2\hbox{${\mathaccent"0362{\kern0pt#1}}^H$}%
  \ifdim\ht0=\ht2 #3\else #2\fi
  }
\newcommand*\rel@kern[1]{\kern#1\dimexpr\macc@kerna}
\newcommand*\widebar[1]{\@ifnextchar^{{\wide@bar{#1}{0}}}{\wide@bar{#1}{1}}}
\newcommand*\wide@bar[2]{\if@single{#1}{\wide@bar@{#1}{#2}{1}}{\wide@bar@{#1}{#2}{2}}}
\newcommand*\wide@bar@[3]{%
  \begingroup
  \def\mathaccent##1##2{%
    \let\mathaccent\save@mathaccent
    \if#32 \let\macc@nucleus\first@char \fi
    \setbox\z@\hbox{$\macc@style{\macc@nucleus}_{}$}%
    \setbox\tw@\hbox{$\macc@style{\macc@nucleus}{}_{}$}%
    \dimen@\wd\tw@
    \advance\dimen@-\wd\z@
    \divide\dimen@ 3
    \@tempdima\wd\tw@
    \advance\@tempdima-\scriptspace
    \divide\@tempdima 10
    \advance\dimen@-\@tempdima
    \ifdim\dimen@>\z@ \dimen@0pt\fi
    \rel@kern{0.6}\kern-\dimen@
    \if#31
      \overline{\rel@kern{-0.6}\kern\dimen@\macc@nucleus\rel@kern{0.4}\kern\dimen@}%
      \advance\dimen@0.4\dimexpr\macc@kerna
      \let\final@kern#2%
      \ifdim\dimen@<\z@ \let\final@kern1\fi
      \if\final@kern1 \kern-\dimen@\fi
    \else
      \overline{\rel@kern{-0.6}\kern\dimen@#1}%
    \fi
  }%
  \macc@depth\@ne
  \let\math@bgroup\@empty \let\math@egroup\macc@set@skewchar
  \mathsurround\z@ \frozen@everymath{\mathgroup\macc@group\relax}%
  \macc@set@skewchar\relax
  \let\mathaccentV\macc@nested@a
  \if#31
    \macc@nested@a\relax111{#1}%
  \else
    \def\gobble@till@marker##1\endmarker{}%
    \futurelet\first@char\gobble@till@marker#1\endmarker
    \ifcat\noexpand\first@char A\else
      \def\first@char{}%
    \fi
    \macc@nested@a\relax111{\first@char}%
  \fi
  \endgroup
}
\makeatother

\theoremstyle{plain}
\newtheorem{theorem}{Theorem}[section]

\newtheorem{lemma}[theorem]{Lemma}
\newtheorem{corollary}[theorem]{Corollary}
\theoremstyle{definition}
\newtheorem{definition}[theorem]{Definition}
\newtheorem{assumption}[theorem]{Assumption}
\theoremstyle{remark}

\newcommand{\calig}[1]{\mathcal{#1}}
\newcommand{\bb}[1]{\mathbb{#1}}
\newcommand{\E}{\bb{E}}
\newcommand{\A}{\calig{A}}
\newcommand{\B}{\calig{B}}
\newcommand{\C}{\calig{C}}
\newcommand{\D}{\calig{D}}
\newcommand{\Ecal}{\calig{E}}
\newcommand{\F}{\calig{F}}
\newcommand{\G}{\calig{G}}
\newcommand{\K}{\calig{K}}
\newcommand{\M}{\calig{M}}
\newcommand{\V}{\calig{V}}
\newcommand{\T}{\calig{T}}
\renewcommand{\O}{\calig{O}}
\newcommand{\Otilde}{\tilde{\calig{O}}}
\newcommand{\brac}[1]{\left(#1\right)}
\newcommand{\sbrac}[1]{\left[#1\right]}
\newcommand{\cbrac}[1]{\left\{#1\right\}}
\newcommand{\norm}[1]{\left\Vert#1\right\Vert}
\newcommand{\prob}{\mathbb{P}}
\newcommand{\reals}{\ensuremath{\mathbb{R}}}
\DeclarePairedDelimiter\abs{\lvert}{\rvert}
\newcommand{\Jmax}{\ensuremath{J_{\max}}}
\newcommand{\proj}[2]{\Pi_{#1}(#2)}
\DeclarePairedDelimiter\floor{\lfloor}{\rfloor}

\DeclareMathOperator*{\argmin}{arg\,min}

\newcommand{\batchsize}{\ensuremath{N}}

\newcommand{\nablat}{\nabla f(x_t)}
\newcommand*{\QEDW}{\null\nobreak\hfill\ensuremath{\square}}

\newcommand{\methodName}{DynaBRO}
\newcommand{\badrounds}{\ensuremath{\tau_{d}}}
\newcommand{\goodrounds}{\ensuremath{\tau_{s}}}

\newcommand{\blue}[1]{\textcolor{blue}{#1}}

\usepackage[textsize=tiny]{todonotes}

\icmltitlerunning{Dynamic Byzantine-Robust Learning}

\begin{document}

\twocolumn[
\icmltitle{Dynamic Byzantine-Robust Learning: Adapting to Switching Byzantine Workers}



\icmlsetsymbol{equal}{*}

\begin{icmlauthorlist}
\icmlauthor{Ron Dorfman}{tech}
\icmlauthor{Naseem Yehya}{tech}
\icmlauthor{Kfir Y. Levy}{tech}
\end{icmlauthorlist}

\icmlaffiliation{tech}{Department of Electrical and Computer Engineering, Technion, Haifa, Israel}
\icmlcorrespondingauthor{Ron Dorfman}{rdorfman@campus.technion.ac.il}

\icmlkeywords{Machine Learning, ICML}

\vskip 0.3in
]



\printAffiliationsAndNotice{}  

\begin{abstract}
Byzantine-robust learning has emerged as a prominent fault-tolerant distributed machine learning framework. However, most techniques focus on the \emph{static} setting, wherein the identity of Byzantine workers remains unchanged throughout the learning process. This assumption fails to capture real-world \emph{dynamic} Byzantine behaviors, which may include intermittent malfunctions or targeted, time-limited attacks. Addressing this limitation, we propose \methodName{} -- a new method capable of withstanding any sub-linear number of identity changes across rounds. Specifically, when the number of such changes is $\O(\sqrt{T})$ (where $T$ is the total number of training rounds), \methodName{} nearly matches the state-of-the-art asymptotic convergence rate of the static setting. Our method utilizes a multi-level Monte Carlo (MLMC) gradient estimation technique applied at the server to robustly aggregated worker updates. By additionally leveraging an adaptive learning rate, we circumvent the need for prior knowledge of the fraction of Byzantine workers.
\end{abstract}
\section{Introduction}\label{sec:intro}
Recently, there has been an increasing interest in large-scale distributed machine learning (ML), where multiple machines (i.e., \emph{workers}) collaboratively aim at minimizing some global objective under the coordination of a central \emph{server}~\cite{verbraeken2020survey,kairouz2021advances}. This approach, leveraging the collective power of multiple computation nodes, holds the promise of significantly reducing training times for complex ML models, such as large language models~\cite{brown2020language}. However, the growing reliance on distributed ML systems exposes them to potential errors, malfunctions, and even adversarial attacks. These \emph{Byzantine} faults pose a significant risk to the integrity of the learning process and could lead to reduced reliability and accuracy in the predictions of the resulting models~\cite{lamport2019byzantine,guerraoui2023byzantine}.

Due to its significance in distributed ML, Byzantine fault-tolerance has attracted considerable interest. Indeed, many prior works have focused on Byzantine-robust learning, aiming to ensure effective learning in the presence of Byzantine machines. These efforts have led to the development of algorithms capable of enduring Byzantine attacks~\cite{feng2014distributed,su2016fault,blanchard2017machine,chen2017distributed,alistarh2018byzantine,guerraoui2018hidden,yin2018byzantine,allen2020byzantine,karimireddy2021learning,karimireddy2022byzantinerobust,farhadkhani2022byzantine,allouah2023fixing}.

While the existing body of research has significantly advanced the understanding of Byzantine-robust learning, a notable gap persists in the treatment of the problem. The vast majority of research in this area has focused on the static setting, wherein the identity of Byzantine workers remains fixed throughout the entire learning process. Nevertheless, real-world distributed learning systems may often encounter \emph{dynamic} Byzantine behaviors, where machines exhibit erratic and unpredictable fault patterns. \mbox{Despite the} importance of these scenarios, the investigation of robustness against such dynamic behaviors remains underexplored.

In federated learning, for instance, the concept of partial participation inherently introduces a dynamic aspect~\cite{bonawitz2019towards,kairouz2021advances}. Workers typically join and leave the training process, leading to a scenario where a Byzantine worker might be present in one round and absent in the next. In fact, the same node might switch between honest and Byzantine behaviors; such fluctuations could stem from strategic manipulations by an adversary seeking to evade detection, or due to software updates that inadvertently trigger or resolve certain security vulnerabilities. Another domain includes volunteer computing paradigms~\cite{kijsipongse2018hybrid,ryabinin2020towards,atre2021distributed}, characterized by a large pool of less reliable workers and regular occurrences of node failures. Prolonged training times of complex ML models in these settings often result in intermittent node failures, typically due to hardware issues, network instabilities, or maintenance activities. These domains, where dynamic Byzantine behaviors are prevalent, pose challenges unaddressed by the static approach.


Previous research has established that convergence is unattainable
when Byzantine workers change their identities in each round~\cite{karimireddy2021learning}.\footnote{The lower bound of \citet{karimireddy2021learning} applies when the number of Byzantines is logarithmic in the number of rounds.} However, it remains unclear how a \textbf{limited} number of identity-switching rounds affects the convergence rate. This work addresses this challenge by introducing a new method we term \methodName{}. 
We establish its convergence, noting that it maintains the asymptotic convergence rate of the static setting as long as the number of rounds featuring Byzantine behavior alterations does not exceed $\O(\sqrt{T})$, where $T$ is the total number of training rounds. Beyond this threshold, the convergence rates begin to degrade linearly with the increase in the number of such rounds.

The key ingredient of our approach is a multi-level Monte Carlo (MLMC) gradient estimation technique~\cite{dorfman2022adapting,beznosikov2023first}, serving as a bias-reduction tool. In \Cref{sec:warmup-static}, we show that combining it with a large class of aggregation rules~\cite{allouah2023fixing} mitigates the bias introduced by Byzantine workers in the static setting, an analysis of independent interest. 
Transitioning to the dynamic setting in \Cref{sec:dynamic}, we refine the MLMC estimator with an added fail-safe filter to address its inherent susceptibility to dynamic Byzantine changes that could introduce a significant bias. This modification is crucial due to the estimator’s reliance on multiple consecutive samples.
Then, in \Cref{sec:adaptivity}, we shift our focus to optimality and adaptivity. The introduction of a new aggregation rule enables us to derive asymptotically optimal convergence bounds for a limited number of identity-switching rounds. Furthermore, by employing an adaptive learning rate, we eliminate the need for prior knowledge of the objective's smoothness and the fraction of Byzantine workers. Finally, in \Cref{sec:experiments}, we explore the practical aspects and benefits of our approach through experiments on image classification tasks with two dynamic identity-switching strategies.
\section{Preliminaries and Related Work}\label{sec:prel_related}
In this section, we formalize our objective, specify our assumptions, and overview relevant related work. 
\subsection{Problem Formulation and Assumptions}
We consider stochastic optimization problems, where the objective is to minimize the expected loss given an unknown data distribution $\D$ and a set of loss functions $\cbrac{x\mapsto F(\cdot; \xi)}$ parameterized by $\xi\sim\D$. Formally, our goal is to solve:
\begin{equation}\label{eq:objective}
    \min_{x\in\K}{f(x)\coloneqq\E_{\xi\sim\D}\sbrac{F(x; \xi)}}\; ,
\end{equation} 
where $\K\subseteq\reals^d$ is the optimization domain. To this end, we assume there are $m$ workers (i.e., computations nodes), each with access to samples from $\D$. This homogeneous setting was previously studied in the context of Byzantine-robust learning~\cite{blanchard2017machine,alistarh2018byzantine,allen2020byzantine}, and it is realistic in collaborative learning scenarios, where workers may have access to the entire dataset~\cite{kijsipongse2018hybrid,diskin2021distributed,gorbunov2022variance}. For each round $t$, we define $\G_t\subseteq\cbrac{1,\ldots,m}$ as the set of honest workers, adhering to the prescribed protocol; the remaining workers are Byzantine and may send arbitrary vectors to the server. Notably, in the static setting, the identity of honest workers is fixed over time, i.e., $\G_t$ is identical across rounds. 
\citet{allen2020byzantine} refer to this dynamic model as \emph{ID relabeling}.
For simplicity, we assume the fraction of Byzantine workers is fixed across rounds and denote it by $\delta\coloneqq 1 - \abs{\G_1}/m<1/2$.

We focus on smooth minimization, namely, we assume the objective $f$ is $L$-smooth, i.e., for every $x,y\in\K$, it holds that $f(y)\leq f(x) + \nabla f(x)^\top(y - x) + \frac{L}{2}\norm{y - x}^2$. Additionally, we will adopt one of the following two assumptions regarding the stochastic gradient noise.
\begin{assumption}[Bounded variance]
\label{assump:bounded-variance}
For every $x\in\K$, 
\[
    \E_{\xi\sim\D}\!\norm{\nabla F(x; \xi) - \nabla f(x)}^2\leq \sigma^2\; .
\]
\end{assumption}
\vspace{-0.6em}
This standard assumption in Byzantine-robust optimization~\cite{karimireddy2021learning,karimireddy2022byzantinerobust,farhadkhani2022byzantine,allouah2023fixing} is utilized in \Cref{sec:warmup-static} to establish intuition in the static setting. For the dynamic case, we require a stronger assumption of deterministically bounded noise~\cite{alistarh2018byzantine,allen2020byzantine}.\footnote{
While our analysis could extend to sub-Gaussian noise (up to logarithmic factors), it becomes more technical. Thus, following \citet{allen2020byzantine}, we opt for bounded noise to ensure simplicity and technical clarity (see footnote 3 therein).}
\begin{assumption}[Bounded noise]
\label{assump:bounded-noise}
For every $x\in\K$, $\xi\sim\D$, 
\begin{equation*}
    \norm{\nabla F(x; \xi) - \nabla f(x)}\leq \V\; .
\end{equation*}
\end{assumption}
\vspace{-0.5em}
We provide convergence bounds for both convex and non-convex problems, presenting only non-convex results for brevity and moving convex analyses to the appendix.\footnote{Since global non-convex minimization is generally intractable, we focus on finding an approximate stationary point.} For the convex case, we assume a bounded domain, as follows:
\begin{assumption}[Bounded domain]
\label{assump:bounded-domain}
The domain $\K$ is convex, compact, and for every $x,y\in\K$: $\norm{x-y}\leq D$.
\end{assumption}

\begin{table*}[t]
    \centering
    \begin{threeparttable}
        \caption{Comparison of history-dependence in different Byzantine-robust techniques.}
        \vspace{0.1in}
        \begin{tabular}{cccc}
            \hline
            \textbf{Method}               &  \textbf{Per-worker Cost}    & \textbf{Window Size}                & \textbf{Window Type} \\ \hline
            ByzantineSGD~\cite{alistarh2018byzantine} & $T$         & $T$                          & Deterministic        \\
            SafeguardSGD~\cite{allen2020byzantine} & $T$        & $\O(T^{5/8})$  & Deterministic        \\
            Worker-momentum~(e.g., \citealp{karimireddy2021learning}) & $T$     & $\Otilde(\sqrt{T})\tnote{*}$ & Deterministic        \\
            \textbf{MLMC (ours)}& $\O(T\log{T})\tnote{\dag}$  & $\O(\log{T})\tnote{\dag}$              & Stochastic           \\ \hline
        \end{tabular}
        \label{tab:history-dependence}
        \begin{tablenotes}
            \item[*] For momentum parameter $\alpha\coloneqq 1-\beta\propto 1/\sqrt{T}$. 
            \item[\dag] In expectation.
        \end{tablenotes}
    \end{threeparttable}
\end{table*}

\textbf{Notations.} Throughout, $\norm{\cdot}$ represents the $L_2$ norm, and for any $n\in\mathbb{N}$, $\sbrac{n}\coloneqq\cbrac{1,\ldots,n}$. We denote any optimal solution of \eqref{eq:objective} by $x^*$ and its corresponding optimal value by $f^*$. Additionally, we define $\F_t$ as the filtration at round $t$, encompassing all prior randomness. Expectation and probability are denoted by $\E$ and $\prob$, respectively, with $\E_{t}$ and $\prob_{t}$ indicating their conditional counterparts given $\F_{t}$. 
Finally, we use standard big-O notation, where $\O(\cdot)$ hides numerical constants and $\Otilde(\cdot)$ omits poly-logarithmic terms.

\subsection{Related Work}
\paragraph{The importance of history for Byzantine-robustness.} It has been well-established that history is critical for achieving Byzantine-robustness in the static setting. \citet{karimireddy2021learning} were the first to formally prove that any memoryless method---where the update for each round depends only on computations performed in that round---may fail to converge. As we detail in \Cref{subsec:motivation}, the importance of history arises from the ability of Byzantine workers to inject in each round bias proportional to the natural noise of honest workers, which is sufficient to circumvent convergence. 
Thus, employing a variance reduction technique, requiring some historical dependence, is crucial to mitigate the injected bias. Consequently, they proposed applying a robust aggregation rule to worker momentums instead of directly to gradients. By setting the momentum parameter as $\alpha\coloneqq 1-\beta\propto 1/\sqrt{T}$, they achieved state-of-the-art convergence guarantees in the presence of Byzantine workers. This momentum method has emerged as the leading approach for Byzantine-robust learning, with its efficacy also demonstrated through empirical evidence~\cite{el2020distributed,farhadkhani2022byzantine,allouah2023fixing}.


Additional history-dependent methods include ByzantineSGD~\cite{alistarh2018byzantine} for convex minimization and SafeguardSGD~\cite{allen2020byzantine} for finding local minima of non-convex functions. Both techniques estimate the set of honest workers by tracking worker statistics (e.g., gradient-iterate products) and applying some median-based filtering. 
While ByzantineSGD relies on the entire history, SafeguardSGD uses information within windows of $T_1\in\O(T^{5/8})$ rounds and incorporates a reset mechanism. Consequently, SafeguardSGD can withstand Byzantine identity changes occurring between these windows. However, there is no reason to assume that ID relabeling occurs only at specific rounds, and our method can handle such changes at any round, without restrictions on when the changes happen.

Although history dependence is crucial in static environments, it poses significant challenges in dynamic settings, in which history may become unreliable. For example, the computations of a Byzantine worker turning honest may still be influenced by prior misbehaviors if the history dependence window encompasses those rounds. Therefore, methods that rely on long historical information are vulnerable to such identity changes. Our approach relies on MLMC gradient estimation technique, which confines the history window size to $\O(\log{T})$ in expectation; refer to \Cref{tab:history-dependence} for a comparison between different history-dependent methods.

\paragraph{Byzantine-robustness and worker sampling.} To date, we are only aware of two works that address the challenges of the dynamic setting~\cite{data2021byzantine,malinovsky2023byzantine}, where the dynamic behavior stems from worker sampling, i.e., different workers may actively participate in each training round. \citet{data2021byzantine} were the first to study this challenging setting, providing convergence results for both strongly- and non-convex objectives. Yet, their upper bounds include a non-vanishing term proportional to the gradient noise, which is sub-optimal in the homogeneous setting. 
\citet{malinovsky2023byzantine} improved upon previous limitations, proposing Byz-VR-MARINA-PP, which can handle some rounds dominated by Byzantine workers. Their work follows a rich body of literature on Byzantine-robust finite-sum minimization~\cite{wu2020federated,zhu2021broadcast,gorbunov2022variance}. Nonetheless, these studies do not provide excess loss (i.e., generalization) guarantees. 

\paragraph{MLMC estimation.} Originally utilized in the context of stochastic differential equations~\cite{giles2015multilevel}, MLMC methods have since been applied in various ML and optimization contexts. These include, for example, distributionally robust optimization~\cite{levy2020large,hu2021bias} and latent variable models~\cite{shi2021multilevel}. \citet{asi2021stochastic} employed an MLMC optimum estimator for calculating proximal points and gradients of the Moreau-Yoshida envelope. Specifically for gradient estimation, MLMC is useful for efficiently generating low-bias gradients in scenarios where obtaining unbiased gradients is either impractical or computationally intensive, e.g., conditional stochastic optimization~\cite{hu2021bias}, stochastic optimization with Markovian noise~\cite{dorfman2022adapting,beznosikov2023first}, and reinforcement learning~\cite{suttle2023beyond}.

\section{Warm-up: Static Robustness with MLMC}\label{sec:warmup-static}
In this section, we study the static setting where Byzantine machine identities are fixed across time. We develop intuition in \Cref{subsec:motivation}, illustrating how Byzantine machines can hinder convergence. In \Cref{subsec:mlmc_estimation}, we introduce our MLMC estimator, present its bias-variance properties, and highlight its role as a bias reduction technique. Finally, in \Cref{subsec:static-robust}, we show how integrating the MLMC estimator with a robust aggregation rule implies Byzantine-robustness. 

\subsection{Motivation}\label{subsec:motivation}
We examine the distributed stochastic gradient descent update rule, defined as $x_{t+1} = x_{t} - \eta \A(g_{t,1},\ldots,g_{t,m})$, with $\eta>0$ as learning rate, $\A:\reals^{d\times{m}}\to\reals^d$ as an aggregation rule, and $g_{t,i}$ representing the message from worker $i\in\sbrac{m}$ at time $t\in\sbrac{T}$. In the Byzantine-free case, $\A$ typically averages the inputs, yet a single Byzantine worker can arbitrarily manipulate this aggregation result~\cite{blanchard2017machine}. 

Ideally, $\A$ would isolate the Byzantine inputs and average the honest inputs, yielding a conditionally unbiased gradient estimate with reduced variance. 
However, Byzantine workers can blend in by aligning their messages closely with the expected noise range of honest gradients. Thus, they can inject a bias that is indistinguishable from the natural noise inherent to honest messages, thereby hindering convergence. Specifically, if honest messages deviate by $\sigma$ from the true gradient, Byzantine workers can introduce a bias of $\O(\sigma\sqrt{\delta})$ at each round, effectively bounding convergence to a similarly scaled neighborhood~\cite{ajalloeian2020convergence}.

Addressing this challenge, a straightforward mitigation strategy might involve all honest workers computing stochastic gradients across large mini-batches. This approach reduces their variance, thereby shrinking the feasible `hiding region' for Byzantine messages. As we establish in \Cref{app:biased-SGD}, theory suggests a mini-batch size of $\Omega(T)$ is required to ensure sufficiently low bias. However, this approach proves to be extremely inefficient, necessitating an excessive total of $\Omega(T^2)$ stochastic gradient evaluations per worker.  

Instead of implicitly reducing Byzantine bias through honest worker variance reduction, we propose a direct bias reduction strategy by employing an MLMC gradient estimation technique post-aggregation, i.e., to the robustly aggregated gradients. In the next section, we introduce the MLMC estimator and establish its favorable bias-variance properties.

\subsection{MLMC Gradient Estimation}\label{subsec:mlmc_estimation}
Our MLMC estimator utilizes any mapping $\M_f:\reals^d\times\mathbb{N}\to\reals^d$ that, given a query vector $x$ and a budget $N$ (in terms of stochastic gradient evaluations), produces a vector whose mean squared error (MSE) in estimating $\nabla f(x)$ is inversely proportional to $N$. Formally, for some $c>0$, we have
\begin{equation}\label{eq:lmgo}
    \E\norm{\M_{f}(x, N) - \nabla f(x)}^2\leq \frac{c^2}{N}, \enskip \forall x\in\reals^d, N\!\in\mathbb{N}\; .
\end{equation}


The MLMC gradient estimator is defined as follows: \\ Sample $J\sim\text{Geom}(\nicefrac{1}{2})$ and, denoting $g^{j}\!\coloneqq\!\M_{f}(x,2^{j})$, set
\begin{equation}\label{eq:mlmc}
    g^{\text{MLMC}} = g^{0} + \begin{cases}
        2^{J}\brac{g^{J} - g^{J-1}}, &2^J\leq T \\
        0, &\text{otherwise}
    \end{cases}\; ,
\end{equation}
The next lemma details its properties (cf. \citealp[Lemma 3.2]{dorfman2022adapting}; and \Cref{app:mlmc_bias_reduction}).
\begin{restatable}{lemma}{mlmcprops}\label{lem:mlmc}
    For $\M_f$ satisfying \Cref{eq:lmgo}, we have that
    \begin{enumerate}
        \item $\textup{Bias}(g^{\textup{MLMC}})\coloneqq\norm{\E g^{\textup{MLMC}} - \nabla f(x)} \leq \sqrt{2c^2/T}$.
        \item $\textup{Var}(g^{\textup{MLMC}})\coloneqq\E\!\norm{g^{\textup{MLMC}} {-} \E g^{\textup{MLMC}}}^2\leq 14c^2\log{T}$.
        \item The expected cost of constructing $g^{\textup{MLMC}}$ is $\O(\log{T})$.
    \end{enumerate}
\end{restatable}
This result implies that we can use $\M_f$ to construct a gradient estimator with: \textbf{(1)} low bias, proportional to $1/\sqrt{T}$; \textbf{(2)} near-constant variance; and \textbf{(3)} only logarithmic cost.

\subsection{Byzantine-Robustness with MLMC Gradients}\label{subsec:static-robust}
Next, we show that combining the MLMC estimator with a robust aggregation rule ensures Byzantine-robustness. We consider $(\delta,\kappa_{\delta})$-robust aggregation rules, a concept recently introduced by \citet{allouah2023fixing}, which unifies preceding formulations like $(\delta_{\max},c)$-\textit{RAgg}~\cite{karimireddy2022byzantinerobust}.
\vspace{-1em}
\begin{definition}[$(\delta, \kappa_{\delta})$-robustness]
\label{def:byz-robustness}
Let $\delta< \nicefrac{1}{2}$ and $\kappa_{\delta}\geq 0$. An aggregation rule $\A$ is \textup{$(\delta, \kappa_{\delta})$-robust} if for any vectors $g_1,...,g_m$, and any set $S\subseteq\sbrac{m}$ of size $(1-\delta)m$, we have
\begin{equation*}
    \norm{\A(g_1,\ldots,g_m) - \overline{g}_S}^2\leq \frac{\kappa_\delta}{\abs{S}}\sum_{i\in S}{\norm{g_i - \overline{g}_S}^2}\; ,
\end{equation*}
where $\overline{g}_S = \frac{1}{\abs{S}}\sum_{i\in S}{g_i}$.
\end{definition}

This definition includes, for example, coordinate-wise median~(CWMed, \citealp{yin2018byzantine}) and geometric median~\cite{pillutla2022robust}. 
The lemma below shows that robustly aggregating mini-batch gradients from honest workers yields a mapping satisfying \Cref{eq:lmgo} with $c^2=2\sigma^2\brac{\kappa_{\delta} + \frac{1}{m}}$.

\begin{restatable}{lemma}{robustagglemmavariance}\label{lem:robust_agg_is_lmgo}
    Consider $x\in\K$ and let $\widebar{g}_1^{\batchsize},\ldots,\widebar{g}_m^{\batchsize}$ be $m$ vectors such that $\forall i\in\G$, $\widebar{g}_i^{\batchsize}$ is a mini-batch gradient estimator based on $\batchsize$ i.i.d samples, i.e., $\widebar{g}_{i}^{\batchsize} = \frac{1}{\batchsize}\sum_{n=1}^{\batchsize}{\nabla F(x; \xi_{i}^{n})}$, where $\xi_{i}^n\overset{\text{i.i.d}}{\sim}\D$. Then, under \Cref{assump:bounded-variance}, any $(\delta, \kappa_{\delta})$-robust aggregation rule $\A$ satisfies
    \[
        \E\norm{\A(\widebar{g}_1^{\batchsize},\ldots,\widebar{g}_m^{\batchsize}) - \nabla f(x)}^2 \leq \frac{2\sigma^2}{N}\brac{\kappa_{\delta} + \frac{1}{m}}\; .
    \]
\end{restatable}
\begin{algorithm}[t]
\caption{Byzantine-Robust Optimization with MLMC}\label{alg:bro_mlmc}
\begin{algorithmic}
    \STATE {\bfseries Input:} Initial iterate $x_1\in\reals^d$, learning rate $\eta$.
    \FOR{$t=1,\ldots,T$}
        \STATE Draw $J_{t} \sim \text{Geom}(1/2)$
        \FOR{$i\in\sbrac{m}$ \textcolor{blue}{in parallel}}
            \FOR{$j\in\cbrac{0,J_t-1,J_t}$}
                \STATE Compute $\widebar{g}_{t,i}^{j}\gets \frac{1}{2^j}\sum_{k=1}^{2^j}{\nabla F(x_t; \xi_{t,i}^{k})}$
            \ENDFOR
        \STATE Send $(\widebar{g}_{t,i}^{0}, \widebar{g}_{t,i}^{J_t-1},\widebar{g}_{t,i}^{J_t})$ if $i\in\G$, else send $(*,*,*)$
        \ENDFOR
        \FOR{$j\in\cbrac{0,J_t-1,J_t}$}
            \STATE $\widehat{g}_t^j\leftarrow\A(\widebar{g}_{t,1}^{j},\ldots,\widebar{g}_{t,m}^{j})$ \algorithmiccomment{robust aggregation}
        \ENDFOR
        \STATE $g_t\gets \widehat{g}_t^{0} + \begin{cases} 2^{J_t}(\widehat{g}_t^{J_{t}} - \widehat{g}_t^{J_{t}-1}),&{J_t} \leq \Jmax{=}\floor{\log{T}}\\
        0, & \text{otherwise}.
        \end{cases}$ 
        \STATE $x_{t+1}\gets\proj{\K}{x_t - \eta g_t}$ 
    \ENDFOR
\end{algorithmic}
\end{algorithm}
Building on this result, we propose \Cref{alg:bro_mlmc}. In each round $t\in\sbrac{T}$, honest workers compute and send mini-batch gradients of sizes $1$, $2^{J_t-1}$, and $2^{J_t}$, with $J_t\sim\text{Geom}(\nicefrac{1}{2})$. Then, the server applies a robust aggregation rule to each group of gradients and subsequently constructs an MLMC estimator as in \Cref{eq:mlmc} to perform an SGD update. The next result confirms the convergence of \Cref{alg:bro_mlmc} for non-convex functions; its full proof and convex analysis are deferred to \Cref{app:static-analysis}, with a proof sketch provided here. 

\begin{restatable}{theorem}{nonconvexstatic}\label{thm:nonconvex-static}
    Under \Cref{assump:bounded-variance}, with a $(\delta, \kappa_{\delta})$-robust aggregator $\A$, consider \Cref{alg:bro_mlmc} with learning rate given by $\eta = \min\cbrac{\frac{\sqrt{\Delta_1}}{4\sigma\sqrt{L\gamma T\log{T}}}, \frac{1}{L}}$, where $\gamma\coloneqq \kappa_{\delta} + \frac{1}{m}$ and $\Delta_1\coloneqq f(x_1) - f^*$. Denoting $\nabla_t\coloneqq\nabla f(x_t)$, it holds that
    \begin{align*}
        \frac{1}{T}\!\sum_{t=1}^{T}{\E\!\norm{\nabla_t}^2} \!\leq\!  16\sqrt{\frac{L\Delta_1\sigma^2 \gamma \log{T}}{T}} + \frac{2\!\brac{L\Delta_1 + 2\sigma^2 \gamma}}{T}.
    \end{align*}
\end{restatable}
When $\kappa_{\delta}\in\O(\delta)$, e.g., for CWMed combined with Nearest Neighbor Mixing~\cite{allouah2023fixing}, the rate in \Cref{thm:nonconvex-static} is consistent with the state-of-the-art convergence guarantees~\cite{karimireddy2021learning,allouah2023fixing}, up to a $\sqrt{\log{T}}$ factor. Moreover, the expected per-worker sample complexity is $\O(T\log{T})$, representing a modest increase of only a $\log{T}$ factor over existing methods. 
Yet, our approach is conceptually different that prior work. Instead of implicitly reducing Byzantine-induced bias through honest worker variance reduction, we use direct bias reduction strategy by constructing MLMC gradients post-aggregation.


\textit{Proof Sketch.} 
For $\eta\leq\frac{1}{L}$, one can show that the SGD updates $x_{t+1}\!=\!x_t - \eta g_t$ satisfy (cf. \Cref{lem:nonconvex-sgd})
\begin{equation}\label{eq:proof_sketch_nonconvex}
    \frac{1}{T}\sum_{t=1}^{T}{\E\!\norm{\nabla_t}^2} \leq \frac{2\Delta_1}{T\eta} + \frac{\eta L}{T}\sum_{t=1}^{T}{\E V_t^2} + \frac{1}{T}\sum_{t=1}^{T}{\E\!\norm{b_t}^2}\; ,
\end{equation}
where $b_t\coloneqq\E g_t - \nabla_t$ and $V_t^2\coloneqq\E\!\norm{g_t - \E g_t}^2$ are the bias and variance of $g_t$, respectively. Combining \Cref{lem:mlmc} with \Cref{lem:robust_agg_is_lmgo} implies that for every $t\in\sbrac{T}$, we have
\begin{equation}\label{eq:bias-var-static-proofsketch}
    \norm{b_t} \leq 2\sigma\sqrt{\frac{\gamma}{T}}, \quad\text{and}\quad V_t^2\leq 28\sigma^2\gamma\log{T}\; .
\end{equation}
Plugging these bounds and tuning $\eta$ concludes the proof. \QEDW
\section{\methodName{}: Dynamic Byzantine-Robustness}\label{sec:dynamic}
In this section, we transition to the dynamic setting, demonstrating how a slightly modified version of \Cref{alg:bro_mlmc} endures a non-trivial number of rounds with ID relabeling (i.e., identity changes). This modification involves adding a fail-safe filter designed to mitigate potential bias from identity switches during the MLMC gradient construction.

Since the MLMC gradient in round $t$ depends solely on $2^{J_t}$ computations per-worker in that round, if identity changes occur only between rounds, e.g., when only communications could be altered, then \Cref{alg:bro_mlmc} can be seamlessly applied in the dynamic setting and our analysis remains valid. However, Byzantine-robustness addresses a broader range of failures, where identity changes may also arise during different gradient computations within a round. For instance, in data poisoning attacks~\cite{huang2020metapoison,schwarzschild2021just}, some gradients for the same worker might be contaminated while others remain clean, depending on the integrity of the sampled data. To address this challenge, we propose a fail-safe filter specifically designed for the MLMC gradient. Initially, we slightly adjust our notation: we denote by $\G_t^k$ the set of honest workers in round $t$, at the $k$-th gradient computation, with $k\!\in\!\sbrac{2^{J_t}}$. In addition, we define $\goodrounds\!\coloneqq\!\big\{t:\G_{t}^{1}=\!\cdots\!=\G_{t}^{2^{J_t}}\big\}$ as the set of \emph{static} rounds, where worker identities remain fixed within the round, and $\badrounds\!\coloneqq\!\sbrac{T}\!\setminus\!\goodrounds$ indicates the set of \emph{dynamic} rounds.

\begin{algorithm*}[t]
\caption{\methodName{} (\textbf{Dyna}mic \textbf{B}yzantine-\textbf{R}obust \textbf{O}ptimization)}\label{alg:method-new}
\begin{algorithmic}
    \STATE {\bfseries Input:} Initial iterate $x_1\in\reals^d$, learning rate sequence $\cbrac{\eta_t}_{t\geq 1}$, universal coefficient $C\coloneqq\sqrt{8\log{\brac{16m^2 T}}}$\;.
    \FOR{$t=1,\ldots,T$}
        \STATE Draw $J_{t} \sim \text{Geom}(1/2)$    
            \FOR{$k=1,\ldots,2^{J_t}$}
                \FOR{$i\in\sbrac{m}$ \blue{in parallel}}
                    \STATE Compute and send $g_{t,i}^{k}\gets \nabla F(x_t; \xi_{t,i}^{k})$ if $i\in\G_t^{k}$, else send $*$
                \ENDFOR
            \ENDFOR
        \FOR{$i\in\sbrac{m}$ \blue{in parallel}}
                \STATE Compute $\widebar{g}_{t,i}^{j}\gets \frac{1}{2^{j}}\sum_{k=1}^{2^{j}}{g_{t,i}^{k}}$ for $j\in\cbrac{0, J_t-1, J_t}$
        \ENDFOR
        \FOR{$j\in\cbrac{0, J_{t}-1, J_t}$}
            \STATE \textcolor{purple}{\textbf{Option 1: }} $\widehat{g}_t^{j}\gets \A(\widebar{g}_{t, 1}^{j}, \ldots, \widebar{g}_{t, m}^{j})$, \enskip $c_{\Ecal}\coloneqq\sqrt{\gamma}$ \algorithmiccomment{$\A$ is $(\delta,\kappa_{\delta})$-robust; $\gamma\coloneqq 2\kappa_{\delta} + \frac{1}{m}$}
            \STATE \textcolor{purple}{\textbf{Option 2: }} $\widehat{g}_t^{j}\gets \text{MFM}(\widebar{g}_{t, 1}^{j}, \ldots, \widebar{g}_{t, m}^{j}; \T^{j}\!\coloneqq\!2C\V/\sqrt{2^{j}})$, \enskip $c_{\Ecal}\coloneqq 6\sqrt{2}$ \algorithmiccomment{See \Cref{alg:mfm}}
        \ENDFOR
        \STATE Define fail-safe event $\Ecal_t = \cbrac{\lVert \widehat{g}_t^{J_t} - \widehat{g}_t^{J_t - 1}\rVert\leq (1+\sqrt{2})c_{\Ecal}C\V/\sqrt{2^{J_t}}}$ 
        \STATE Construct MLMC gradient $g_t\gets \widehat{g}_t^{0} + \begin{cases} 2^{J_t}(\widehat{g}_t^{J_{t}} - \widehat{g}_t^{J_{t}-1}),& \text{if } {J_t} \leq \floor{\log{T}} \text{ and } \Ecal_t \text{ holds} \\
        0, & \text{otherwise}.
        \end{cases}$
        \STATE Update $x_{t+1}\gets\proj{\K}{x_t - \eta_t g_t}$ 
    \ENDFOR
\end{algorithmic}
\end{algorithm*}

\paragraph{MLMC fail-safe filter.} Recall the MLMC gradient formula for $J_t\leq \floor{\log{T}}$, $g_t=\widehat{g}_t^{0} + 2^{J_t} (\widehat{g}_t^{J_t} \!-\! \widehat{g}_t^{J_t-1})$. Its bias-variance analysis in the static case hinges on \Cref{lem:robust_agg_is_lmgo}, which asserts that for each level $j$, the squared distance between $\widehat{g}_t^{j}$ and $\widehat{g}_t^{j-1}$ is proportional to $2^{-j}$. However, this lemma is not applicable in dynamic rounds due to the absence of a consistent set of honest workers. 
Since the MLMC elements of a worker, $\bar{g}_{t,i}^{J_t-1}$ and $\bar{g}_{t,i}^{J_t}$, are averages of gradients computed during the round, even a single instance of Byzantine behavior could compromise them. This might disrupt the expected trend of decreasing distances between aggregated gradients at consecutive levels. To counteract this potential manipulation, we introduce an event $\Ecal_t$ to verify that this distance remains within expected bounds:
\begin{equation}\label{eq:Ecal}
    \Ecal_t = \cbrac{\big\lVert\widehat{g}_t^{J_t} - \widehat{g}_t^{J_t - 1}\big\rVert\leq (1+\sqrt{2})c_{\Ecal}C\V/\sqrt{2^{J_t}}}\; , 
\end{equation}
where $C \coloneqq\sqrt{8 \log{(16m^2 T)}}$. If $\Ecal_t$ holds, we proceed with the standard MLMC construction; otherwise, we revert to a simpler aggregated gradient using a single sample per-worker, similar to when $J_t\!>\!\floor{\log{T}}$ (see \Cref{alg:method-new}). 
The parameter $c_{\Ecal}$ is set to ensure that $\Ecal_t$ holds with high probability in static rounds, where this modification contributes only a lower-order term to the bias. In dynamic rounds, it restricts the bias to a near-constant level.

The convergence of our modified algorithm is established in the following theorem, detailed in \Cref{app:dynamic_general}.

\begin{restatable}{theorem}{nonconvex}\label{thm:nonconvex}
    Under \Cref{assump:bounded-noise}, with a $(\delta, \kappa_{\delta})$-robust aggregator $\A$, consider \Cref{alg:method-new} with \textcolor{purple}{\textbf{Option 1}} and with a fixed learning rate $\eta \coloneqq \min\cbrac{\frac{\sqrt{\Delta_1}}{3C\V\sqrt{L\gamma T\log{T}}}, \frac{1}{L}}$, where $\gamma\coloneqq 2\kappa_{\delta} + \frac{1}{m}$ and $\Delta_1\coloneqq f(x_1) - f^*$. Then, it holds that
    \begin{align*}
        \frac{1}{T}\sum_{t=1}^{T}{\E\!\norm{\nabla f(x_t)}^2}&\leq 12C\V\sqrt{\frac{L\Delta_1\gamma\log{T}}{T}}\\  &\hspace{-1cm} + \frac{2L\Delta_1 + 9C^2\V^2\gamma}{T} + 16C^2\V^2\gamma\frac{\abs{\badrounds}\log{T}}{T}\; ,
    \end{align*}
    where $C \coloneqq\sqrt{8 \log{(16m^2 T)}}$.
\end{restatable}
When there are no identity switches during gradient computations within the same round (i.e., $\abs{\badrounds}=0$), we revert to the rate established in \Cref{thm:nonconvex-static} for the static setting, but with $\sigma$ effectively replaced by $C\V$ due to differing noise assumptions. The bounded noise assumption allows the application of a concentration inequality, ensuring that $\Ecal_t$ occurs with high probability in static rounds. When identity switches are present, convergence is still assured, provided that the number of identity-switching rounds is sub-linear. Specifically, \Cref{thm:nonconvex} leads to the subsequent corollary.



\begin{corollary}\label{cor:nonconvex}
    The asymptotic convergence rate implied by \Cref{thm:nonconvex} is given by
    \[
        \frac{1}{T}\sum_{t=1}^{T}{\E\!\norm{\nabla f(x_t)}^2}\in\Otilde\brac{\V\sqrt{\frac{L\Delta_1\gamma}{T}} + \V^2\gamma\frac{\abs{\badrounds}}{T}}\; .
    \]
    Thus, \Cref{alg:method-new} can withstand $\abs{\badrounds}\in\Otilde(\sqrt{T/\gamma})$ dynamic rounds (omitting dependence on $L, \Delta_1, $ and $\V$), while matching the static convergence rate. Concretely, when $\kappa_{\delta}\in\O(\delta)$, this rate becomes $\Otilde\brac{\V\sqrt{L\Delta_1 (\delta + \nicefrac{1}{m})/T}}$. 
\end{corollary}
\textit{Proof Sketch.} Following the methodology of \Cref{thm:nonconvex-static}, we bound $\frac{1}{T}\sum_{t=1}^{T}{\E\!\norm{\nabla f(x_t)}^2}$ as in \Cref{eq:proof_sketch_nonconvex}. The bias and variance bounds for $g_t$ are detailed in \Cref{lem:bias-var-mlmc-general}. We highlight key differences from the static setting: in static rounds, the bias is similarly bounded with an additional lower-order term, reflecting instances where $\Ecal_t$ does not hold. In dynamic rounds, $\Ecal_t$ limits the expected distance between aggregated gradients at consecutive levels as follows:
\begin{align*}
    \E_{t-1}\!\sbrac{\lVert\widehat{g}_t^{j} - \widehat{g}_t^{j-1}\rVert^2\cdot\mathbbm{1}_{\Ecal_t(j)}} &= \\ &\hspace{-10em}\E_{t-1}\sbrac{\lVert{\widehat{g}_t^{j} - \widehat{g}_t^{j-1}}\rVert^2 | \Ecal_t(j)}\prob(\Ecal_t(j))
    \in \Otilde\brac{\frac{\V^2\gamma}{2^{j}}}\; ,
\end{align*}
where $\Ecal_t(j)$ is defined as in \Cref{eq:Ecal} given $J_t\!=\!j$. Taking expectation w.r.t $J_t$ leads to 
\[
    \text{MSE}_t\coloneqq\E_{t-1}\!\norm{g_t - \nabla f(x_t)}^2\in\Otilde(\V^2\gamma)\; .
\]
Consequently, the bias and variance are bounded as $\Otilde(\V\sqrt{\gamma})$ and $\Otilde(\V^2\gamma)$, respectively.  
Substituting these bounds and setting $\eta$ completes the proof. \QEDW

\paragraph{When worker-momentum fails. } 
As detailed in \Cref{app:momentum_breaks}, the worker-momentum approach may fail when Byzantine workers change identities. By meticulously crafting a momentum-tailored \emph{dynamic attack} that leverages the momentum parameter and exploits its diminishing effect on past gradients, we can induce sufficient drift (i.e., bias) across all momentums simultaneously. This strategy requires only $\O(\sqrt{T})$ rounds of identity switches, which our method can withstand, as shown in \Cref{cor:nonconvex}. It remains an intriguing open question whether the worker-momentum method can be augmented with a mechanism similar to our fail-safe filter to handle identity changes.

\section{Optimality and Adaptivity}\label{sec:adaptivity}
While we have demonstrated convergence for \Cref{alg:method-new} when employing a general $(\delta, \kappa_{\delta})$-robust aggregator, this class of aggregators does not ensure optimal convergence under the bounded noise assumption. This limitation arises because the most effective aggregator features $\kappa_{\delta}\in\O(\delta)$~\cite{allouah2023fixing}, suggesting that the Byzantine-related error term scales with $\sqrt{\delta/T}$ at best. Conversely, the optimal scaling under this noise assumption is proportional to $\delta/\sqrt{T}$~\cite{alistarh2018byzantine,allen2020byzantine}.

In this section, we introduce the \emph{Median-Filtered Mean} (MFM) aggregator. Although it does not meet the $(\delta, \kappa_{\delta})$-robustness criteria, it facilitates near-optimal convergence rates. Additionally, by employing an adaptive learning rate, we eliminate the need for prior knowledge of the smoothness parameter and the fraction of Byzantine workers, which is typically necessary to determine the learning rate.


\paragraph{Median-Filtered Mean. } Consider vectors $g_1,\ldots,g_m$, and a threshold parameter $\T$. Our proposed aggregation method, outlined in \Cref{alg:mfm}, computes the mean of vectors within $\T$-proximity of a representative median vector, $g_{\text{med}}$. This median is chosen to ensure that the majority of messages fall within a $\T/2$ radius of it. If no such median vector exists, i.e., if there is no vector around which at least half of the other messages lie within a $\T/2$ radius, the algorithm defaults to outputting the zero vector. We note that a similar mechanism for gradient estimation was previously used by~\citet{alistarh2018byzantine,allen2020byzantine}.

In \Cref{subapp:mfm_is_not_delta_kappa_robust}, we demonstrate that the MFM aggregator does not satisfy the $(\delta,\kappa_{\delta})$-robustness criteria. Yet, by appropriately setting the threshold parameter $\T$, we achieve a superior asymptotic bound on the distance between aggregated and true gradients in static rounds. This leads to an improved convergence rate, as we detail later. What follows is an informal statement of \Cref{lem:core_lemma_mfm}.
\begin{lemma}[Informal]
    Consider the setting of \Cref{lem:robust_agg_is_lmgo}, replacing \Cref{assump:bounded-variance} with \Cref{assump:bounded-noise}, and let $\widehat{g}^{\batchsize}=\text{MFM}(\widebar{g}_{1}^{\batchsize}, \ldots, \widebar{g}_{m}^{\batchsize}; \T^{\batchsize})$ be the output of \Cref{alg:mfm} with $\T^{\batchsize}\in\tilde{\Theta}(\V/\sqrt{N})$. Then, with high probability, 
    \[
        \norm{\widehat{g}^{\batchsize} - \nabla f(x)}^2 \in\Otilde\brac{\frac{\V^2}{\batchsize}\brac{\delta^2 + \frac{1}{m}}}\; .
    \]
\end{lemma}
\vspace{-0.8em}
Conversely, a $(\delta,\kappa_{\delta})$-robust aggregator yields a worse asymptotic bound of
$\Otilde\!\brac{\frac{\V^2}{\batchsize}(\kappa_{\delta} + \frac{1}{m})}$; see \Cref{lem:core_lemma_general}. 
\begin{algorithm}[t]
\caption{Median-Filtered Mean (MFM)}\label{alg:mfm}
\begin{algorithmic}
    \STATE {\bfseries Input:} Vectors $g_{1},\ldots,g_{m}$, threshold $\T$.
    \STATE $\M\gets\cbrac{i\in\sbrac{m} : \abs{\cbrac{j\in\sbrac{m} : \norm{g_{j} - g_i} \leq \nicefrac{\T}{2}}}> \nicefrac{m}{2}}$
    \IF{$\M\neq\emptyset$}
        \STATE $g_{\text{med}}\gets g_i$ for some $i\in\M$
        \STATE $\widehat{\G}\gets \cbrac{i\in\sbrac{m} : \norm{g_i - g_{\text{med}}}\leq\T}$
        \STATE $\widehat{g}\gets \frac{1}{\abs{\widehat{\G}}}\sum_{i\in\widehat{\G}}{g_i}$
    \ELSE
        \STATE $\widehat{g}\gets 0$
    \ENDIF
    \STATE {\bfseries Return:} $\widehat{g}$
\end{algorithmic}
\end{algorithm}

\paragraph{Adaptive learning rate. } For some $\eta_0>0$, we consider the following version of the AdaGrad-Norm learning rate~\cite{levy2018online,ward2020adagrad,faw2022power}:
\begin{equation}\label{eq:adagrad}
    \eta_t = \frac{\eta_0}{\sqrt{\sum_{s=1}^{t}{\norm{g_s}^2}}}\; .
\end{equation}

It is well-known that AdaGrad-Norm adapts to the gradient's variance and the objective's smoothness, achieving the same asymptotic convergence rates as if these parameters were known in advance~\cite{kavis2022high,attia2023sgd}. Unlike the specifically tuned learning rate in \Cref{thm:nonconvex}, this adaptive learning rate does not incorporate $\V$, $L$, or $\delta$. Yet, our method still requires knowledge of the noise level $\V$ to set the MFM threshold. While it may initially seem that $\delta$ is also necessary for configuring the parameter $c_{\Ecal}$ within the event $\Ecal_t$, we can adjust $c_{\Ecal}$ to be independent of $\delta$ by trivially bounding it. Thus, our method effectively adapts to both the smoothness $L$ and the fraction of Byzantine workers $\delta$.

We now present the convergence result for \Cref{alg:method-new} when employing the MFM aggregator (\textcolor{purple}{\textbf{Option 2}}) alongside the AdaGrad-Norm learning rate. For ease of analysis, we consider problems with bounded objectives, such as neural networks with bounded output activations, e.g., sigmoid or softmax (cf. \Cref{app:dynamic_adaptive} for a detailed analysis).

\begin{restatable}{theorem}{nonconvexadaptive}\label{thm:nonconvex-adaptive}
    Suppose \Cref{assump:bounded-noise} holds and $f$ is bounded by $M$ (i.e., $\max_{x}{\abs{f(x)}}\leq M$). Considering \Cref{alg:method-new} with \textcolor{purple}{\textbf{Option 2}} and the AdaGrad-Norm learning rate as given in \Cref{eq:adagrad}, define $\zeta\coloneqq \frac{2M}{\eta_0} + \eta_0 L$. Then,
    \begin{align*}
        \frac{1}{T}\sum_{t=1}^{T}{\E\norm{\nabla f(x_t)}^2}\in\Otilde\brac{\zeta\V\sqrt{\frac{\tilde{\gamma}}{T}} + \frac{\zeta^2}{T} + \V^2\frac{\abs{\badrounds}}{T}}\; ,
    \end{align*}
    where $\tilde{\gamma}\coloneqq 32\delta^2 + \frac{1}{m}$.
\end{restatable}
In contrast to \Cref{thm:nonconvex}, the third term, associated with the number of dynamic rounds, lacks a factor of $\tilde{\gamma}$ due to our adjustment of the parameter $c_{\Ecal}$ to $\O(1)$ instead of $\O(\sqrt{\tilde{\gamma}})$. Had we utilized the latter, as in \textcolor{purple}{\textbf{Option 1}}, a similar bound could be achieved, but it would necessitate prior knowledge of $\delta$. 
Consequently, the absence of this $\tilde{\gamma}$ factor restricts the number of dynamic rounds we can withstand without compromising convergence to $\abs{\badrounds}\in\Otilde(\sqrt{\tilde{\gamma}T})$, a more restrictive bound compared to \Cref{cor:nonconvex}, as detailed in \Cref{cor:nonconvex-optimal}. With this more restrictive number of dynamic rounds, we achieve a (near-)optimal convergence rate of $\Otilde\brac{\V\sqrt{(\delta^2 + \nicefrac{1}{m})/T}}$. This observation raises a compelling open question: Does adaptivity inherently lead to decreased robustness against Byzantine identity changes?


\textit{Proof Sketch. } Our proof mirrors the approach used in \Cref{thm:nonconvex}, with two key differences. First, using the AdaGrad-Norm learning rate leads to
$G_T^2\coloneqq\sum_{t=1}^{T}{\norm{\nabla f(x_t)}^2}$ exhibiting a `self-boundness' property, in contrast to the non-adaptive, learning rate-dependent bound in \Cref{eq:proof_sketch_nonconvex}. \Cref{lem:nonconvex-adagrad} formalizes this property, indicating:
\begin{equation*}
    \E G_{T}^2 \leq 2\zeta\!\left(\sqrt{\sum_{t=1}^{T}{\E V_t^2}} + \sqrt{2\E S_T^2} + \sqrt{2\E G_T^2}\right) + \E S_T^2\; ,
\end{equation*}
where $S_T^2\coloneqq\sum_{t=1}^{T}{\norm{b_t}^2}$. Secondly, adjusting $c_{\Ecal}$ to $\O(1)$ results in slightly different bias and variance bounds in dynamic rounds, lacking $\tilde{\gamma}$ factors. Specifically, we have:\footnote{Here, we ignore the low-probability event where the MFM aggregator outputs zero. See \Cref{lem:bias-var-mlmc-mfm} for a formal statement.}
\begin{equation*}
    \norm{b_t}\!\in\!\begin{cases}
        \Otilde(\V), \!&t\in\badrounds \\
        \Otilde\big(\V\sqrt{\frac{\tilde{\gamma}}{T}}\big), \!&t\in\goodrounds
    \end{cases}, \hspace{0.3em}V_t^2\!\in\!\begin{cases}
        \Otilde(\V^2), &t\in\badrounds \\
        \Otilde(\V^2\gamma), &t\in\goodrounds
    \end{cases}.
\end{equation*}
Applying these bounds, solving for $\E G_T^2$, and dividing by $T$ yields the final bound. \QEDW

\section{Experiments}\label{sec:experiments}

In this section, we provide numerical experiments to evaluate our approach. In our experiments, we aim to demonstrate: \textbf{(1)} the trade-off between an algorithm's history window size dependence and its susceptibility to dynamic Byzantine changes; and \textbf{(2)} the benefit of our MLMC-based method compared to the prominent worker-momentum approach \cite{karimireddy2021learning,farhadkhani2022byzantine,allouah2023fixing}. To this end, we study two types of simulated dynamic identity-switching strategies:
\begin{enumerate}
    \item \textbf{Periodic($K$): } Once in every $K$ rounds, a new subset (of size $\delta m$) of Byzantine workers is sampled uniformly at random. Between any two such samplings, all worker identities remain fixed. A lower value of $K$ corresponds to a higher rate of identity switches, which implies a stronger dynamic attack.
    \item \textbf{Bernoulli($p, D, \delta_{\max}$): } For each worker, we sample $X\sim\text{Ber}(p)$ independently across workers and iterations. If $X=1$, then the worker becomes Byzantine for a fixed duration of $D$ iterations, up to a maximum of $\delta_{\max}$-fraction of Byzantine workers per iteration.
\end{enumerate}

Note that for the \textbf{Periodic} strategy, the number of Byzantine workers in each iteration remains the same, whereas for the \textbf{Bernoulli} strategy, it changes throughout training.


We study image classification on the MNIST~\cite{lecun1998gradient} and CIFAR-10~\cite{krizhevsky2009learning} datasets using convolutional neural networks (CNNs) with $2$ and $4$ layers, respectively, as in \citet{allouah2023fixing}. Additional training details are deferred to \Cref{app:experiments} for brevity. We benchmark our method against worker-momentum using momentum parameter $\beta\in\cbrac{0.9, 0.99}$, as well as against vanilla SGD, which corresponds to $\beta=0$. In our experiments, we did not use the fail-safe MLMC filter, opting instead for \Cref{alg:bro_mlmc}. Since the MLMC estimator typically requires multiple gradient computations per update, to ensure a fair comparison, we present all results based on an equivalent total number of gradient computations. We ran all experiments with $5$ random seeds and report their mean and standard deviation.

\begin{figure}
    \centering
    \includegraphics[trim={0.4cm 0.1cm 0.1cm 0.4cm},clip,width=\linewidth]{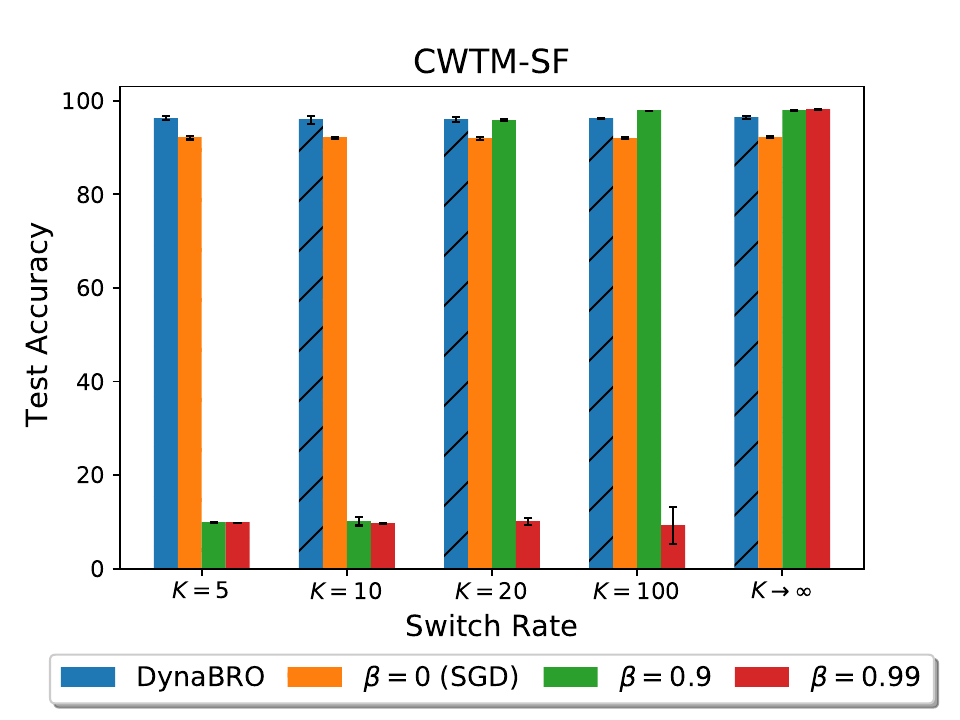}
    \vspace{-1em}
    \caption{Final test accuracy on MNIST under the \textbf{Periodic($K$)} identity-switching strategy for different values of $K$. Byzantine workers implement the SF attack and the server employs CWTM.}
    \vspace{-1em}
\label{fig:mnist_sf_cwtm}
\end{figure}
\paragraph{MNIST. } We first consider the MNIST dataset under the \textbf{Periodic($K$)} switching strategy for different values of $K\in\cbrac{5, 10, 20, 100}$ and $K\rightarrow\infty$, corresponding to the static setting where the set of Byzantine workers is fixed. In this setup, we consider $m=17$ workers, of which $\delta m=8$ are Byzantine. 
In \Cref{fig:mnist_sf_cwtm}, we visualize the final test accuracy when the Byzantine workers implement the sign-flip attack (SF,~\citealp{allen2020byzantine}) and the server uses the coordinate-wise trimmed mean aggregator (CWTM,~\citealp{yin2018byzantine}). As observed, the performance of our approach remains stable across different values of $K$. On the other hand, for worker-momentum, performance declines as $K$ decreases, with higher momentum values experiencing more significant effects. For example, momentum with $\beta=0.99$ fails when $K=100$, whereas with $\beta=0.9$ it performs well even when $K=20$. These results are expected, as momentum with parameter $\beta$ effectively averages the last $\frac{1}{1-\beta}$ iterations, leading us to anticipate this switching rate as its `break-off' point. In \Cref{subapp:mnist_exp}, we provide the results of an analogous experiment with a different pair of attack and aggregator, demonstrating a similar trend.

\paragraph{CIFAR-10. } Next, we consider CIFAR-10 classification with $m=25$ workers under the \textbf{Bernoulli($p, D, \delta_{\max}$)} switching strategy. We investigate three switching configurations: \textbf{(1)} $p=0.01, D=10$; \textbf{(2)} $p=0.01, D=50$; and \textbf{(3)} $p=0.05, D=10$. For all configurations, we restrict the maximum fraction of Byzantine workers in any single iteration to $\delta_{\max}=0.72$ (equivalent to $18$ out of $25$ workers), indicating that the fraction of Byzantine workers in any given iteration may exceed $0.5$. In \Cref{subapp:cifar_exp}, we present similar results for $\delta_{\max}=0.48$. The Byzantine workers employ the inner-product manipulation attack (IPM,~\citealp{xie2020fall}), while the server utilizes CWMed. We compare our method against vanilla SGD and momentum with parameter $\beta=0.9$. In \Cref{fig:cifar10_ipm_cwmed}, we display the test accuracy and histograms of the fraction of Byzantine workers throughout training. Surprisingly, for the first configuration, which exhibits a relatively low number of Byzantine workers in any iteration, momentum with $\beta=0.9$ slightly outperforms our method. However, in the other configurations, which have a larger number of Byzantine workers per iteration and a non-negligible number of iterations with $\delta>0.5$, our method significantly outperforms both SGD and momentum.

\begin{figure}[t]
    \centering
    \includegraphics[trim={0.4cm 0.1cm 0.1cm 0.4cm},clip,width=\linewidth]{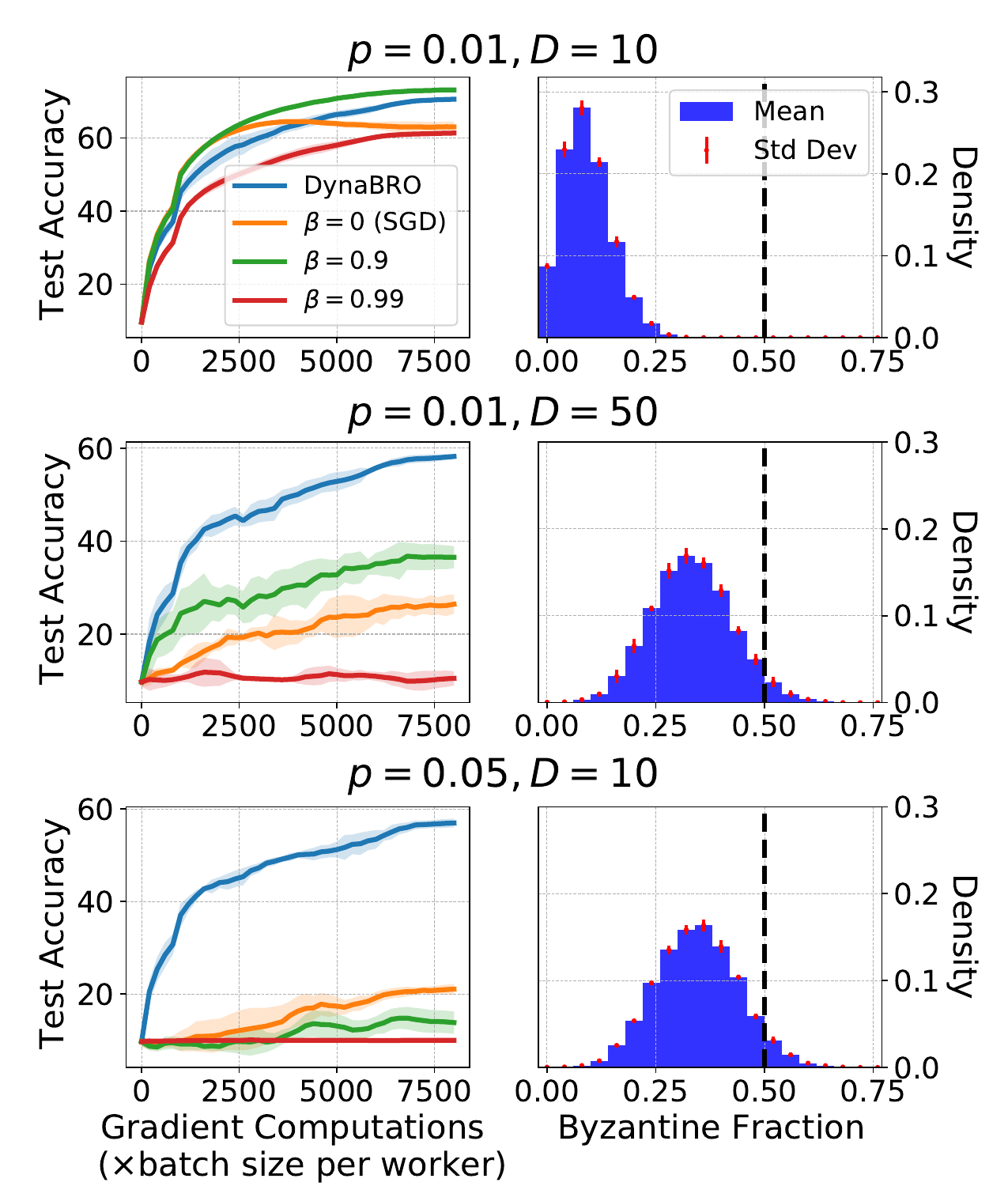}
    \vspace{-1.3em}
    \caption{Test accuracy and histogram of the fraction of Byzantine workers over time on CIFAR-10 under the \textbf{Bernoulli($p, D, \delta_{\max}$)} identity-switching strategy for different values of $p$ and $D$. Byzantine workers employ the IPM attack and the server uses CWMed.}
    \vspace{-1em}
    \label{fig:cifar10_ipm_cwmed}
\end{figure}
\section{Conclusion and Future Work}\label{sec:conclusion}
In this work, we introduced \methodName{}, a novel approach for Byzantine-robust learning in dynamic settings. We tackled the challenge of Byzantine behavior alterations, demonstrating that our method withstands a substantial number of Byzantine identity changes while achieving the asymptotic convergence rate of the static setting. A key innovation is the use of an MLMC gradient estimation technique and its integration with a fail-safe filter, which enhances robustness against dynamic Byzantine strategies. Coupled with an adaptive learning rate, our approach further alleviates the necessity for prior knowledge of the smoothness parameter and the fraction of Byzantine workers. 


Several important avenues for future research emerge from our study. First, we study the homogeneous case, where all workers minimize the same objective. Extending our analysis to heterogeneous datasets is an exciting and challenging direction since a direct application of our approach results in suboptimal bounds for this setting. Another direction includes exploring simultaneous adaptivity to the noise level, the smoothness, and the fraction of Byzantine workers; this presents a complex challenge as we are unaware of any optimal aggregation rule that is agnostic to both the noise level and the fraction of Byzantine workers. 

\section*{Acknowledgments}
This research was partially supported by Israel PBC-VATAT, by the Technion Artificial Intelligent Hub (Tech.AI) and by the Israel Science Foundation (grant No. 447/20).
\section*{Impact Statement}
This paper introduces advancements in Byzantine-robust learning. While inherent societal implications concerning security and reliability in distributed ML are acknowledged, we do not feel they must be specifically highlighted here.

\bibliography{refs}
\bibliographystyle{icml2024}

\newpage
\appendix
\onecolumn

\section*{Appendix Outline}
We provide a concise outline of the appendix structure as follows:

\paragraph{\Cref{app:biased-SGD}. } General analysis of SGD with biased gradients for convex and non-convex functions.
\paragraph{\Cref{app:mlmc_bias_reduction}. } Proof of properties for the MLMC estimator (\Cref{lem:mlmc}).
\paragraph{\Cref{app:static-analysis}. } Analysis of the MLMC approach in the static Byzantine setting for convex and non-convex objectives.
\paragraph{\Cref{app:dynamic_general}. } Analysis of the dynamic Byzantine setting where identities change over time.


\paragraph{\Cref{app:momentum_breaks}. } Analyzing how the worker-momentum method fails in scenarios involving Byzantine ID relabeling.
\paragraph{\Cref{app:mfm_prop}. } Properties of the MFM aggregator, introduced to guarantee optimal convergence rates.
\paragraph{\Cref{app:biased-adagrad}. } Results mirroring \Cref{app:biased-SGD} using AdaGrad-Norm as the learning rate.
\paragraph{\Cref{app:dynamic_adaptive}. } Convergence guarantees for \methodName{} with MFM and AdaGrad-Norm -- improved rates and adaptivity.
\paragraph{\Cref{app:technical_lemmata}. } Technical lemmata.
\paragraph{\Cref{app:experiments}. } Experimental setup and training details.
\section{SGD with Biased Gradients}\label{app:biased-SGD}
Consider the SGD update rule, defined for some initial iterate $x_1\in\reals^d$ and a fixed learning rate $\eta>0$ as,
\begin{equation}\label{eq:sgd}
    x_{t+1} = \proj{\K}{x_t - \eta g_t}\; ,
\end{equation}
where $g_t$ is an estimator of $\nabla f(x_t)$ with bias $b_t \coloneqq \E[g_t - \nabla f(x_t) | x_t]$ and variance $V_t^2 \coloneqq \E[\norm{g_t - \E g_t}^2 | x_t]$.

The following lemmas establish bounds on the optimality gap and sum of squared gradients norm for SGD with (possibly) biased gradients, when applied to convex and non-convex functions, respectively. All our convergence results rely on these lemmas; to be precise, we derive bounds on the bias and variance of the relevant gradient estimator and plug them into our results in a black-box fashion.

\begin{lemma}[Convex SGD]\label{lem:convex-sgd}
    Assume $f$ is convex. Consider \Cref{eq:sgd} with $\eta\leq \frac{1}{2L}$. If the domain $\K$ is bounded with diameter $D$ (\Cref{assump:bounded-domain}), then
    \[
        \E f(\widebar{x}_T) - f(x^*) \leq \frac{D^2}{2T\eta} + \frac{\eta}{T}\sum_{t=1}^{T}{ \E V_t^2} + \frac{D}{T}\sum_{t=1}^{T}{\E\norm{b_t}}\; ,
    \]
    where $\widebar{x}_T\coloneqq \frac{1}{T}\sum_{t=2}^{T+1}{x_t}$.
\end{lemma}
\begin{proof}
    By the convexity of $f$, the gradient inequality implies that
    \begin{equation}\label{eq:grad_ineq_decomp}
        \sum_{t\in\sbrac{T}}{\brac{\E f(x_t) - f(x^*)}} \leq \sum_{t\in\sbrac{T}}{\E[\nablat^\top(x_t - x^*)]} = \sum_{t\in\sbrac{T}}{\E[g_t^{\top}(x_t - x^*)]} - \sum_{t\in\sbrac{T}}{\E[b_t^\top(x_t - x^*)]}\; .
    \end{equation}
    Focusing on the first term in the R.H.S, by applying \Cref{lem:classical_psgd} with $x=x^*$, we have
    \begin{align}\label{eq:bound_on_regret_convex}
        \sum_{t\in\sbrac{T}}{\E[g_t^\top(x_t - x^*)]} &\leq \frac{1}{2\eta}\sum_{t\in\sbrac{T}}{\brac{\E\!\norm{x_t - x^*}^2 - \E\!\norm{x_{t+1} - x^*}^2}} + \sum_{t\in\sbrac{T}}{\brac{\E[g_t^\top(x_t - x_{t+1})] - \frac{1}{2\eta}\E\!\norm{x_t - x_{t+1}}^2}} \nonumber \\ &\leq \frac{D^2}{2\eta} + \sum_{t\in\sbrac{T}}{\brac{\E[g_t^\top(x_t - x_{t+1})] - \frac{1}{2\eta}\E\!\norm{x_t - x_{t+1}}^2}}\; ,
    \end{align}
    where the last inequality follows from a telescoping sum and $\norm{x_1 - x^*}^2\leq D^2$. On the other hand, by the smoothness of $f$, we can bound the L.H.S as follows:
    \begin{equation}\label{eq:lower_bound_on_regret}
        \sum_{t\in\sbrac{T}}{\brac{\E f(x_t) - f(x^*)}} \geq \sum_{t\in\sbrac{T}}{\brac{\E f(x_{t+1}) - f(x^*) - \E[\nablat^\top(x_{t+1} - x_{t})] - \frac{L}{2}\E\!\norm{x_t - x_{t+1}}^2}}\; .
    \end{equation}
    Plugging \Cref{eq:bound_on_regret_convex,eq:lower_bound_on_regret} back into \Cref{eq:grad_ineq_decomp} yields:
    \begin{align*}
        \sum_{t\in\sbrac{T}}{\brac{\E f(x_{t+1}) - f(x^*)}} &\leq \\ &\hspace{-1cm}\frac{D^2}{2\eta} + \sum_{t\in\sbrac{T}}{\brac{\E[(g_t - \nablat)^\top(x_t - x_{t+1})] - \brac{\frac{1}{2\eta} - \frac{L}{2}}\E\!\norm{x_t - x_{t+1}}^2}} - \sum_{t\in\sbrac{T}}{\E[b_t^\top(x_t - x^*)]} \\ &\hspace{-1.5cm}\leq \frac{D^2}{2\eta} + \sum_{t\in\sbrac{T}}{\brac{\E[(g_t - \nablat)^\top(x_t - x_{t+1})] - \frac{1}{4\eta}\E\!\norm{x_t - x_{t+1}}^2}} - \sum_{t\in\sbrac{T}}{\E[b_t^\top(x_t - x^*)]} \\ &\hspace{-1.5cm}\leq \frac{D^2}{2\eta} + \sum_{t\in\sbrac{T}}{\brac{\E[(g_t - \E g_t)^\top(x_t - x_{t+1})] - \frac{1}{4\eta}\E\!\norm{x_t - x_{t+1}}^2}} - \sum_{t\in\sbrac{T}}{\E[b_t^\top(x_{t+1} - x^*)]} \\ &\hspace{-1.5cm}\leq \frac{D^2}{2\eta} + \eta\sum_{t\in\sbrac{T}}{\E V_t^2} -\sum_{t\in\sbrac{T}}{\E[b_t^\top(x_{t+1} - x^*)]}\; ,
    \end{align*}
    where the second inequality uses $\frac{1}{2\eta} - \frac{L}{2}\geq \frac{1}{4\eta}$ and the final inequality uses Young's inequality, $a^\top b - \frac{1}{2}\norm{b}^2\leq \frac{1}{2}\norm{a}^2$. Using Cauchy-Schwarz inequality, we have $-b_t^\top(x_{t+1} - x^*) \leq \norm{b_t}\norm{x_{t+1} - x^*}\leq D\norm{b_t}$; plugging this bound and using Jensen's inequality gives
    \begin{align*}
        \E f(\widebar{x}_T) - f(x^*) &\leq \frac{1}{T}\sum_{t\in\sbrac{T}}{\E f(x_{t+1}) - f(x^*)} \leq \frac{D^2}{2T\eta} + \frac{\eta}{T}\sum_{t\in\sbrac{T}}{\E V_t^2} + \frac{D}{T}\sum_{t\in\sbrac{T}}{\E\!\norm{b_t}}\; .
    \end{align*}
\end{proof}

\begin{lemma}[Non-convex SGD]\label{lem:nonconvex-sgd}
    Consider \Cref{eq:sgd} with $\K=\reals^d$ (i.e., unconstrained) and $\eta\leq \frac{1}{L}$, and let $\Delta_1 \coloneqq f(x_1) - f^*$. Then, 
    \[
        \frac{1}{T}\sum_{t=1}^{T}{\E\norm{\nabla f(x_t)}^2} \leq \frac{2\Delta_1}{T\eta} + \frac{\eta L}{T}\sum_{t=1}^{T}{\E V_t^2} + \frac{1}{T}\sum_{t=1}^{T}{\E\norm{b_t}^2}\; .
    \]
\end{lemma}
\begin{proof}
    We begin our proof by following the methodology presented in Lemma 2 of \cite{ajalloeian2020convergence}. By the smoothness of $f$, we have
    \begin{align*}
        \E_{t-1} f(x_{t+1}) &\leq f(x_t) - \eta\nablat^\top\E_{t-1} g_t + \frac{\eta^2L}{2}\E_{t-1}\norm{g_t}^2 \\ &= f(x_t) - \eta\nablat^\top(\nablat+b_t) + \frac{\eta^2L}{2}\brac{V_t^2 + \E_{t-1}\norm{\E g_t}^2} \\ &= f(x_t) - \eta\norm{\nablat}^2  - \eta\nablat^\top b_t + \frac{\eta^2L}{2}V_t^2 + \frac{\eta^2L}{2}\norm{\nablat + b_t}^2 \\ &\leq f(x_t) - \eta\norm{\nablat}^2  - \eta\nablat^\top b_t + \frac{\eta^2L}{2}V_t^2 + \frac{\eta}{2}\brac{\norm{\nablat}^2 + 2\nablat^\top b_t + \norm{b_t}^2} \\ &= f(x_t) - \frac{\eta}{2}\brac{\norm{\nablat}^2 + \norm{b_t}^2} + \frac{\eta^2L}{2}V_t^2\; ,
    \end{align*}
    where the last inequality follows from $\eta\leq \frac{1}{L}$. Denote: $\Delta_t \coloneqq \E f(x_t) - f^*$. By rearranging terms and taking expectation, we obtain
    \begin{align*}
        \E\norm{\nablat}^2 &\leq \frac{2\brac{\Delta_{t} - \Delta_{t+1}}}{\eta} + \eta L \E V_t^2 + \E\norm{b_t}^2\; .
    \end{align*}
    Summing over $t\in\sbrac{T}$,
    \begin{align*}
        \frac{1}{T}\sum_{t=1}^{T}{\E\norm{\nablat}^2} &\leq \frac{2}{T\eta}\sum_{t=1}^{T}{(\Delta_t - \Delta_{t+1})} + \frac{\eta L}{T}\sum_{t=1}^{T}{\E V_t^2} + \frac{1}{T}\sum_{t=1}^{T}{\E\norm{b_t}^2} \\ &\leq \frac{2\Delta_1}{T \eta} + \frac{\eta L}{T}\sum_{t=1}^{T}{\E V_t^2} + \frac{1}{T}\sum_{t=1}^{T}{\E\norm{b_t}^2}\; ,
    \end{align*}
    which concludes the proof.
\end{proof}

\section{Efficient Bias-Reduction with MLMC Gradient Estimation}\label{app:mlmc_bias_reduction}

In this section, we establish the properties of the MLMC estimator in \Cref{eq:mlmc} through the proof of \Cref{lem:mlmc}.
\mlmcprops*
\begin{proof}
    Our proof follows those presented in Lemma 3.1 of \cite{dorfman2022adapting} and in Proposition 1 of \cite{asi2021stochastic}. Let $\Jmax = \floor{\log{T}}$, and recall that $\E\!\norm{g^{j} - \nabla f(x)}^2\leq\frac{c^2}{2^j}$. By explicitly writing the expectation over $J\sim\text{Geom}(\frac{1}{2})$, we have
    \[
        \E g^{\text{MLMC}} = \E g^{0} + \sum_{j=1}^{\Jmax}{2^{-j}\cdot 2^{j}\E[g^{j} - g^{j-1}]} = \E[g^{\Jmax}]\; ,
    \]
    where the last equality follows from a telescoping sum. Thus, be Jensen's inequality, it holds that
    \begin{align*}
        \norm{\E g^{\text{MLMC}} - \nabla f(x)} = \norm{\E g^{\Jmax} - \nabla f(x)} \leq \sqrt{\E\!\norm{g^{\Jmax} - \nabla f(x)}^2} \leq \sqrt{\frac{c^2}{2^{\Jmax}}} \leq \sqrt{\frac{2c^2}{T}}\; ,
    \end{align*}
    where the last inequality follows from $2^{\Jmax} = 2^{\floor{\log{T}}}\geq T/2$. For the second part, we have
    \begin{align*}
        \E\!\norm{g^{\text{MLMC}} - \E g^{\text{MLMC}}}^2 \leq \E\!\norm{g^{\text{MLMC}} - \nabla f(x)}^2 \leq 2\E\!\norm{g^{\text{MLMC}} - g^{0}}^2 + 2\E\!\norm{g^{0} - \nabla f(x)}^2\; .
    \end{align*}
    Focusing on the first term in the R.H.S, 
    \begin{align*}
        \E\!\norm{g^{\text{MLMC}} - g^{0}}^2 &= \sum_{j=1}^{\Jmax}{2^{-j}\cdot 2^{2j}\E\!\norm{g^{j} - g^{j-1}}^2} \\ &\leq 2\sum_{j=1}^{\Jmax}{2^{j}\brac{\E\!\norm{g^{j} - \nabla f(x)}^2 + \E\!\norm{g^{j-1} - \nabla f(x)}^2}} \\ &\leq 2\sum_{j=1}^{\Jmax}{2^{j}\brac{\frac{c^2}{2^{j}} + \frac{c^2}{2^{j-1}}}} = 6\sum_{j=1}^{\Jmax}{c^2} \leq 6c^2\log{T}\; ,
    \end{align*}
    where the last inequality uses $\Jmax\leq\log{T}$. Using $\E\norm{g^{0} - \nabla f(x)}^2\leq c^2$, we have 
    \[
        \E\!\norm{g^{\text{MLMC}} - \E g^{\text{MLMC}}}^2 \leq 2\cdot 6c^2\log{T} + 2\cdot c^2 \leq 14c^2\log{T}\; .
    \]
    Finally, since we call $\M_f(x; 1)$, and with probability $2^{-j}$ we call $\M_f(x; 2^{j})$ and $\M_f(x; 2^{j-1})$, the expected number of stochastic gradient evaluations is at most $1+\sum_{j=1}^{\Jmax}{2^{-j}\brac{2^j + 2^{j-1}}} = 1+\frac{3}{2}\Jmax \leq 1+\frac{3}{2}\log{T}\in \O(\log{T})$.
\end{proof}
\section{Static Byzantine-Robustness with MLMC Gradient Estimation}\label{app:static-analysis}
In this section, we analyze \Cref{alg:bro_mlmc} in the standard, static setting, where the identity of Byzantine workers remains fixed. We show that distributed SGD with MLMC estimation applied to robustly-aggregated gradients is Byzantine-resilient. 

We begin by establishing \Cref{lem:robust_agg_is_lmgo}, which asserts that utilizing a $(\delta, \kappa_{\delta})$-robust aggregator when honest workers compute stochastic gradients over a mini-batch satisfies \Cref{eq:lmgo} with $c^2 = 2\sigma^2\brac{\kappa_{\delta} + \frac{1}{m}}$.
\robustagglemmavariance*

\begin{proof}
    For ease of notation, denote $\widehat{g}^{\batchsize}\coloneqq \A(\widebar{g}_1^{\batchsize},\ldots,\widebar{g}_m^{\batchsize})$ and $\widebar{g}^{\batchsize}\coloneqq \frac{1}{\abs{\G}}{\sum_{i\in\G}}{\widebar{g}_i^{\batchsize}}$. Since $\E \widebar{g}^{\batchsize} = \nabla f(x)$, and by the $(\delta,\kappa_{\delta})$-robustness of $\A$, we have
    \begin{align*}
        \E\!\norm{\widehat{g}^{\batchsize} - \nabla f(x)}^2 &= \E\!\norm{\widehat{g}^{\batchsize} - \widebar{g}^{\batchsize}}^2 + \E\!\norm{\widebar{g}^{\batchsize} - \nabla f(x)}^2 \\ &\leq \frac{\kappa_{\delta}}{\abs{\G}}\sum_{i\in\G}{\E\!\norm{\widebar{g}_i^{\batchsize} - \widebar{g}^{\batchsize}}^2} + \E\!\norm{\widebar{g}^{\batchsize} - \nabla f(x)}^2 \\ &= \frac{\kappa_{\delta}}{\abs{\G}}{\sum_{i\in\G}}{\E\!\norm{\widebar{g}_i^{\batchsize} - \nabla f(x)}^2} + (\kappa_{\delta} + 1)\E\!\norm{\widebar{g}^{\batchsize} - \nabla f(x)}^2 \\ &\leq \kappa_{\delta}\frac{\sigma^2}{\batchsize} + (\kappa_{\delta} + 1)\frac{\sigma^2}{\abs{\G}\batchsize} \\ &= \frac{\sigma^2}{\batchsize}\brac{\kappa_{\delta} + \frac{\kappa_{\delta} + 1}{\abs{\G}}} \\ &\leq \frac{2\sigma^2}{\batchsize}\brac{\kappa_{\delta} + \frac{1}{m}} \; ,
    \end{align*}
    where the second inequality holds as $\widebar{g}_i^{\batchsize}$ for every $i\in\G$ and $\widebar{g}^{\batchsize}$ are the averages of $\batchsize$ and $\abs{\G}\batchsize$ i.i.d samples, respectively, each with variance bounded by $\sigma^2$. In the last inequality we used $\abs{\G}> m/2$.
\end{proof}

\subsection{Convex Case}
Next, we now establish the convergence of \Cref{alg:bro_mlmc} in the convex setting.
\begin{theorem}\label{thm:convex-static}
    Assume $f$ is convex. Under Assumptions~\ref{assump:bounded-variance} and \ref{assump:bounded-domain}, and with a $(\delta, \kappa_{\delta})$-robust aggregator $\A$, consider \Cref{alg:bro_mlmc} with learning rate 
    \[
        \eta = \min\cbrac{\frac{D}{8\sigma\sqrt{\gamma T\log{T}}}, \frac{1}{2L}}\; ,
    \]
    where $\gamma\coloneqq \kappa_{\delta} + \frac{1}{m}$. Then, for every $T\geq 1$,
    \begin{align*}
        \E f(\widebar{x}_{T}) - f(x^*)&\leq 10D\sigma\sqrt{\frac{\gamma\log{T}}{T}} + \frac{LD^2}{T} \; .
    \end{align*}
\end{theorem}
\begin{proof}
    Employing \Cref{lem:convex-sgd}, we have
    \begin{equation}\label{eq:convex_sgd}
        \E f(\widebar{x}_T) - f(x^*) \leq \frac{D^2}{2T\eta} + \frac{\eta}{T}\sum_{t=1}^{T}{ \E V_t^2} + \frac{D}{T}\sum_{t=1}^{T}{\E\norm{b_t}}\; ,
    \end{equation}
    where $b_t$ and $V_t$ are the bias and variance of $g_t$, respectively. Combining \Cref{lem:mlmc} with \Cref{lem:robust_agg_is_lmgo} implies that 
    \begin{equation}\label{eq:mlmc_props_static}
        \norm{b_t} = \norm{\E g_t - \nablat} \leq \sqrt{\frac{2\cdot 2\sigma^2\gamma}{T}} = 2\sigma\sqrt{\frac{\gamma}{T}}, \quad \text{ and }\quad \E V_t^2 = \E\!\norm{g_t - \E g_t}^2 \leq 28\sigma^2\gamma\log{T}\; .
    \end{equation}
    We note that computing $g_t$ requires $\O(m\log{T})$ stochastic gradient evaluations, in expectation. Plugging these bounds back to \Cref{eq:convex_sgd} gives:
    \begin{align*}
        \E f(\widebar{x}_T) - f(x^*) &\leq \frac{D^2}{2T\eta} + 28\eta\sigma^2\gamma\log{T} + 2D\sigma\sqrt{\frac{\gamma}{T}} \leq \frac{1}{2}\brac{\frac{D^2}{T\eta} + 64\eta\sigma^2\gamma\log{T}} + 2D\sigma\sqrt{\frac{\gamma}{T}}\; .
    \end{align*}
    Since $\eta = \min\cbrac{\frac{D}{8\sigma\sqrt{\gamma T\log{T}}}, \frac{1}{2L}}$, applying \Cref{lem:lr_min_of_2_lrs} with $a= \frac{D^2}{T}, b=64\sigma^2\gamma\log{T}$, and $c=2L$, enables to bound the sum of the first two terms as,
    \[
        \frac{D^2}{T\eta} + 64\eta \sigma^2\gamma\log{T} \leq 16D\sigma\sqrt{\frac{\gamma\log{T}}{T}} + \frac{2LD^2}{T}\; .
    \]
    Plugging this bound back gives:
    \[
        \E f(\widebar{x}_T) - f(x^*) \leq 8D\sigma\sqrt{\frac{\gamma\log{T}}{T}} + \frac{LD^2}{T} + 2D\sigma\sqrt{\frac{\gamma}{T}} \leq 10D\sigma\sqrt{\frac{\gamma\log{T}}{T}} + \frac{LD^2}{T}\; .
    \]
\end{proof}

\subsection{Non-convex Case}
Moving forward, we establish the convergence of \Cref{alg:bro_mlmc} for non-convex functions in \Cref{thm:nonconvex-static}, restated here.
\nonconvexstatic*
\begin{proof}
    We follow the proof of \Cref{thm:convex-static}, substituting \Cref{lem:convex-sgd} with \Cref{lem:nonconvex-sgd}, which implies that
    \begin{equation}
        \frac{1}{T}\sum_{t=1}^{T}{\E\norm{\nabla f(x_t)}^2} \leq \frac{2\Delta_1}{T\eta} + \frac{\eta L}{T}\sum_{t=1}^{T}{\E V_t^2} + \frac{1}{T}\sum_{t=1}^{T}{\E\norm{b_t}^2}\; .
    \end{equation}
    Plugging the bounds in \Cref{eq:mlmc_props_static}, we then obtain
    \begin{align*}
        \frac{1}{T}\sum_{t=1}^{T}{\E\norm{\nabla f(x_t)}^2} \leq \frac{2\Delta_1}{T\eta} + 28\eta L\sigma^2\gamma\log{T} + \frac{4\sigma^2\gamma}{T} \leq 2\brac{\frac{\Delta_1}{T\eta} + 16\eta L\sigma^2\gamma\log{T}} + \frac{4\sigma^2\gamma}{T}\; .
    \end{align*}
    Since $\eta = \min\cbrac{\frac{\sqrt{\Delta_1}}{4\sigma\sqrt{L\gamma T\log{T}}}, \frac{1}{L}}$, applying \Cref{lem:lr_min_of_2_lrs} with $a=\frac{\Delta_1}{T}, b=16L\sigma^2\gamma\log{T}$, and $c=L$, allows us to bound the sum of the first two terms as follows
    \[
        \frac{\Delta_1}{T\eta} + 16\eta L\sigma^2\gamma\log{T} \leq 8\sqrt{\frac{L\Delta_1\sigma^2\gamma\log{T}}{T}} + \frac{L\Delta_1}{T}\; .
    \]
    Substituting this bound back yields:
    \begin{equation*}
        \frac{1}{T}\sum_{t=1}^{T}{\E\norm{\nabla f(x_t)}^2} \leq 16\sqrt{\frac{L\Delta_1\sigma^2\gamma\log{T}}{T}} + \frac{2\brac{L\Delta_1 + 2\sigma^2\gamma}}{T}\; .
    \end{equation*}
\end{proof}
\section{Dynamic Byzantine-Robustness with General $(\delta,\kappa)$-robust Aggregator}\label{app:dynamic_general}
In this section, we analyze \methodName{} (\Cref{alg:method-new}) with \textcolor{purple}{\textbf{Option 1}}, which utilizes a general $(\delta,\kappa)$-robust aggregator. 


Recall that \Cref{alg:method-new} with \textcolor{purple}{\textbf{Option 1}} performs the following update rule for every $t\in\sbrac{T}$:
\begin{align}
    & J_t\sim\text{Geom}(\nicefrac{1}{2}) \nonumber \\
    & \widehat{g}_t^{j}\gets \A(\widebar{g}_{t,1}^{j},\ldots,\widebar{g}_{t,m}^{j}), \quad \text{ where } \widebar{g}_{t,i}^{j} = \frac{1}{2^{j}}\sum_{k=1}^{2^{j}}{\nabla F(x_t; \xi_{t,i}^{k})} \text{ for every } i\in\G_t \text{ if } t\notin\badrounds \label{eq:robust_aggregated_grad} \\
    & g_t \gets  \widehat{g}_t^{0} + \begin{cases}
        2^J_t\brac{\widehat{g}_t^{J_t} - \widehat{g}_t^{J_t-1}}, &\text{if } J_t\leq\Jmax\coloneqq\floor{\log{T}} \text{ and } \Ecal_t(J_t) \text{ holds} \\
        0, &\text{otherwise}
    \end{cases} \label{eq:mlmc_agg_general} \\
    & x_{t+1} \gets \proj{\K}{x_t - \eta_t g_t }\; , \nonumber
\end{align}
where the associated event $\Ecal_t(J_t)$ in this scenario is defined as,
\begin{equation}\label{eq:event_E_option1}
    \Ecal_t(J_t)\coloneqq \cbrac{\lVert \widehat{g}_t^{J_t} - \widehat{g}_t^{J_t-1}\rVert\leq (1 + \sqrt{2})\frac{c_{\Ecal}C\V}{\sqrt{2^{J_t}}}}, \quad c_{\Ecal} \coloneqq \sqrt{\gamma}, \quad C\coloneqq\sqrt{8\log\brac{16m^2 T}}, \quad \gamma \coloneqq 2\kappa_{\delta} + \frac{1}{m}\; .
\end{equation}

We start by establishing a deterministic and a high probability bound on the distance between the true gradient and the robustly-aggregated stochastic gradients, when honest workers compute gradients over a mini-batch. By combining these bounds and adjusting the probability parameter, we provide an upper bound on the expected squared distance, i.e., MSE, as presented in \Cref{cor:mse_bound_general}.
\begin{lemma}\label{lem:core_lemma_general}
    Consider the setting in \Cref{lem:robust_agg_is_lmgo}, i.e.,
    let $x\in\K$ and $\widebar{g}_1^{\batchsize},\ldots,\widebar{g}_m^{\batchsize}$ be $m$ vectors such that for each $i\in\G$, $\widebar{g}_i^{\batchsize}$ is a mini-batch gradient estimator based on $\batchsize\in\mathbb{N}$ i.i.d samples. Then, under \Cref{assump:bounded-noise}, any $(\delta, \kappa_{\delta})$-robust aggregation rule $\A$ satisfies,
    \begin{enumerate}
        \item $\norm{\A(\widebar{g}_1^{\batchsize},\ldots,\widebar{g}_m^{\batchsize}) - \nabla f(x)}^2\leq 2\brac{4\kappa_{\delta}+1}\V^2$.
        \item With probability at least $1-p$,
        \[
            \norm{\A(\widebar{g}_1^{\batchsize},\ldots,\widebar{g}_m^{\batchsize}) - \nabla f(x)}^2
            \leq C_p^2\frac{\V^2\gamma}{\batchsize}
        \]
        where $C_p^2 = 8\log{\brac{{4m}/{p}}}$.
    \end{enumerate}
\end{lemma}


\begin{proof}
    Denote the aggregated gradient and the empirical average of honest workers by $\widehat{g}^\batchsize\!\coloneqq\!\A(\widebar{g}_1^{\batchsize},\ldots,\widebar{g}_m^{\batchsize})$ and $\widebar{g}^{\batchsize}\!\coloneqq\!\frac{1}{\abs{\G}}\sum_{i\in\G}{\widebar{g}_{i}^{\batchsize}}$, respectively. Thus,
    \begin{align*}
        \norm{\widehat{g}^{\batchsize} - \nabla f(x)}^2 &\leq 2\norm{\widehat{g}^{\batchsize} - \widebar{g}^{\batchsize}}^2 + 2\norm{\widebar{g}^{\batchsize} - \nabla f(x)}^2 \\ &\leq \frac{2\kappa_{\delta}}{\abs{\G}}\sum_{i\in\G}{\norm{\widebar{g}_{i}^{\batchsize} - \widebar{g}^{\batchsize}}^2} + 2\norm{\widebar{g}^{\batchsize} - \nabla f(x)}^2\\ &\leq 
        \frac{4\kappa_{\delta}}{\abs{\G}}\sum_{i\in\G}{\lVert{\widebar{g}_{i}^{\batchsize} - \nabla\rVert}^2} + \brac{4\kappa_{\delta} + 2}\lVert{\widebar{g}^{\batchsize} - \nabla f(x)\rVert}^2
    \end{align*}
    where we used $\norm{a+b}^2\leq 2\norm{a}^2 + 2\norm{b}^2$ and the $(\delta,\kappa_{\delta})$-robustness of $\A$. Since it trivially holds that $\norm{\widebar{g}_i^{\batchsize} - \nabla f(x)}\leq\V$ for every $i\in\G$ and $\norm{\widebar{g}^{\batchsize} - \nabla f(x)}\leq\V$, we have
    \begin{align*}
        \norm{\widehat{g}^{\batchsize} - \nabla f(x)}^2 \leq 2(4\kappa_{\delta} + 1)\V^2\; ,
    \end{align*}
    which establishes the first part. For the second part, we employ the concentration argument presented in \Cref{lem:concentration}. With probability at least $1-\tilde{p}$, it holds that 
    \begin{equation}\label{eq:concentration_g_bar}
        \norm{\widebar{g}^{\batchsize} - \nabla f(x)} \leq \V\sqrt{\frac{2\log{\brac{2/\tilde{p}}}}{\batchsize\abs{\G}}} \leq 2\V\sqrt{\frac{\log{\brac{2/\tilde{p}}}}{\batchsize m}}\; ,
    \end{equation}
    as $\abs{\G} > \nicefrac{m}{2}$. Additionally, for each $i\in\G$ separately, we have with probability at least $1-\tilde{p}$ that
    \begin{equation}\label{eq:concentration_g_i_bar}
        \norm{\widebar{g}_i^{\batchsize} - \nabla f(x)} \leq \V\sqrt{\frac{2\log{\brac{2/\tilde{p}}}}{\batchsize}}\; .
    \end{equation}
    Hence, with probability at least $1-(1+\abs{\G})\tilde{p}$, the union bound ensures that \Cref{eq:concentration_g_bar,eq:concentration_g_i_bar} hold simultaneously for every $i\in\G$. Since $1-(1+\abs{\G})\tilde{p}\geq 1-2\abs{\G}\tilde{p}\geq 1 - 2m\tilde{p}$, with probability at least $1-p$ it holds that
    \begin{align*}
        \norm{\widehat{g}^{\batchsize} - \nabla f(x)}^2 \leq 4\kappa_{\delta} \brac{\V\sqrt{\frac{2\log{\brac{4m/p}}}{\batchsize}}}^2 + (4\kappa_{\delta} + 2)\brac{2\V\sqrt{\frac{\log{\brac{4m/p}}}{\batchsize m}}}^2 &=8\log{\brac{\frac{4m}{p}}}\frac{\V^2}{\batchsize}\brac{\kappa_{\delta} + \frac{2\kappa_{\delta} + 1}{m}} \\ &\leq 8\log{\brac{\frac{4m}{p}}}\frac{\V^2}{\batchsize}\brac{2\kappa_{\delta} + \frac{1}{m}}\; ,
    \end{align*}
    which concludes the proof, assuming $m\geq 2$.
\end{proof}

\begin{corollary}[MSE of Aggregated Gradients]\label{cor:mse_bound_general}
    Let $\widebar{g}_{1}^{\batchsize}, \ldots, \widebar{g}_{m}^{\batchsize}$ be as defined in \Cref{lem:core_lemma_general}, with $N\leq T$. Under \Cref{assump:bounded-noise}, any $(\delta, \kappa_{\delta})$-robust aggregator $\A$ satisfies,
    \[
        \E\!\norm{\A(\widebar{g}_1^{\batchsize},\ldots,\widebar{g}_m^{\batchsize}) - \nabla f(x)}^2 \leq \frac{2C^2 \V^2 \gamma}{\batchsize} \; ,
    \]
    where $C=\sqrt{8\log\brac{16m^2 T}}$ and $\gamma = 2\kappa_{\delta} + \frac{1}{m}$ as in \Cref{eq:event_E_option1}.
\end{corollary}
\begin{proof}
    By choosing $p=\frac{1}{4mT}$, item $2$ of \Cref{lem:core_lemma_general} implies that with probability at least $1-\frac{1}{4mT}$,
    \[
        \norm{\A(\widebar{g}_1^{\batchsize},\ldots,\widebar{g}_m^{\batchsize}) - \nabla f(x)}^2 \leq \frac{C^2 \V^2 \gamma}{\batchsize}\; .
    \]
    In addition, by item $1$ of \Cref{lem:core_lemma_general}, we have $\norm{\A(\widebar{g}_1^{\batchsize},\ldots,\widebar{g}_m^{\batchsize}) - \nabla f(x)}^2 \leq 2(4\kappa_{\delta} + 1)\V^2$, deterministically. Combining these results, by the law of total expectation, we get
    \[
        \E\!\norm{\A(\widebar{g}_1^{\batchsize},\ldots,\widebar{g}_m^{\batchsize}) - \nabla f(x)}^2 \leq \frac{C^2 \V^2 \gamma}{\batchsize} + 2(4\kappa_{\delta} + 1)\V^2\cdot\frac{1}{4mT} \leq \frac{C^2 \V^2 \gamma}{\batchsize} + \frac{\V^2\gamma}{T} \leq \frac{2C^2\V^2\gamma}{\batchsize}\; , 
    \]
    where the second inequality follows from $\frac{4\kappa_{\delta} + 1}{2m}\leq \gamma$, and the last inequality from $C^2\geq 1$ and $N\leq T$.
\end{proof}

Before we establish bounds on the bias and variance of the MLMC gradient estimator defined in \Cref{eq:mlmc_agg_general}, we show that $\Ecal_t$ is satisfied with high probability.
\begin{lemma}\label{lem:Ecal_hp_option1}
    Consider $\Ecal(J_t)$ defined in \Cref{eq:event_E_option1}. For every $t\in\goodrounds$ and $j=0,\ldots,\Jmax$, we have 
    \[
        \prob_{t-1}(\Ecal_t(j))\geq 1 - \frac{1}{2mT}\; ,    
    \]
    where the randomness is w.r.t the stochastic gradient samples.
\end{lemma}
\begin{proof}
    By item 2 of \Cref{lem:core_lemma_general}, we have 
    \[
        \prob_{t-1}\brac{\lVert \widehat{g}_t^{j} - \nabla_t\rVert \leq \frac{c_{\Ecal}C\V}{\sqrt{2^j}}} = \prob_{t-1}\brac{\lVert \widehat{g}_t^{j} - \nabla_t\rVert \leq C\V\sqrt{\frac{\gamma}{2^{j}}}} \geq 1-\frac{1}{4mT}, \quad \forall j=0,\ldots,\Jmax\; . 
    \]
    This bound, in conjunction with the union bound, allows us to bound $\prob_{t-1}(\Ecal_t(j)^c)$ as,
    \begin{align*}
        \prob_{t-1}(\Ecal_t(j)^c) &= \prob_{t-1}\brac{\lVert \widehat{g}_t^{j} - \widehat{g}_t^{j-1}\rVert> (1 + \sqrt{2})\frac{c_{\Ecal}C\V}{\sqrt{2^j}}} \\ &\leq\prob_{t-1}\brac{\cbrac{\lVert \widehat{g}_t^{j} - \nabla_t\rVert> \frac{c_{\Ecal}C\V}{\sqrt{2^j}}}\bigcup \cbrac{\lVert \widehat{g}_t^{j-1} - \nabla_t \rVert> \frac{c_{\Ecal}C\V}{\sqrt{2^{j-1}}}}} \\ &\leq \prob_{t-1}\brac{\lVert \widehat{g}_t^{j} - \nabla_t\rVert> \frac{c_{\Ecal}C\V}{\sqrt{2^j}}} + \prob_{t-1}\brac{\lVert \widehat{g}_t^{j-1} - \nabla_t\rVert> \frac{c_{\Ecal}C\V}{\sqrt{2^{j-1}}}} \\ &\leq \frac{1}{4mT} + \frac{1}{4mT} = \frac{1}{2mT}\; .
    \end{align*}
\end{proof}

Moving forward, we now provide bounds on the bias and variance estimator in \Cref{eq:mlmc_agg_general}. In \Cref{lem:bias-var-mlmc-general}, we establish that the bias in static rounds ($t\in\goodrounds$) is proportionate to $\V\sqrt{\gamma/T}$, whereas in bad rounds it is near-constant; the variance is also near-constant for every $t\in\sbrac{T}$. 

\begin{lemma}[MLMC Bias and Variance]\label{lem:bias-var-mlmc-general}
    Consider $g_t$ defined as in \Cref{eq:mlmc_agg_general}. Then,
    \begin{enumerate}
        \item The bias $b_t\coloneqq \E_{t-1} g_t - \nabla_t$ is bounded as 
        \begin{align*}
            \norm{b_t} \leq \begin{cases}
                2C\V\sqrt{\frac{\gamma}{T}} + 2\V\sqrt{\frac{\gamma}{m}}\frac{\log{T}}{T}, &t\in\goodrounds \\ 
                4C\V\sqrt{\gamma\log{T}}, &t\in\badrounds
            \end{cases}\; .
        \end{align*}
        \item The variance $V_t^2\coloneqq \E_{t-1}\norm{g_t - \E_{t-1} g_t}^2$ is bounded as 
        \begin{align*}
            V_t^2 \leq 16C^2 \V^2\gamma\log{T}, \quad\forall t\in\sbrac{T}\; .
        \end{align*}
    \end{enumerate}
\end{lemma}
\begin{proof}
    Our proof closely follows and builds upon the strategy employed in \Cref{lem:mlmc}. We begin by bounding the variance,
    \begin{align}\label{eq:mlmc_var_explicit}
        V_t^2 \coloneqq \E_{t-1}\norm{g_t - \E_{t-1} g_t}^2 \leq \E_{t-1}\norm{g_t - \nabla_t}^2 &= \sum_{j=1}^{\infty}{2^{-j}\E_{t-1}\norm{\widehat{g}_t^{0} + 2^{j}\brac{\widehat{g}_t^{j} - \widehat{g}_t^{j-1}}\mathbbm{1}_{\cbrac{j\leq\Jmax}\cap\Ecal_t(j)} - \nabla_t}^2} \nonumber \\ &\leq 2\sum_{j=1}^{\infty}{2^{-j}\E_{t-1}\norm{\widehat{g}_t^{0} - \nabla_t}^2} + 2\sum_{j=1}^{\Jmax}{2^{j}\E_{t-1}\sbrac{\lVert{\widehat{g}_t^{j} - \widehat{g}_t^{j-1}}\rVert^2\mathbbm{1}_{\Ecal_t(j)}}} \nonumber \\ &= 2\E_{t-1}\norm{\widehat{g}_t^{0} - \nabla_t}^2 + 2\sum_{j=1}^{\Jmax}{2^{j}\underbrace{\E_{t-1}\sbrac{\lVert{\widehat{g}_t^{j} - \widehat{g}_t^{j-1}}\rVert^2\mathbbm{1}_{\Ecal_t(j)}}}_{=(\dag)}}\; ,
    \end{align}
    where the last equality holds as $\sum_{j=1}^{\infty}{2^{-j}}=1$. Focusing on $\brac{\dag}$, by the law of total expectation, we have that
    \begin{align*}
        \E_{t-1}\sbrac{\lVert{\widehat{g}_t^{j} - \widehat{g}_t^{j-1}}\rVert^2\mathbbm{1}_{\Ecal_t(j)}} = \E_{t-1}\sbrac{\lVert{\widehat{g}_t^{j} - \widehat{g}_t^{j-1}}\rVert^2 | \Ecal_t(j)}\underbrace{\prob(\Ecal_t(j))}_{\leq 1} \leq \frac{(1+\sqrt{2})^2C^2 \V^2 \gamma}{2^{j}}\leq \frac{6C^2\V^2\gamma}{2^{j}} \; ,
    \end{align*}
    where the first inequality follows from the bound of $\lVert \widehat{g}_t^{j} - \widehat{g}_t^{j-1}\rVert$ under the event $\Ecal_t(j)$ (see \cref{eq:event_E_option1}). Furthermore, by \Cref{cor:mse_bound_general}, we can bound $\E_{t-1}\lVert \widehat{g}_t^{0}-\nabla_t \rVert^2\leq 2C^2 \V^2\gamma$. Substituting these bounds back into \Cref{eq:mlmc_var_explicit} finally gives:
    \begin{align*}
        V_t^2 \leq \E_{t-1}\norm{g_t - \nabla_t}^2 &\leq 4C^2\V^2\gamma + 2\sum_{j=1}^{\Jmax}{2^j\cdot\frac{6C^2\V^2\gamma}{2^{j}}} \leq 4C^2\V^2\gamma + 12C^2\V^2\gamma\Jmax \leq 16C^2\V^2\gamma\log{T}\; ,
    \end{align*}
    where the last inequality follows from $\Jmax\leq\log{T}$. 
    
    Proceeding to bound the bias, for every $t\in\sbrac{T}$, we have by Jensen's inequality,
    \begin{align*}
        \norm{b_t} \leq \sqrt{\E_{t-1}\norm{g_t - \nabla_t}^2} \leq 4C\V\sqrt{\gamma\log{T}}\; .
    \end{align*}
    However, for $t\in\goodrounds$ the Byzantine workers are fixed, and a tighter bound can be derived. Taking expectation w.r.t $J_t$ gives:
    \begin{equation}\label{eq:mlmc_exp}
        \E_{t-1}[g_t] = \sum_{j=1}^{\infty}{2^{-j}\cdot \E_{t-1}\sbrac{\widehat{g}_t^{0} + 2^{j}\brac{\widehat{g}_t^{j} - \widehat{g}_t^{j-1}}\mathbbm{1}_{\cbrac{j\leq\Jmax}\cap\Ecal_t(j)}}} = \E_{t-1}[\widehat{g}_t^{0}] + \sum_{j=1}^{\Jmax}{\E_{t-1}\sbrac{\brac{\widehat{g}_t^{j} - \widehat{g}_t^{j-1}}\mathbbm{1}_{\Ecal_t(j)}}}\; .
    \end{equation}
    Utilizing \Cref{lem:expectation_indicator}, we can express each term in the sum as,
    \[
        \E_{t-1}\sbrac{\brac{\widehat{g}_t^{j} - \widehat{g}_t^{j-1}}\mathbbm{1}_{\Ecal_t(j)}} = \E_{t-1}\sbrac{\widehat{g}_t^{j} - \widehat{g}_t^{j-1}} - \E_{t-1}\sbrac{\widehat{g}_t^{j} - \widehat{g}_t^{j-1} | \Ecal_t(j)^c}\prob_{t-1}(\Ecal_t(j)^{c})\; .
    \]
    Denote the last term in the R.H.S by $z_t^{j}\coloneqq\!\E_{t-1}\!\sbrac{\widehat{g}_t^{j} - \widehat{g}_t^{j-1} | \Ecal_t(j)^c}\prob_{t-1}(\Ecal_t(j)^{c})$. Plugging this back into \Cref{eq:mlmc_exp}, we obtain:
    \begin{align}
        \E_{t-1}[g_t] &= \E_{t-1}[\widehat{g}_t^{0}] + \sum_{j=1}^{\Jmax}{\brac{\E_{t-1}[\widehat{g}_t^{j} - \widehat{g}_t^{j-1}] - z_t^{j}}} = \E_{t-1}[\widehat{g}_t^{\Jmax}] + y_t \nonumber \; ,
    \end{align}
    where $y_t \coloneqq -\sum_{j=1}^{\Jmax}{z_t^{j}}$. Thus, by the triangle inequality and Jensen's inequality, it holds that
    \begin{align}
        \norm{b_t} = \norm{\E_{t-1}[g_t - \nabla_t]} = \lVert{\E_{t-1}[\widehat{g}_t^{\Jmax} + y_t - \nabla_t]}\rVert  &\leq \lVert \E_{t-1}[\widehat{g}_t^{\Jmax} - \nabla_t]\rVert + \norm{\E_{t-1}y_t} \label{eq:bias-mlmc-good-rounds} \\ &\leq \sqrt{\E_{t-1}\lVert{\widehat{g}_t^{\Jmax} - \nabla_t\rVert}^2} + \E_{t-1}\norm{y_t} \nonumber\\ &\leq 2C\V\sqrt{\frac{\gamma}{T}} + \E_{t-1}\norm{y_t}\; , \label{eq:bias-mlmc-general-good-rounds} 
    \end{align}
    where the last inequality follows from \Cref{cor:mse_bound_general} and $2^{\Jmax}\geq T/2$. Our objective now is to bound $\norm{y_t}$; note that for every $j=1,\ldots,\Jmax$, we can bound $z_t^{j}$ using Jensen's inequality as follows:
    \begin{align*}
        \lVert{z_t^{j}\rVert} &= \norm{\E_{t-1}\sbrac{\widehat{g}_t^{j} - \widehat{g}_t^{j-1} | \Ecal_t(j)^c}\prob_{t-1}(\Ecal_t(j)^{c})} \leq \E_{t-1}\sbrac{\lVert \widehat{g}_t^{j} - \widehat{g}_t^{j-1} \rVert | \Ecal_t(j)^c}\prob_{t-1}(\Ecal_t(j)^c)\; .
    \end{align*}
    By item 1 of \Cref{lem:core_lemma_general}, for every $j=0,\ldots,\Jmax$ it holds that $\lVert{\widehat{g}_t^{j} - \nabla_t\rVert}\leq \V\sqrt{2(4\kappa_{\delta} + 1)}$, which implies that 
    \[
        \lVert{ \widehat{g}_t^{j} - \widehat{g}_t^{j-1}\rVert}\leq \lVert{\widehat{g}_t^{j} - \nabla_t\rVert} + \lVert{\widehat{g}_t^{j-1} - \nabla_t\rVert}\leq 2\V\sqrt{2(4\kappa_{\delta} + 1)}\; . 
    \]
    In addition, we have by \Cref{lem:Ecal_hp_option1} that $\prob_{t-1}(\Ecal_t(j)^c)\leq \frac{1}{2mT}$. Thus, we obtain:
    \[
        \lVert z_t^{j}\rVert \leq 2\V\sqrt{2(4\kappa_{\delta} + 1)}\cdot\frac{1}{2mT} \leq \frac{2\V}{T}\sqrt{\frac{\gamma}{m}}\; ,
    \]
    where we used $\frac{4\kappa_{\delta}+1}{m}\leq 2\gamma$. This in turn implies, by the triangle inequality, the following bound on $y_t$:
    \[
        \norm{y_t} \leq \sum_{j=1}^{\Jmax}{\lVert z_t^{j}\rVert} \leq \frac{2\V}{T}\sqrt{\frac{\gamma}{m}} \Jmax \leq 2\V\sqrt{\frac{\gamma}{m}}\frac{\log{T}}{T} \; ,
    \]
    where we used $\Jmax\leq\log{T}$. Plugging this bound back into \Cref{eq:bias-mlmc-general-good-rounds} yields:
    \begin{align*}
        \norm{b_t} \leq 2C\V\sqrt{\frac{\gamma}{T}} + 2\V\sqrt{\frac{\gamma}{m}}\frac{\log{T}}{T}, \quad\forall t\in\goodrounds\; .
    \end{align*}
\end{proof}

Using the established bounds on bias and variance, we derive convergence guarantees for the convex and non-convex cases.

\subsection{Convex Case}\label{subapp:convex-dynamic}
The following result establishes the convergence of our approach in the convex setting.
\begin{theorem}
   Assume $f$ is convex. Under Assumptions \ref{assump:bounded-noise} and \ref{assump:bounded-domain}, and with a $(\delta, \kappa_{\delta})$-robust aggregator $\A$, consider \Cref{alg:method-new} with \textcolor{purple}{\textbf{Option 1}} and a fixed learning rate given by
    \[
        \eta_t = \eta \coloneqq \min\cbrac{\frac{D}{6C\V\sqrt{\gamma T\log{T}}}, \frac{1}{2L}}\; .
    \]
    Then, for $\widebar{x}_T \coloneqq \frac{1}{T}\sum_{t=2}^{T+1}{x_t}$, the following holds:
    \begin{align*}
        \E f(\widebar{x}_T) - f(x^*) &\leq 9CD\V\sqrt{\frac{\gamma\log{T}}{T}} + \frac{LD^2}{T} +4CD\V\sqrt{\gamma\log{T}}\frac{\abs{\badrounds}}{T}\; .
    \end{align*}
\end{theorem}

\begin{proof}
    By \Cref{lem:convex-sgd}, we have
    \begin{equation*}
        \E f(\widebar{x}_T) - f(x^*) \leq \frac{D^2}{2T\eta} + \frac{\eta}{T}\sum_{t=1}^{T}{ \E V_t^2} + \frac{D}{T}\sum_{t=1}^{T}{\E\norm{b_t}}\; .
    \end{equation*}
    Substituting the established bounds on $\norm{b_t}$ and $V_t^2$ from \Cref{lem:bias-var-mlmc-general}, 
    \begin{align*}
        \E f(\widebar{x}_T) - f(x^*) &\leq \frac{D^2}{2T\eta} + 16\eta C^2\V^2\gamma\log{T} + \frac{D}{T}\brac{\sum_{t\in\badrounds}{\E\norm{b_t}} + \sum_{t\in\goodrounds}{\E\norm{b_t}}} \\ &\leq \frac{1}{2}\brac{\frac{D^2}{T\eta} + 36\eta C^2 \V^2\gamma\log{T}} \!+\! \frac{D}{T}\brac{4C\V\sqrt{\gamma\log{T}}\abs{\badrounds} \!+\! \brac{2C\V\sqrt{\frac{\gamma}{T}} \!+\! 2\V\sqrt{\frac{\gamma}{m}}\frac{\log{T}}{T}}(T - \abs{\badrounds})} \\ &\leq \frac{1}{2}\brac{\frac{D^2}{T\eta} + 36\eta C^2 \V^2\gamma\log{T}} + 4CD\V\sqrt{\gamma\log{T}}\frac{\abs{\badrounds}}{T} + 2CD\V\sqrt{\frac{\gamma}{T}} + 2D\V\sqrt{\frac{\gamma}{m}}\frac{\log{T}}{T}\; .
    \end{align*}
    Since $\eta = \min\cbrac{\frac{D}{6C\V\sqrt{\gamma T\log{T}}}, \frac{1}{2L}}$, applying \Cref{lem:lr_min_of_2_lrs} with $a= D^2/T, b=36C^2\V^2\gamma\log{T}$, and $c=2L$, allows us to bound the sum of the first two terms as follows
    \[
        \frac{D^2}{T\eta} + 36\eta C^2 \V^2\gamma\log{T} \leq 12CD\V\sqrt{\frac{\gamma\log{T}}{T}} + \frac{2LD^2}{T}\; .
    \]
    Plugging this bound back gives:
    \begin{align*}
        \E f(\widebar{x}_T) - f(x^*) &\leq 6CD\V\sqrt{\frac{\gamma\log{T}}{T}} + \frac{LD^2}{T} + 4CD\V\sqrt{\gamma\log{T}}\frac{\abs{\badrounds}}{T} + 2CD\V\sqrt{\frac{\gamma}{T}} + 2D\V\sqrt{\frac{\gamma}{m}}\frac{\log{T}}{T} \\ &\leq 8CD\V\sqrt{\frac{\gamma\log{T}}{T}} + \frac{LD^2}{T} + 4CD\V\sqrt{\gamma\log{T}}\frac{\abs{\badrounds}}{T} + CD\V\frac{\sqrt{\gamma\log{T}}}{\sqrt{m}T} \\ &\leq 9CD\V\sqrt{\frac{\gamma\log{T}}{T}} + \frac{LD^2}{T} + 4CD\V\sqrt{\gamma\log{T}}\frac{\abs{\badrounds}}{T}\; ,
    \end{align*}
    where in the second inequality we used $C=2\sqrt{2\log\brac{16m^2 T}}\geq 2\sqrt{\log{T}}$ to bound the last term.
\end{proof}

This theorem implies the following observation.
\begin{corollary}\label{cor:convex}
    If $\abs{\badrounds}\in\O(\sqrt{T})$, the first term dominates the convergence rate, which is $\Otilde(D\V\sqrt{\nicefrac{\gamma}{T}})$. Specifically, for $\kappa_{\delta}\in\O(\delta)$ this rate is given by $\Otilde\big(D\V\sqrt{\brac{\delta + \nicefrac{1}{m}}/T}\big)$.
\end{corollary}
\subsection{Non-convex Case}\label{subapp:nonconvex-dynamic}
Having established the proof for the convex case, we move on to proving convergence in the non-convex scenario. For ease of reference, \Cref{thm:nonconvex} is restated here.
\begin{customthm}{4.1}
    Under \Cref{assump:bounded-noise}, and with a $(\delta, \kappa_{\delta})$-robust aggregator $\A$, consider \Cref{alg:method-new} with \textcolor{purple}{\textbf{Option 1}} and a fixed learning rate given by
    \[
        \eta_t = \eta \coloneqq \min\cbrac{
        \frac{\sqrt{\Delta_1}}{3C\V\sqrt{L\gamma T\log{T}}}, \frac{1}{L}}\; .
    \]
    where $\gamma\coloneqq 2\kappa_{\delta} + \frac{1}{m}$. Then, the following holds:
    \begin{align*}
        \frac{1}{T}\sum_{t=1}^{T}{\E\!\norm{\nabla_t}^2}&\leq 12C\V\sqrt{\frac{L\Delta_1\gamma\log{T}}{T}} + \frac{2L\Delta_1 + 9C^2\V^2\gamma}{T} + 16C^2\V^2\gamma\log{T}\frac{\abs{\badrounds}}{T}\; .
    \end{align*}
\end{customthm}

\begin{proof}
    Utilizing \Cref{lem:nonconvex-sgd}, we get:
    \begin{equation}\label{eq:nonconvex-sgd-general}
        \frac{1}{T}\sum_{t=1}^{T}{\E\norm{\nabla f(x_t)}^2} \leq \frac{2\Delta_1}{T\eta} + \frac{\eta L}{T}\sum_{t=1}^{T}{\E V_t^2} + \frac{1}{T}\underbrace{\sum_{t=1}^{T}{\E\norm{b_t}^2}}_{=(\star)}\; .
    \end{equation}
    \paragraph{Bounding $(\star)$: } Using the bound on the bias in item 1 of \Cref{lem:bias-var-mlmc-general}, we can bound
    \begin{align*}
        \sum_{t=1}^{T}{\E\norm{b_t}^2} = \sum_{t\in\badrounds}{\E\norm{b_t}^2} + \sum_{t\notin\badrounds}{\E\norm{b_t}^2} &\leq 16C^2 \V^2\gamma\log{T}\abs{\badrounds} + \brac{2C\V\sqrt{\frac{\gamma}{T}} + 2\V\sqrt{\frac{\gamma}{m}}\frac{\log{T}}{T}}^2(T-\abs{\badrounds}) \\ &\leq 16C^2 \V^2\gamma\log{T}\abs{\badrounds} + 2T\brac{\frac{4C^2\V^2\gamma}{T} + \frac{4\V^2\gamma\log^2{T}}{mT^2}} \\ &\leq 16C^2 \V^2\gamma\log{T}\abs{\badrounds} + 8C^2\V^2\gamma + \frac{C^2\V^2\gamma\log{T}}{mT} \\ &\leq C^2\V^2\gamma\brac{16\abs{\badrounds}\log{T} + 9}\; ,
    \end{align*}
    where the second inequality follows from $(a+b)^2\leq 2a^2 + 2b^2$ and $T-\abs{\badrounds}\leq T$; the third inequality uses $8\log{T}\leq C^2$; and the last inequality follows from $\log{T}\leq mT$.
    
    Substituting the bound on $(\star)$ and the variance bound from \Cref{lem:bias-var-mlmc-general} back into \Cref{eq:nonconvex-sgd-general} yields:
    \begin{align*}
        \frac{1}{T}\sum_{t=1}^{T}{\E\norm{\nablat}^2} &\leq \frac{2\Delta_1}{T\eta} + 16\eta LC^2 \V^2 \gamma\log{T} + \frac{C^2\V^2\gamma\brac{16\abs{\badrounds}\log{T} + 9}}{T} \\ &\leq 2\brac{\frac{\Delta_1}{T\eta} + 9\eta LC^2 \V^2 \gamma\log{T}} + \frac{C^2\V^2\gamma\brac{16\abs{\badrounds}\log{T} + 9}}{T}\; .
    \end{align*}
    Utilizing \Cref{lem:lr_min_of_2_lrs} with $a= \Delta_1/T, b=9LC^2\V^2\gamma\log{T}$, and $c=L$ enables to bound the sum of the first two terms as
    \[
        \frac{\Delta_1}{T\eta} + 9\eta L C^2\V^2\gamma\log{T} \leq 6C\V\sqrt{\frac{L\Delta_1\gamma\log{T}}{T}} + \frac{L\Delta_1}{T}\; .
    \]
    Plugging this bound, we get:
    \begin{align*}
        \frac{1}{T}\sum_{t=1}^{T}{\E\norm{\nablat}^2} &\leq 12C\V\sqrt{\frac{L\Delta_1\gamma\log{T}}{T}} + \frac{2L\Delta_1}{T} + \frac{C^2\V^2\gamma\brac{16\abs{\badrounds}\log{T} + 9}}{T} \\ &= 12C\V\sqrt{\frac{L\Delta_1\gamma\log{T}}{T}} + \frac{2L\Delta_1 + 9C^2\V^2\gamma}{T} + 16C^2\V^2\gamma\frac{\abs{\badrounds}\log{T}}{T} \; .
    \end{align*}
\end{proof}
\section{When Worker-Momentum Fails}\label{app:momentum_breaks}
Next, we take a detour, to show how the worker-momentum approach may fail in the presence of Byzantine identity changes. To this end, we introduce a Byzantine identity \emph{switching strategy}, which utilizes the momentum recursion to ensure that all workers suffer from a sufficient bias. For simplicity, we consider a setting with $m=3$ workers\footnote{For general $m$, we can divide the workers into $3$ groups and apply our switching strategy to these groups.}, of which only a single worker is Byzantine in each round. 

For some round $t$, consider the following momentum update rule with parameter $\beta\in[0, 1)$ for worker $i$,
\begin{equation*}\label{eq:worker_momentum}
    \tilde{m}_{t,i} = \beta \tilde{m}_{t-1, i} + (1-\beta) \tilde{g}_{t,i}\; .
\end{equation*}

As mentioned, this update rule effectively averages the last $1/\alpha$ gradients, where $\alpha \coloneqq 1-\beta$. 
Thus, the $1/\alpha$ rounds following a Byzantine-to-honest identity switch still heavily depend on the Byzantine behavior. Intuitively, this `healing phase' of $1/\alpha$ rounds is the time required for the worker to produce informative honest updates. Our attack leverages this property to perform an identity switch once in every $1/3\alpha$, to maintain all workers under the Byzantine effect, i.e., to prevent workers from completely `healing' from the attack. Recall that existing approaches to Byzantine-resilient strategy suggests choosing $\alpha\approx 1/\sqrt{T}$ to establish theoretical guarantees (see, e.g.,~\citealp{karimireddy2021learning,allouah2023fixing}). Thus, henceforth, we will assume $\alpha\leq 1/6$ and, for the sake of simplicity, that $1/3\alpha$ is an integer.

We divide the $T$ training rounds into epochs (i.e., windows) of size $1/\alpha\approx\sqrt{T}$. Within these epochs, we perform an identity switch once in every $1/3\alpha$ rounds, periodically, implying $3$ identity switches per-epoch and $3\alpha T\in\O(\sqrt{T})$ switches overall. Concretely, denoting by $\tilde{g}_{t,i}$ the gradient used by worker $i\in\cbrac{1,2,3}$ at time $t$ to perform momentum update, we consider the following attack strategy:
\[
    \tilde{g}_{t,i} = g_{t,i} + v_t\cdot \mathbbm{1}\cbrac{t\hspace{-0.5em}\mod \frac{1}{\alpha}\in\sbrac{\frac{i-1}{3\alpha}+1, \frac{i}{3\alpha}}}\; ,
\]
where $g_{t,i}$ is an honest stochastic gradient and $v_t$ is an attack vector to be defined later. Note that under this attack strategy, there is indeed only a single Byzantine machine at a time in all rounds. 

Denote by $\tilde{m}_{t,i}\coloneqq m_{t,i} + b_{t,i}$ the momentum used by worker $i$ in round $t$, where $m_{t,i}=(1-\alpha) m_{t-1,i} + \alpha g_{t,i}$ is the honest momentum (without Byzantine attack) and $b_{t,i}$ is the bias introduced by our attack. We want to find a recursion for $b_{t,i}$ to characterize the dynamics of the deviation from the honest momentum protocol. By plugging $\tilde{m}_{t-1,i}$ and $\tilde{g}_{t,i}$, we get:
\begin{align*}
    \tilde{m}_{t,i} &= (1-\alpha) \brac{m_{t-1,i} + b_{t-1,i}} + \alpha\brac{g_{t,i} + v_t\cdot \mathbbm{1}\cbrac{t\hspace{-0.5em}\mod \frac{1}{\alpha}\in\sbrac{\frac{i-1}{3\alpha}+1, \frac{i}{3\alpha}}}} \\ &= m_{t,i} + (1 - \alpha) b_{t-1,i} + \alpha v_t\cdot \mathbbm{1}\cbrac{t\hspace{-0.5em}\mod \frac{1}{\alpha}\in\sbrac{\frac{i-1}{3\alpha}+1, \frac{i}{3\alpha}}}\; .
\end{align*}
\begin{figure}[t]
    \centering
\includegraphics[width=0.8\linewidth]{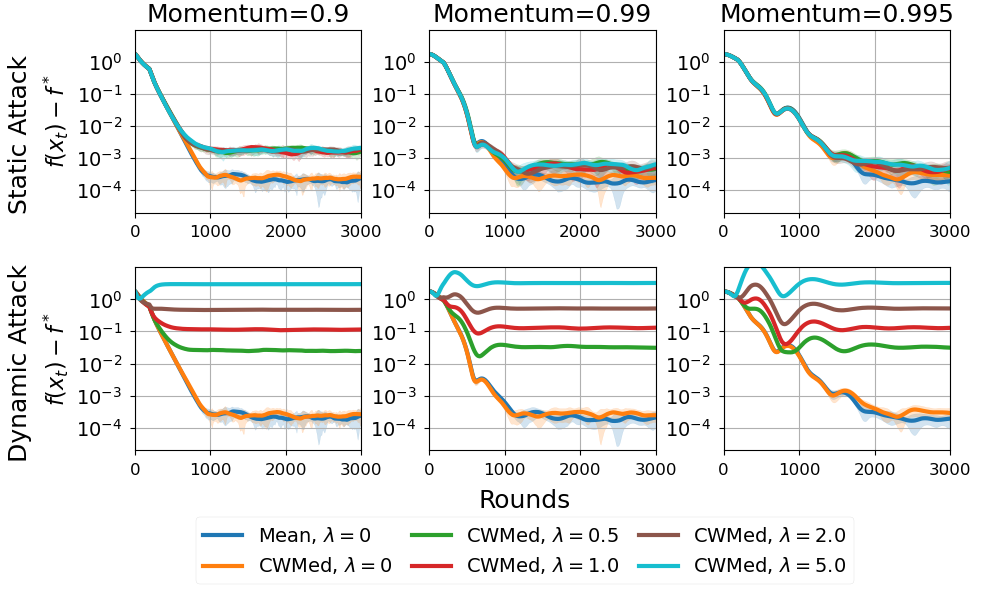}
    \vspace{-3mm}
    \caption{Optimality gap ($f(x_t)-f^*$) under static (\textbf{top}) and dynamic (\textbf{bottom}) attacks across various momentum parameters and for different attack strengths ($\lambda=0, 0.5, 1, 2, 5$). The average and 95\% confidence interval are presented over $20$ random seeds.}
    \label{fig:err}
\end{figure}
This implies the following recursion for the bias of worker $i$,
\begin{equation*}
    b_{t,i} = (1 - \alpha) b_{t-1,i} + \alpha v_t\cdot \mathbbm{1}\cbrac{t\hspace{-0.5em}\mod \frac{1}{\alpha}\in\sbrac{\frac{i-1}{3\alpha}+1, \frac{i}{3\alpha}}}\; .
\end{equation*}
Let $v\in\reals^d$ be some fixed vector. By carefully choosing $v_t$, as we describe next, we ensure that $b_{t,i}=v$ for all rounds $t$ under which worker $i$ is Byzantine. Note that $b_{1,1} = v_{1}$, and $b_{1,2}=b_{1,3}=0$. We distinguish between the first epoch and the following ones. For the first epoch, i.e., $t\in\sbrac{1/\alpha}$, we choose:
\begin{align*}
    v_{t} = v\cdot\begin{cases}
        1/\alpha, &t\in\cbrac{\frac{1}{3\alpha}+1, \frac{2}{3\alpha}+1} \\ 1, &\text{otherwise}.
    \end{cases}
\end{align*}
For the subsequent epochs, i.e., $t\in\sbrac{\frac{1}{\alpha}+1, T}$, we choose:
\begin{align*}
    v_{t} = v\cdot\begin{cases}
        \brac{1 - (1-\alpha)^{2/3\alpha}}/\alpha, &t\hspace{-0.5em}\mod \frac{1}{\alpha} = 1 \\ 1, &\text{otherwise}.
    \end{cases}
\end{align*}
The following lemma establishes that all worker momentums are sufficiently biased under this attack.

\begin{lemma}
    Given the above attack, for any $t$ starting from the second epoch, the following holds for $i=1,2,3$:
    \begin{align*}
        \tilde{m}_{t,i} = m_{t,i} + \theta_{t,i} v\; ,
    \end{align*}
    where $\theta_{t,i}\in\sbrac{\theta_{\min}, 1}$ for each $i$ and for all $t>1/\alpha$. Here, $\theta_{\min}\coloneqq \brac{\frac{5}{6}}^{4}> 0.48$.
\end{lemma}
\begin{proof}
    Let us examine the bias dynamics of the first worker under the described attack.

    In the first third of each epoch, namely, for $t\hspace{-0.5em}\mod \frac{1}{\alpha}\in\sbrac{\frac{1}{3\alpha}+1, \frac{1}{3\alpha}}$, we have a fixed bias of $b_{t,1}=v$. For the remaining two thirds, we have an exponential bias decay, where $b_{t,1}=\theta_{t,1}v$ for $\theta_{t,1}\coloneqq(1-\alpha)^{\brac{t-\frac{1}{3\alpha}}\hspace{-0.5em}\mod\hspace{-0.2em}\frac{1}{\alpha}}$. Since $\beta^{t}$ is decreasing with $t$ for $\beta<1$, the coefficients sequence $\theta_{t,i}$ obtains its minimum at the end of each epoch, when $t\hspace{-0.5em}\mod\hspace{-0.2em}\frac{1}{\alpha}=0$, given by $(1-\alpha)^{2/3\alpha}$. For $\alpha\leq 1/6$, it holds that $(1-\alpha)^{2/3\alpha}\geq \brac{\frac{5}{6}}^{4} \coloneqq \theta_{\min}$. Since $\theta_{t,2}$ and $\theta_{t,3}$ are simply a shift of $\theta_{t,1}$ by $1/3\alpha$ and $2/3\alpha$ rounds, respectively, they satisfy the same bounds. 
\end{proof}

Given the above bias statement, any robust aggregation rule has no ability to infer an unbiased momentum path, and it would arbitrarily fail as $v$ is unbounded.

\begin{figure}[t]
    \centering    \includegraphics[width=0.8\linewidth]{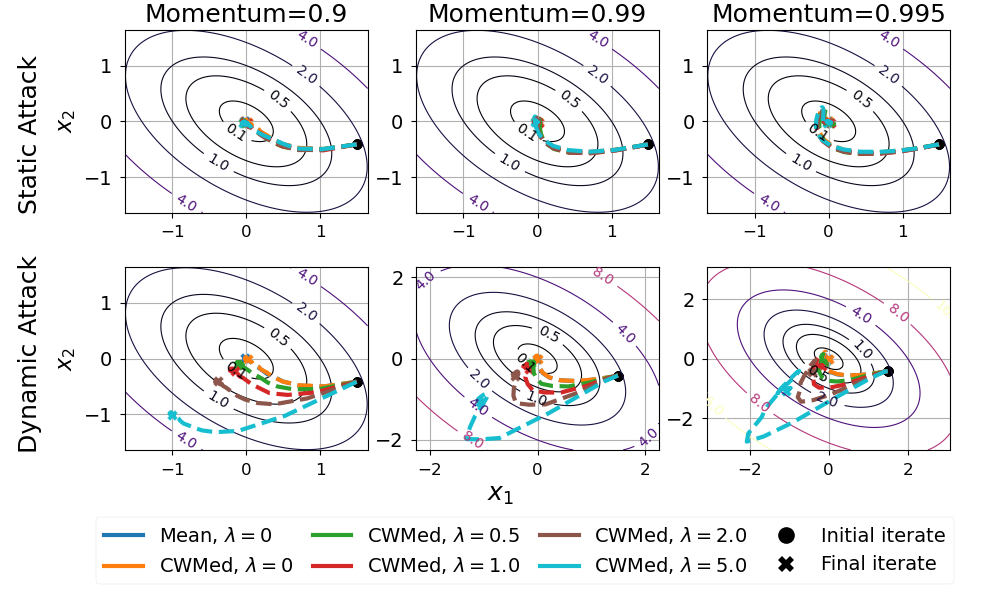}
    \vspace{-3mm}
    \caption{Optimization trajectories under static (\textbf{top}) and dynamic (\textbf{bottom}) attacks for a range of momentum parameters and for different attack strengths ($\lambda=0, 0.5, 1, 2, 5$). Note that under the dynamic attack, the algorithm converges to a sub-optimal solution.}
    \label{fig:opt_path}
\end{figure}

We provide an empirical evidence 
to demonstrate our observation using a simple 2D quadratic example. Consider the function $f(x) = \frac{1}{2}x^T A x$ with $x\in\reals^{2}$ and $A$ is the matrix $\begin{bmatrix}
    2 & 1 \\
    1 & 2 
\end{bmatrix}$. In our attack setup, each worker ($i=1,2,3$) employs momentum-SGD. The honest gradient oracle for each worker is defined as $g_{t,i}=\nabla_t + n_{t,i}$, where $n_{t,i}\overset{i.i.d}{\sim}\mathcal{N}(0,\sigma^2 I)$ with $\sigma=0.5$. We set the attack vector to $v=\lambda\cdot \begin{bmatrix}
    1 & 1
\end{bmatrix}^\top$, and examine various values of $\lambda\in\cbrac{0, 0.5, 1, 2, 5}$ ($\lambda=0$ corresponds to the Byzantine-free setting). At the server level, we process the worker-momentums using either simple averaging (\textsf{Mean}) or coordinate-wise median (\textsf{CWMed}). The aggregated momentum $\widehat{m}_{t}$ is then used in the update rule: $x_{t+1} = x_{t} - \eta \widehat{m}_t$ with a learning rate $\eta = 5\cdot 10^{-3}$ over $T=3000$ rounds. We experiment with various momentum parameters 
$\beta$ from $\cbrac{0.9, 0.99, 0.995}$, corresponding to $\alpha$ values of $\cbrac{0.1. 0.01, 0.005}$, and repeat each experiment with $20$ different random seeds. Note that these values correspond to $\abs{\badrounds}= 999, 90, 45$ rounds with Byzantine identity changes. Additionally, we include a 'static attack' scenario where only the first worker is consistently Byzantine, using $\tilde{g}_{t,1} = g_{t,1} + v$ throughout all rounds.



In \Cref{fig:err}, we illustrate the optimality gap during the training process. Notably, in the presence of a dynamic attack (where $\lambda > 0$), we observe that the error plateaus at a sub-optimal level for all values of the momentum parameter. Furthermore, there is a clear trend that shows an increase in the final error magnitude in direct proportion to the strength of the attack (as $\lambda$ increases). This trend is distinct from what we observe under a static attack, where such a correlation between the attack strength and final error is not apparent.

Correspondingly, in \Cref{fig:opt_path}, we present a representative example showcasing the optimization paths under the influence of the static and dynamic attacks, with various momentum parameters. The trajectories visibly diverge towards sub-optimal points under dynamic attacks, with the divergence growing as the attack strength, i.e., $\lambda$, is increased. In contrast, the static attack scenarios reveal paths that remain relatively stable despite changes in attack strength. This visual illustration underscores the possible failure of the worker-momentum approach under dynamic Byzantine attack.




\section{Properties of the MFM Aggregator}\label{app:mfm_prop}
In this section, we establish the properties of the MFM aggregator, crucial for our analysis of \Cref{sec:adaptivity}. Additionally, in \Cref{subapp:mfm_is_not_delta_kappa_robust}, we demonstrate that MFM does not meet the $(\delta, \kappa)$-robustness criteria. 

We assume the gradient noise is bounded (\Cref{assump:bounded-noise}) and consider the MFM aggregator with inputs $\widebar{g}_1^{\batchsize}, \ldots, \widebar{g}_m^{\batchsize}$ as in \Cref{lem:robust_agg_is_lmgo} and threshold parameter set to $\T_p^{\batchsize}= 2C_p\V/\sqrt{N}$, where $C_p\coloneqq\sqrt{8\log\!\brac{2m/p}}$.

Initially, we introduce the following event, under which we derive valuable insights regarding \Cref{alg:mfm}.
\begin{equation}\label{eq:eventB}
    \B = \cbrac{\forall i\in\G : \norm{\widebar{g}_i^{\batchsize} - \nabla f(x)} \leq \frac{\T_p^{\batchsize}}{4} = \V\sqrt{\frac{2\log{\brac{2m/p}}}{\batchsize}}}\; .
\end{equation}

Here, we analyze \Cref{alg:mfm} assuming the honest workers are fixed when computing mini-batches. 

First, we show that $\B$ is satisfied with high probability.
\begin{lemma}\label{lem:B_has_high_probability}
    For every $x\in\reals$, it holds that $\prob(\B)\geq 1-p$.
\end{lemma}
\begin{proof}
    Since for every $i\in\G$, $\widebar{g}_i^{\batchsize}=\frac{1}{\batchsize}\sum_{n=1}^{\batchsize}{\nabla F(x; \xi_{i}^{n})}$, where $\norm{\nabla F(x; \xi_i^{n}) - \nabla f(x)}\leq \V$, utilizing \Cref{lem:concentration} implies that with probability at least $1-\tilde{p}$, we have $\norm{\widebar{g}_i^{\batchsize} - \nabla f(x)} \leq \V\sqrt{\frac{2\log\brac{2/\tilde{p}}}{\batchsize}}$. Employing the union bound and using $\abs{\G}\leq m$ establishes the result.
\end{proof}

Next, we establish some results assuming $\B$ holds.

\begin{lemma}\label{lem:mfm_helper}
    Under the event $\B$, the following holds:
    \begin{enumerate}
        \item The set $\M$ is not empty.
        \item $\G\subseteq\widehat{\G}$.
        \item For every $i\in\widehat{\G}$, we have $\norm{\widebar{g}_i^{\batchsize} - \nabla f(x)}\leq 2\T_p^{\batchsize}$.
    \end{enumerate}
\end{lemma}
\begin{proof}
    Following the proof of Claim 3.4 in \citet{alistarh2018byzantine}: under the event $\B$, for every $i,j\in\G$, we have by the triangle inequality that $\norm{\widebar{g}_i^{\batchsize} - \widebar{g}_j^{\batchsize}}\leq \norm{\widebar{g}_i^{\batchsize} - \nabla f(x)} + \norm{\widebar{g}_j^{\batchsize} - \nabla f(x)}\leq \nicefrac{\T_p^{\batchsize}}{2}$. Since $\delta<\nicefrac{1}{2}$, every $i\in\G$ is also in $\M$, namely, $\G\subseteq\M$, which concludes the first part. 
    
    For the second and third parts, we first show that $\norm{g_{\text{med}} - \nabla f(x)}\leq \nicefrac{3\T_p^{\batchsize}}{4}$; assuming $\norm{g_{\text{med}} - \nabla f(x)} > \nicefrac{3\T_p^{\batchsize}}{4}$, we get by the triangle inequality that $\norm{g_{\text{med}} - \widebar{g}_{i}^{\batchsize}} > \nicefrac{\T_p^{\batchsize}}{2}$ for every $i\in\G$, thus contradicting the definition of $g_{\text{med}}$ (chosen from $\M$) as $\abs{\G}>\nicefrac{m}{2}$. 

    To prove the second part, fix some $i\in\G$. By the triangle inequality: $\norm{\widebar{g}_i^{\batchsize} - g_{\text{med}}} \leq \norm{\widebar{g}_i^{\batchsize} - \nabla f(x)} + \norm{g_{\text{med}} - \nabla f(x)} \leq \nicefrac{\T_p^{\batchsize}}{4} + \nicefrac{3\T_p^{\batchsize}}{4} = \T_p^{\batchsize}$, where the bound on $\norm{\widebar{g}_i^{\batchsize} - \nabla f(x)}$ follows from the definition of $\B$. This bound implies that $i\in\widehat{\G}$ as well, concluding the second part.
    
    Finally, note that for any $i\in\widehat{\G}$, the triangle inequality implies $\norm{\widebar{g}_i^{\batchsize} - \nabla f(x)} \leq \norm{\widebar{g}_i^{\batchsize} - g_{\text{med}}} + \norm{g_{\text{med}} - \nabla f(x)} \leq \T_p^{\batchsize} + \nicefrac{3\T_p^{\batchsize}}{4} \leq 2\T_p^{\batchsize}$, where $\norm{\widebar{g}_i^{\batchsize} - g_{\text{med}}}\leq \T_p^{\batchsize}$ holds for any $i\in\widehat{\G}$ as established in the previous part.
\end{proof}

With these insights, we now prove \Cref{lem:core_lemma_mfm}, which, similarly to \Cref{lem:core_lemma_general}, provides deterministic and high-probability bounds on the aggregation error.
\begin{lemma}\label{lem:core_lemma_mfm}
    Consider the setting in \Cref{lem:core_lemma_general} and let $\widehat{g}^{\batchsize}\gets\text{MFM}(\widebar{g}_{1}^{\batchsize}, \ldots, \widebar{g}_{m}^{\batchsize}; \T_p^{\batchsize})$ be the output of \Cref{alg:mfm} with $\T_p^{\batchsize}\coloneqq 2C_p\V/\sqrt{N}$, where $C_p\coloneqq \sqrt{8\log\!\brac{2m/p}}$. Then,
    \begin{enumerate}
        \item $\norm{\widehat{g}^{\batchsize} - \nabla f(x)}^2\leq \frac{9}{2}(\T_p^{\batchsize})^2 + 2\V^2 +  \norm{\nabla f(x)}^2 \leq  20C_p^2\V^2 + \norm{\nabla f(x)}^2$.
        \item With probability at least $1-p$,
        \[
            \norm{\widehat{g}^{\batchsize} - \nabla f(x)}^2 \leq \frac{2C_p^2\V^2}{\batchsize}\brac{\frac{1}{m} + 128\delta^2}\; .
        \]
    \end{enumerate}
\end{lemma}

\begin{proof}
Denote: $\nabla\coloneqq \nabla f(x)$. To prove the first part, we consider two cases: when $\M$ is either empty or non-empty. If $\M=\emptyset$, then $\widehat{g}^{\batchsize} = 0$, leading to $\norm{\widehat{g}^{\batchsize} - \nabla}^2 = \norm{\nabla}^2$. In the case where $\M$ is non-empty, the triangle inequality implies that
\[
    \norm{\widehat{g}^{\batchsize} - \nabla} = \norm{\frac{1}{\abs{\widehat{\G}}}\sum_{i\in\widehat{\G}}{\brac{\widebar{g}_i^{\batchsize} - \nabla}}} \leq \frac{1}{\abs{\widehat{\G}}}\sum_{i\in\widehat{\G}}{\norm{g_i^{\batchsize} - \nabla}}\; .
\]
Since $g_{\text{med}} = \widebar{g}_i^{\batchsize}$ for some $i\in\M$, by the definition of $\M$, there are more than $\nicefrac{m}{2}$ machines whose distance from $g_{\text{med}}$ is bounded by $\nicefrac{\T_p^{\batchsize}}{2}$. Since there are at most $\delta m<\nicefrac{m}{2}$ Byzantine workers, it implies that at least one of the above workers is good; denote one such worker by $\ell\in\G$. Thus, for every $i\in\widehat{\G}$ we can bound,
\begin{align*}
    \norm{\widebar{g}_i^{\batchsize} - \nabla} &\leq \underbrace{\norm{\widebar{g}_i^{\batchsize} - g_{\text{med}}}}_{\leq\T_p^{\batchsize}} + \underbrace{\norm{g_{\text{med}} - \widebar{g}_{\ell}^{\batchsize}}}_{\leq\nicefrac{\T_p^{\batchsize}}{2}} + \underbrace{\norm{\widebar{g}_{\ell}^{\batchsize} - \nabla}}_{\leq\V} \leq \frac{3\T_p^{\batchsize}}{2} + \V\; ,
\end{align*}
where the first bound is by the definition of $\widehat{\G}$, the second bound is a consequence of the chosen $\ell$ (and the definition of $g_{\text{med}}$), and the final bound is trivial for any good worker. Since this bound holds for any $i\in\widehat{\G}$, it also holds for $\widehat{g}^{\batchsize}$. Thus, using $(a+b)^2\leq 2a^2 + 2b^2$, we have
\begin{align*}
    \norm{\widehat{g}^{\batchsize} - \nabla}^2 &\leq 2\brac{\frac{9}{4}(\T_p^{\batchsize})^2 + \V^2} = \frac{9}{2}(\T_p^{\batchsize})^2 + 2\V^2\; .
\end{align*}
Overall, in any case, we have that
\begin{align*}
    \norm{\widehat{g}^{\batchsize} - \nabla}^2 \leq \max\cbrac{\norm{\nabla}^2, \frac{9}{2}(\T_p^{\batchsize})^2 + 2\V^2} &\leq \frac{9}{2}(\T_p^{\batchsize})^2 + 2\V^2 + \norm{\nabla}^2 \\ &= \frac{9}{2}\cdot\frac{4C_p^2\V^2}{\batchsize} + 2\V^2 + \norm{\nabla}^2 \\ &\leq 20C_p^2\V^2 + \norm{\nabla}^2\; ,
\end{align*}
where we used $C_p^2\geq 1$. This concludes the first part. 

For the second part, we denote the average of honest workers by $\widebar{g}^{\batchsize} \coloneqq \frac{1}{\abs{\G}}\sum_{i\in\G}{\widebar{g}_i^{\batchsize}}$ and define the following events:
\begin{align*}
    &\widehat{\C}\coloneqq\cbrac{\norm{\widehat{g}^{\batchsize} - \nabla}\leq 2\V\sqrt{\frac{\log{\brac{2/p}}}{m \batchsize}} + 4\delta\T_p^{\batchsize}}, \\
    &\widebar{\C}\coloneqq \cbrac{\norm{\widebar{g}^{\batchsize} - \nabla}\leq 2\V\sqrt{\frac{\log{\brac{2/p}}}{m \batchsize}}}\; .
\end{align*}
Our objective is to establish that $\widehat{\C}$ holds with probability at least $1-2p$. Recall that under the event $\B$ (\Cref{eq:eventB}), the set $\M$ is non-empty (item 1 of \Cref{lem:mfm_helper}), implying $\widehat{g}^{\batchsize} = \frac{1}{\abs{\widehat{\G}}}\sum_{i\in\widehat{\G}}{\widebar{g}_i^{\batchsize}}$, and $\G\subseteq\widehat{\G}$. Hence, under the event $\B$, we have that 
\begin{align}\label{eq:mfm_hp_helper}
    \norm{\widehat{g}^{\batchsize} - \nabla} = \norm{\frac{1}{\abs{\widehat{\G}}}\sum_{i\in\widehat{\G}}{({\widebar{g}_i^{\batchsize}} - \nabla})} &= \norm{\frac{1}{\abs{\widehat\G}}\sum_{i\in\G}{(\widebar{g}_i^{\batchsize} - \nabla)} + \frac{1}{\abs{\widehat{\G}}}\sum_{i\in\widehat{\G}\setminus\G}{(\widebar{g}_i^{\batchsize} - \nabla)}} \nonumber \\ &\leq \frac{\abs{\G}}{\abs{\widehat{\G}}}\norm{\widebar{g}^{\batchsize} - \nabla} + \frac{1}{\abs{\widehat{\G}}}\sum_{i\in\widehat{\G}\setminus\G}{\norm{\widebar{g}_i^{\batchsize} - \nabla}} \nonumber\\ &\leq \norm{\widebar{g}^{\batchsize} - \nabla} + \frac{2}{m}\abs{\widehat{\G}\setminus\G}\cdot 2\T_p^{\batchsize} \nonumber \\ &\leq \norm{\widebar{g}^{\batchsize} - \nabla} + 4\delta\T_p^{\batchsize}\; ,
\end{align}

where the first inequality follows from the triangle inequality and the fact that $\sum_{i\in\G}{(\widebar{g}_i^{\batchsize} - \nabla)} = \abs{\G}(\widebar{g}^{\batchsize} - \nabla)$; the second inequality is
due to $\abs{\widehat{\G}} \geq \abs{\G} \geq\nicefrac{m}{2}$ and item 3 of \Cref{lem:mfm_helper}, namely, $\lVert{\widebar{g}_i^{\batchsize} - \nabla\rVert}\leq 2\T_p^{\batchsize}$ for all $i\in\widehat{\G}$; and the last inequality follows from $\abs{\widehat{\G}\setminus\G}\leq \abs{\sbrac{m}\setminus\G}=\delta m$. 

Based on \Cref{eq:mfm_hp_helper}, we infer that, under the event $\B$, if $\widebar{\C}$ occurs, then so does $\widehat{\C}$, implying $\prob(\widehat{\C}|\B)\geq \prob(\widebar{\C}|\B)$. Furthermore, we have by \Cref{lem:B_has_high_probability} that $\prob(\B) \geq 1-p > 0$. Combining these properties and utilizing \Cref{lem:3events} yields: 
\begin{equation}\label{eq:bound_prob_C}
    \prob(\widehat{\C}) \geq\prob(\widebar{\C}) - \prob(\B^c) \geq \prob(\widebar{\C}) - p\; .
\end{equation}
By \Cref{lem:concentration}, it holds with probability at least $1-p$ that 
\[
    \norm{\widebar{g}^{\batchsize} - \nabla} = \frac{1}{\abs{\G}\batchsize}\norm{\sum_{i\in\G}{\sum_{n\in\sbrac{\batchsize}}{\brac{\nabla F(x; \xi_{i}^{n}) - \nabla}}}} \leq \V\sqrt{\frac{2\log{\brac{2/p}}}{\abs{\G}\batchsize}} \leq 2\V\sqrt{\frac{\log{\brac{2/p}}}{m\batchsize}}\; .
\]
In other words, $\prob(\widebar{\C})\geq 1-p$, indicating, as per \Cref{eq:bound_prob_C}, that with a probability of at least $1-2p$,
\begin{align*}
    \norm{\widehat{g}^{\batchsize} - \nabla} &\leq 2\V\sqrt{\frac{\log{\brac{2/p}}}{m \batchsize}} + 4\delta\T_p^{\batchsize} \!=\! 2\V\sqrt{\frac{\log{\brac{2/p}}}{m \batchsize}} + 16\delta\V\sqrt{\frac{2\log\brac{2m/p}}{\batchsize}} \!=\! 2\V\sqrt{\frac{\log\brac{2m/p}}{\batchsize}}\brac{\frac{1}{\sqrt{m}} + 8\sqrt{2}\delta}\;.
\end{align*}
This result finally implies that with probability at least $1-p$,
\begin{align*}
    \norm{\widehat{g}^{\batchsize} - \nabla}^2 &\leq \frac{4\V^2\log\brac{4m/p}}{\batchsize}\brac{\frac{1}{\sqrt{m}} + 8\sqrt{2}\delta}^2 \leq \frac{8\V^2\log\brac{4m/p}}{\batchsize}\brac{\frac{1}{m} + 128\delta^2}\leq \frac{2C_p^2\V^2}{\batchsize}\brac{\frac{1}{m} + 128\delta^2}\; ,
\end{align*}
where we used $(a+b)^2\leq 2(a^2+b^2)$ and $8\log\brac{4m/p}\leq 16\log\brac{4m/p} = 2C_p^2$.
\end{proof}
Combining the bounds established in \Cref{lem:core_lemma_mfm}, we can derive an upper bound on the expected (squared) aggregation error, mirroring \Cref{cor:mse_bound_general}. 

\begin{corollary}[MSE of Aggregated Gradients]\label{cor:mse_bound_mfm}
    Consider the setting in \Cref{lem:core_lemma_general}, let $\widehat{g}^{\batchsize}\gets\text{MFM}(\widebar{g}_{1}^{\batchsize}, \ldots, \widebar{g}_{m}^{\batchsize}; \T^{\batchsize})$ be the output of \Cref{alg:mfm} with $\T^{\batchsize}\coloneqq 2C\V/\sqrt{N} = 4\V\sqrt{2\log{\brac{16m^2 T}}/\batchsize}$, and assume that $N\leq T$. Then,
    \begin{equation*}
        \E\!\norm{\widehat{g}^{\batchsize} - \nabla f(x)}^2 \leq \frac{\tilde{C}^2\V^2\tilde{\gamma}}{\batchsize}+ \frac{\norm{\nabla f(x)}^2}{8mT}\; ,
    \end{equation*}
    where $\tilde{C}^2\coloneqq 8C^2 = 64\log\!\brac{16m^2 T}$ and $\tilde{\gamma}\coloneqq \frac{1}{m} + 32\delta^2$.
\end{corollary}
\begin{proof}
    Utilizing the previous lemma, we have with probability at least $1 - \frac{1}{8mT}$,
    \[
        \norm{\widehat{g}^{\batchsize} - \nabla f(x)}^2 \leq \frac{2C^2\V^2}{\batchsize}\brac{\frac{1}{m} + 128\delta^2}\; .
    \]
    Additionally, we always have that 
    \begin{align*}
        \norm{\widehat{g}^{\batchsize} - \nabla f(x)}^2 \leq 20C^2\V^2 + \norm{\nabla f(x)}^2\; .
    \end{align*}
    Thus, by the law of total expectation, we get
    \begin{align*}
        \E\!\norm{\widehat{g}^{\batchsize} - \nabla f(x)}^2 &\leq \frac{2C^2\V^2}{\batchsize}\brac{\frac{1}{m} + 128\delta^2} + \brac{20C^2\V^2 + \norm{\nabla f(x)}^2}\cdot\frac{1}{8mT} \\ &= \frac{2C^2\V^2}{\batchsize}\brac{\frac{1}{m} + 128\delta^2} +\frac{5C^2 \V^2}{2mT} + \frac{\norm{\nabla f(x)}^2}{8mT} \\ &= \frac{C^2\V^2}{\batchsize}\brac{\frac{9}{2m} + 256\delta^2} + \frac{\norm{\nabla f(x)}^2}{8mT} \\ &\leq \frac{8C^2\V^2}{\batchsize}\brac{\frac{1}{m} + 32\delta^2} + \frac{\norm{\nabla f(x)}^2}{8mT} \\ &= \frac{\tilde{C}^2 \V^2\tilde{\gamma}}{N} + \frac{\norm{\nabla f(x)}^2}{8mT}\; ,
    \end{align*}
    where the second inequality follows from the assumption that $N\leq T$.
\end{proof}

The bound in \Cref{cor:mse_bound_mfm} closely resembles that in \Cref{cor:mse_bound_general}, differing by an additional factor of $\frac{\norm{\nabla f(x)}^2}{8mT}$. This variation stems from the specific structure of the MFM aggregator, specifically due to the rare event where the set $\M$ is empty, leading to an output of $\widehat{g}=0$.

\subsection{MFM is not $(\delta, \kappa)$-robust}\label{subapp:mfm_is_not_delta_kappa_robust}
Consider \Cref{alg:mfm} with a threshold $\T>0$, and let $x\in\reals^d$. Suppose every honest worker $i\in\G$ provides the true gradient, $g_i = \nabla f(x)$, while every Byzantine worker submits $g_i = \nabla f(x) + \frac{3}{4}\T v$, for some $v\in\reals^d$ with $\norm{v}=1$. In this case, $\M$ is not empty as $\M=\G$. Since all vectors are within $\frac{3}{4}\T$ of each other, $\widehat{\G}$ contains all workers. Consequently, the aggregated gradient, $\widehat{g}$, is given by $\nabla f(x) + \frac{3}{4}\T\delta v$. Denoting the average of honest workers by $\widebar{g}=\frac{1}{\abs{\G}}\sum_{i\in\G}{g_i} = \nabla f(x)$, the above implies a nonzero aggregation error $\norm{\widehat{g} - \widebar{g}}$, while the `variance' among honest workers, $\frac{1}{\abs{\G}}\sum_{i\in\G}{\norm{g_i - \widebar{g}}^2}$, remains zero. This scenario fails to satisfy \Cref{def:byz-robustness}. 
\section{AdaGrad with Biased Gradients}\label{app:biased-adagrad}
Consider the AdaGrad-Norm~\citep{levy2017online, ward2020adagrad,faw2022power} (also known as AdaSGD,~\citealt{attia2023sgd}) update rule, defined for some parameter $\eta_0>0$ as follows:
\begin{equation}\label{eq:adagrad-norm}
    x_{t+1} = \proj{\K}{x_t - \eta_t g_t}, \quad \eta_t = \frac{\eta_0}{\sqrt{\sum_{s=1}^{t}{\norm{g_s}^2}}}\; , \tag{AdaGrad-Norm}
\end{equation}
where $g_t$ has bias $b_t\coloneqq \E[g_t - \nablat | x_t]$ and variance $V_t^2\coloneqq\E[\norm{g_t - \E g_t}^2|x_t]$.

In \Cref{subapp:convex-adagrad,subapp:nonconvex-adagrad}, convergence bounds for \eqref{eq:adagrad-norm} are deduced for convex and non-convex objectives, respectively.

\subsection{Convex Analysis}\label{subapp:convex-adagrad}
We commence with a lemma that establishes a second-order bound on the linearized regret of \eqref{eq:adagrad-norm}, essential in our convex analysis. The proof is included for completeness.

\begin{lemma}[\citealp{levy2017online}, Theorem 1.1]\label{lem:convex-adagrad-helper}
    Suppose \Cref{assump:bounded-domain} holds, i.e., the domain $\K$ is bounded with diameter $D\coloneqq \max_{x,y\in\K}{\norm{x-y}}$ and consider \eqref{eq:adagrad-norm}. Then, for every $u\in\K$, the iterates $x_1,\ldots,x_T$ satisfy:
    \[
        \sum_{t=1}^{T}{g_t^\top(x_t - u)} \leq \brac{\frac{D^2}{2\eta_0} + \eta_0}\sqrt{\sum_{t=1}^{T}{\norm{g_{t}}^2}}\; .
    \]
\end{lemma}
\begin{proof}
    For every $u\in\K$, we have that
    \begin{align*}
        \norm{x_{t+1} - u}^2 &\leq \norm{x_{t} - u}^2 - 2\eta_t g_t^\top (x_t - u) + \eta^2\norm{g_t}^2\; .
    \end{align*}
    Rearranging terms, we get:
    \begin{align*}
        g_t^\top(x_t - u) &\leq \frac{\norm{x_t - u}^2 - \norm{x_{t+1} - u}^2}{2\eta_t} + \frac{\eta_t}{2}\norm{g_t}^2\; .
    \end{align*}
    Summing over $t\in\sbrac{T}$, we then obtain:
    \begin{align*}
        \sum_{t=1}^{T}{g_t^\top(x_t - u)} &\leq \frac{\norm{x_1 -u}^2}{2\eta_1} + \sum_{t=2}^{T}{\frac{\norm{x_t - u}^2}{2}\brac{\frac{1}{\eta_t} - \frac{1}{\eta_{t-1}}}} + \frac{1}{2}\sum_{t=1}^{T}{\eta_t \norm{g_t}^2} \\ &\leq \frac{D^2}{2}\brac{\frac{1}{\eta_1} + \sum_{t=2}^{T}{\brac{\frac{1}{\eta_t} - \frac{1}{\eta_{t-1}}}}} + \frac{\eta_0}{2}\sum_{t=1}^{T}{\frac{\norm{g_t}^2}{\sqrt{\sum_{s=1}^{t}{\norm{g_{s}}^2}}}} \\ &= \frac{D^2}{2\eta_T} + \frac{\eta_0}{2}\sum_{t=1}^{T}{\frac{\norm{g_t}^2}{\sqrt{\sum_{s=1}^{t}{\norm{g_s}^2}}}} \\ &\leq \frac{D^2}{2\eta_0}\sqrt{\sum_{t=1}^{T}{\norm{g_t}^2}} + \eta_0\sqrt{\sum_{t=1}^{T}{\norm{g_t}^2}} \\ &=\brac{\frac{D^2}{2\eta_0} + \eta_0}\sqrt{\sum_{t=1}^{T}{\norm{g_{t}}^2}}\; ,
    \end{align*}
    where the second inequality uses $\norm{x_t - u}^2\leq D^2$ for every $t\in\sbrac{T}$ and $\eta_t\leq \eta_{t-1}$, and the final inequality stems from \Cref{lem:sum_sqrt_lemma}.
\end{proof}

We now establish a regret bound for AdaGrad-Norm.
\begin{lemma}\label{lem:convex-adagrad}
    Assume $f$ is convex. Under \Cref{assump:bounded-domain}, the iterates of \eqref{eq:adagrad-norm} satisfy:
    \begin{equation*}
        \E\mathcal{R}_T \leq D\sqrt{2\sum_{t\in\sbrac{T}}{\E V_t^2}} + 2D\sqrt{\sum_{t\in\sbrac{T}}{\E\!\norm{b_t}^2}} + D\sum_{t\in\sbrac{T}}{\E\!\norm{b_t}} + 2D\sqrt{\sum_{t\in\sbrac{T}}{\E\!\norm{\nablat}^2}}\; ,
    \end{equation*}
    where $\mathcal{R}_T\coloneqq\sum_{t\in\sbrac{T}}{\brac{f(x_t) - f(x^*)}}$.
\end{lemma}
\begin{proof}
    By the convexity of $f$,
    \begin{equation}\label{eq:bias-var-decomp-adaptive}
        \E \mathcal{R}_T \leq \sum_{t\in\sbrac{T}}{\E[\nablat^\top(x_t - x^*)]} = \underbrace{\E\!\sbrac{\sum_{t\in\sbrac{T}}{g_t^{\top}(x_t - x^*)}}}_{=(A)} + \underbrace{\sum_{t\in\sbrac{T}}{\E[-b_t^\top(x_t - x^*)]}}_{=(B)}\; .
    \end{equation}
    \paragraph{Bounding $(A)$. } Utilizing \Cref{lem:convex-adagrad-helper} with $\eta_0 = D/\sqrt{2}$ and Jensen's inequality, we obtain:
    \begin{align*}
        \E\!\sbrac{\sum_{t\in\sbrac{T}}{g_t^{\top}(x_t - x^*)}} &\leq D\E\!\sbrac{\sqrt{2\sum_{t\in\sbrac{T}}{\norm{g_t}^2}}} \leq D\sqrt{2\sum_{t\in\sbrac{T}}{\E\!\norm{g_t}^2}}\; .
    \end{align*}
    We can bound the second moment of $g_t$ as follows:
    \begin{equation}\label{eq:second_moment_bound_bias_var}
        \E_{t-1}\!\norm{g_t}^2 = \E_{t-1}\!\norm{g_t - \E_{t-1}g_t}^2 + \norm{\E_{t-1} g_t}^2 \leq V_t^2 + 2\norm{b_t}^2 + 2\norm{\nablat}^2\; ,
    \end{equation}
    where we used $\norm{a+b}^2\leq 2\norm{a}^2 + 2\norm{b}^2$. Plugging this bound back, we get that $(A)$ is bounded as 
    \begin{align*}
        \E\!\sbrac{\sum_{t\in\sbrac{T}}{g_t^{\top}(x_t - x^*)}} \leq D\sqrt{2\sum_{t\in\sbrac{T}}{\E V_t^2}} + 2D\sqrt{\sum_{t\in\sbrac{T}}{\E\!\norm{b_t}^2}} + 2D\sqrt{\sum_{t\in\sbrac{T}}{\E\!\norm{\nablat}^2}}\; .
    \end{align*}

    \paragraph{Bounding $(B)$. } By Cauchy-Schwarz inequality and \Cref{assump:bounded-domain},
    \begin{align*}
        \sum_{t\in\sbrac{T}}{\E[-b_t^\top(x_t - x^*)]} \leq \sum_{t\in\sbrac{T}}{\E[\norm{b_t}\norm{x_t - x^*}]} \leq D\sum_{t\in\sbrac{T}}{\E\!\norm{b_t}}\; .
    \end{align*}
    Incorporating the bounds on $(A)$ and $(B)$ into \Cref{eq:bias-var-decomp-adaptive} concludes the proof, as
    \begin{align*}
        \E\mathcal{R}_T &\leq  D\sqrt{2\sum_{t\in\sbrac{T}}{\E V_t^2}} + 2D\sqrt{\sum_{t\in\sbrac{T}}{\E\!\norm{b_t}^2}} + D\sum_{t\in\sbrac{T}}{\E\!\norm{b_t}} + 2D\sqrt{\sum_{t\in\sbrac{T}}{\E\!\norm{\nablat}^2}}\; .
    \end{align*}
\end{proof}

\subsection{Non-Convex Analysis}\label{subapp:nonconvex-adagrad}

The subsequent lemma establishes an upper bound on the sum of squared gradient norms when utilizing AdaGrad-Norm for bounded functions.
\begin{lemma}\label{lem:nonconvex-adagrad-helper}
    Assume $f$ is bounded by $M$, i.e., $\max_{x}{\abs{f(x)}}\leq M$, and consider \eqref{eq:adagrad-norm} with $\K=\reals^{d}$. The iterates $x_1,\ldots,x_T$ satisfy:
    \begin{equation*}
        \sum_{t=1}^{T}{\norm{\nablat}^2} \leq \brac{\frac{2M}{\eta_0} + \eta_0 L}\sqrt{\sum_{t=1}^{T}{\lVert g_t\rVert}^2} + \sum_{t=1}^{T}{\brac{\nablat - g_t}^\top\nablat}\; .
    \end{equation*}
\end{lemma}
\begin{proof}
    By the smoothness of $f$, for every $x,y\in\reals^d$ we have that $f(y)\leq f(x) + \nabla f(x)^\top(y-x) + \frac{L}{2}\norm{y-x}^2$. Plugging-in the update rule $x_{t+1} = x_t - \eta_t g_t$, we obtain:
    \begin{align*}
        f(x_{t+1}) &\leq f(x_t) - \eta_t \nablat^\top g_t + \frac{L\eta_t^2}{2}\norm{g_t}^2 \\ &= f(x_t) - \eta_t \norm{\nablat}^2 + \eta_t \brac{\nablat - g_t}^\top\nablat + \frac{L\eta_t^2}{2}\norm{g_t}^2\; .
    \end{align*}
    Rearranging terms and dividing by $\eta_t$ gives:
    \begin{equation*}
        \norm{\nablat}^2 \leq \frac{f(x_t) - f(x_{t+1})}{\eta_t} + \frac{L\eta_t}{2}\norm{g_t}^2 + \brac{\nablat - g_t}^\top\nablat\; .
    \end{equation*}
    Denoting $\Delta_t \coloneqq f(x_t) - f^*$ and $\Delta_{\max} \coloneqq \max_{t\in\sbrac{T}}{f(x_t) - f^*}$. Thus, the above is equivalent to  
    \begin{equation*}
        \norm{\nablat}^2 \leq \frac{\Delta_t - \Delta_{t+1}}{\eta_t} + \frac{L\eta_t}{2}\norm{g_t}^2 + \brac{\nablat - g_t}^\top\nablat\; .
    \end{equation*}
    Summing over $t=1,\ldots,T$, we obtain that 
    \begin{align}\label{eq:lem_helper}
        \sum_{t\in\sbrac{T}}{\norm{\nablat}^2} &\leq \sum_{t\in\sbrac{T}}{\frac{\Delta_t - \Delta_{t+1}}{\eta_t}} + \frac{L}{2}\sum_{t\in\sbrac{T}}{\eta_t\norm{g_t}^2} + \sum_{t\in\sbrac{T}}{\brac{\nablat - g_t}^\top\nablat} \nonumber \\ &\leq \frac{\Delta_{\max}}{\eta_T} + \frac{L}{2}\sum_{t\in\sbrac{T}}{\eta_t\norm{g_t}^2} + \sum_{t\in\sbrac{T}}{\brac{\nablat - g_t}^\top\nablat} \; ,
    \end{align}
    where the last inequality follows from:
    \begin{equation*}
        \sum_{t=1}^{T}{\frac{\Delta_t - \Delta_{t+1}}{\eta_t}} = \frac{\Delta_1}{\eta_1} + \sum_{t=2}^{T}{\brac{\frac{1}{\eta_t} - \frac{1}{\eta_{t-1}}}\Delta_t} - \frac{\Delta_{T+1}}{\eta_{T}} \leq \frac{\Delta_1}{\eta_1} + \sum_{t=2}^{T}{\brac{\frac{1}{\eta_t} - \frac{1}{\eta_{t-1}}}\Delta_t} \leq \frac{\Delta_{\max}}{\eta_T}\; ,
    \end{equation*}
    which holds as $\eta_{t} = \eta_0(\sum_{\tau=1}^{t}{\norm{g_{\tau}}^2})^{-1/2}$ is non-increasing. Next, we can bound the second term in the R.H.S of \Cref{eq:lem_helper} using \Cref{lem:sum_sqrt_lemma} with $a_i = \lVert g_i\rVert^2$ as follows: 
    \begin{align*}
        \sum_{t\in\sbrac{T}}{\eta_t \norm{g_t}^2} &= \eta_0\sum_{t\in\sbrac{T}}{\frac{\norm{g_t}^2}{\sqrt{\sum_{s=1}^{t}{\norm{g_{s}}^2}}}} \leq 2\eta_0\sqrt{\sum_{t\in\sbrac{T}}{\norm{g_t}^2}}\; .
    \end{align*}
    Injecting this bound and $\eta_T$ back into \Cref{eq:lem_helper} and considering that $\Delta_{\max}\leq 2M$ concludes the proof.
\end{proof}

Leveraging \Cref{lem:nonconvex-adagrad-helper}, we derive the following bound, instrumental in proving \Cref{thm:nonconvex-adaptive}.
\begin{lemma}\label{lem:nonconvex-adagrad}
    Assume $f$ is bounded by $M$ and consider \eqref{eq:adagrad-norm} with $\K=\reals^{d}$. Denote:
    \begin{align*}
        G_T^2\coloneqq \sum_{t=1}^{T}{\norm{\nablat}^2}, \quad V_{1:T}^2\coloneqq \sum_{t=1}^{T}{V_t^2}, \quad S_{T}^2\coloneqq\sum_{t=1}^{T}{\norm{b_t}^2}\; .
    \end{align*}
    For every $T\geq 1$, it holds that
    \begin{equation*}
        \E G_T^2 \leq 2\zeta\brac{\sqrt{\E V_{1:T}^2} + \sqrt{2\E S_{T}^2} + \sqrt{2\E G_T^2}} + \E S_{T}^2\; ,
    \end{equation*}
    where $\zeta\coloneqq \frac{2M}{\eta_0} + \eta_0 L$.
\end{lemma}
\begin{proof}
    Employing \Cref{lem:nonconvex-adagrad-helper}, we get
    \begin{align*}
        \E G_T^2 &\leq \zeta\underbrace{\E\!\sbrac{\sqrt{\sum_{t\in\sbrac{T}}{\norm{g_t}^2}}}}_{=(A)} + \underbrace{\sum_{t\in\sbrac{T}}{\E[-b_t^\top \nablat]}}_{=(B)}\; .
    \end{align*}
    \paragraph{Bounding $(A)$. } We apply a technique akin to the one utilized in the proof of \Cref{lem:convex-adagrad}. Concretely, applying \Cref{eq:second_moment_bound_bias_var} and using Jensen's inequality, yields
    \begin{align*}
        \E\!\sbrac{\sqrt{\sum_{t\in\sbrac{T}}{\norm{g_t}^2}}} \leq \sqrt{\sum_{t\in\sbrac{T}}{\brac{\E V_t^2 + 2\E\!\norm{b_t}^2 + 2\E\!\norm{\nablat}^2}}} \leq \sqrt{\E V_{1:T}^2} + \sqrt{2\E S_{T}^2} + \sqrt{2\E G_T^2}\; .
    \end{align*}
    \paragraph{Bounding $(B)$. } Employing Young's inequality, namely, $a^\top b\leq \frac{1}{2}\norm{a}^2 + \frac{1}{2}\norm{b}^2$, results in
    \begin{align*}
        \sum_{t\in\sbrac{T}}{\E[-b_t^\top \nablat]} \leq \frac{1}{2}\sum_{t\in\sbrac{T}}{\E\!\norm{b_t}^2} + \frac{1}{2}\sum_{t\in\sbrac{T}}{\E\!\norm{\nablat}^2} = \frac{1}{2}\E S_{T}^2 + \frac{1}{2}\E G_T^2\; .
    \end{align*}

    Substituting the bounds on $(A)$ and $(B)$ implies that
    \begin{align*}
        \E G_T^2 \leq \zeta\brac{\sqrt{\E V_{1:T}^2} + \sqrt{2\E S_{T}^2} + \sqrt{2\E G_T^2}} + \frac{1}{2}\E S_{T}^2 + \frac{1}{2}\E G_T^2\; .
    \end{align*}
    Subtracting $\frac{1}{2}\E G_T^2$ and multiplying by $2$ establishes the result.
\end{proof}
\section{Dynamic and Adaptive Byzantine-Robustness with MFM and AdaGrad}\label{app:dynamic_adaptive}
Following the methodology outlined in \Cref{app:dynamic_general}, this section focuses on analyzing \Cref{alg:method-new} with \textcolor{purple}{\textbf{Option 2}}. Here, we replace the general $(\delta, \kappa)$-robust aggregator with the MFM aggregator and incorporate the adaptive AdaGrad learning rate~\cite{levy2018online,ward2020adagrad,attia2023sgd}. 

Recall that \Cref{alg:method-new} with \textcolor{purple}{\textbf{Option 2}} and learning rate defined in \Cref{eq:adagrad} performs the following update rule for every $t\in\sbrac{T}$:
\begin{align}
    & J_t\sim\text{Geom}(\nicefrac{1}{2}) \nonumber \\
    & \widehat{g}_t^{j}\gets \text{MFM}(\widebar{g}_{t,1}^{j},\ldots,\widebar{g}_{t,m}^{j}; \T^{j}), \quad \T^j\coloneqq \frac{2C\V}{\sqrt{2^j}} = 4\V\sqrt{\frac{2\log\brac{16m^2 T}}{2^j}} \label{eq:mfm_aggregated_grad} \\
    & g_t \gets  \widehat{g}_t^{0} + \begin{cases}
        2^J_t\brac{\widehat{g}_t^{J_t} - \widehat{g}_t^{J_t-1}}, &\text{if } J_t\leq\Jmax\coloneqq\floor{\log{T}} \text{ and } \Ecal_t(J_t) \text{ holds} \\
        0, &\text{otherwise}
    \end{cases} \label{eq:mlmc_agg_mfm} \\
    & x_{t+1} \gets \proj{\K}{x_t - \eta_t g_t }, \quad \eta_t = \frac{\eta_0}{\sqrt{\sum_{s=1}^{t}{\norm{g_s}^2}}}\; ,  \label{eq:adagrad_update}
\end{align}
where the associated event $\Ecal_t(J_t)$ in this case is defined as,
\begin{equation}\label{eq:event_E_option2}
    \Ecal_t(J_t)\coloneqq \cbrac{\lVert \widehat{g}_t^{J_t} - \widehat{g}_t^{J_t-1}\rVert\leq (1 + \sqrt{2})\frac{c_{\Ecal}C\V}{\sqrt{2^{J_t}}}}, \quad c_{\Ecal} \coloneqq 6\sqrt{2}, \quad \tilde{C}\coloneqq 2\sqrt{2}C=8\sqrt{\log\brac{4m^2 T}}\; .
\end{equation}

For ease of writing, we denote: $\tilde{\gamma}\coloneqq \frac{1}{m} + 32\delta^2$.

We start by showing that $\Ecal_t$ is satisfied with high probability, mirroring \Cref{lem:Ecal_hp_option1}.

\begin{lemma}\label{lem:Ecal_hp_option2}
    Consider $\Ecal(J_t)$ defined in \Cref{eq:event_E_option2}. For every $t\in\goodrounds$ and $j=0,\ldots,\Jmax$, we have 
    \[
        \prob_{t-1}(\Ecal_t(j))\geq 1 - \frac{1}{4mT}\; ,    
    \]
    where the randomness is w.r.t the stochastic gradient samples.
\end{lemma}
\begin{proof}
     By item 2 of \Cref{lem:core_lemma_mfm}, for every $j=0,\ldots,\Jmax$ (separately), we have with probability at least $1-\frac{1}{8mT}$,
    \begin{align*}
        \lVert{\widehat{g}_t^{j} - \nabla_t}\rVert^2 &\leq \frac{2C^2\V^2}{2^{j}}\brac{\frac{1}{m} + 128\delta^2} \leq \frac{2C^2\V^2}{2^{j}}\brac{\frac{4}{m} + 128\delta^2} = \frac{8C^2\V^2\tilde{\gamma}}{2^{j}} = \frac{\tilde{C}^2\V^2\tilde{\gamma}}{2^{j}}\; ,
    \end{align*}
    where we used $\frac{1}{m} + 128\delta^2\leq 4\brac{\frac{1}{m} + 32\delta^2}=4\tilde{\gamma}$ and $\tilde{C}^2 = 8C^2$. Therefore, we have
    \begin{align*}
        \prob_{t-1}\!\brac{\lVert \widehat{g}_t^{j} - \nabla_t\rVert> \frac{c_{\Ecal}C\V}{\sqrt{2^j}}} = \prob_{t-1}\!\brac{\lVert \widehat{g}_t^{j} - \nabla_t\rVert> \frac{3\tilde{C}\V}{\sqrt{2^j}}} &\leq \prob_{t-1}\!\brac{\lVert \widehat{g}_t^{j} - \nabla_t\rVert> \frac{\tilde{C}\V\sqrt{\tilde{\gamma}}}{\sqrt{2^j}}} \leq \frac{1}{8mT}\; ,
    \end{align*}
    where we used $c_{\Ecal}C=6\sqrt{2}C = 3\tilde{C}$, and $\tilde{\gamma} = \frac{1}{m} + 32\delta^2\leq 9$ as $\delta<1/2$.
    This bound, in conjunction with the union bound, allows us to bound $\prob_{t-1}(\Ecal_t(j)^c)$ as,
    \begin{align*}
        \prob_{t-1}(\Ecal_t(j)^c) &= \prob_{t-1}\brac{\lVert \widehat{g}_t^{j} - \widehat{g}_t^{j-1}\rVert> (1 + \sqrt{2})\frac{c_{\Ecal}C\V}{\sqrt{2^j}}} \\ &\leq\prob_{t-1}\brac{\cbrac{\lVert \widehat{g}_t^{j} - \nabla_t\rVert> \frac{c_{\Ecal}C\V}{\sqrt{2^j}}}\bigcup \cbrac{\lVert \widehat{g}_t^{j-1} - \nabla_t \rVert> \frac{c_{\Ecal}C\V}{\sqrt{2^{j-1}}}}} \\ &\leq \prob_{t-1}\brac{\lVert \widehat{g}_t^{j} - \nabla_t\rVert> \frac{c_{\Ecal}C\V}{\sqrt{2^j}}} + \prob_{t-1}\brac{\lVert \widehat{g}_t^{j-1} - \nabla_t\rVert> \frac{c_{\Ecal}C\V}{\sqrt{2^{j-1}}}} \\ &\leq \frac{1}{8mT} + \frac{1}{8mT} = \frac{1}{4mT}\; .
    \end{align*}
\end{proof}

Next, using \Cref{cor:mse_bound_mfm} and \Cref{lem:Ecal_hp_option2}, we establish bounds on the bias and variance of the MLMC estimator defined in \Cref{eq:mlmc_agg_mfm}, resembling those outlined in \Cref{lem:bias-var-mlmc-general}.

\begin{lemma}[MLMC Bias and Variance]\label{lem:bias-var-mlmc-mfm}
    Consider $g_t$ defined as in \Cref{eq:mlmc_agg_mfm}. Then, for every $m\geq 2$,
    \begin{enumerate}
        \item The bias $b_t\coloneqq \E_{t-1} g_t - \nabla_t$ is bounded as 
        \begin{align*}
            \norm{b_t} \leq \begin{cases}
                \tilde{C}\V\sqrt{\frac{2\tilde{\gamma}}{T}} + \frac{\sqrt{5}\tilde{C}\V\log{T}}{mT} + \frac{(1+\sqrt{2})\norm{\nabla_t}}{2\sqrt{2mT}}, &t\in\goodrounds \\ 
                \sqrt{125\log{T}}\tilde{C}\V + \frac{\norm{\nabla_t}}{2\sqrt{mT}}, &t\in\badrounds
            \end{cases}\; .
        \end{align*}
        \item The variance $V_t^2\coloneqq \E_{t-1}\norm{g_t - \E_{t-1} g_t}^2$ is bounded as 
        \begin{align*}
            V_t^2 \leq \begin{cases}
                14\tilde{C}^2\V^2\tilde{\gamma}\log{T} + \log{T}\norm{\nabla_t}^2, &t\in\goodrounds \\ 
                125\tilde{C}^2\V^2\log{T} + \frac{\norm{\nabla_t}^2}{4mT}, &t\in\badrounds
            \end{cases}\; .
        \end{align*}
    \end{enumerate}
\end{lemma}

\begin{proof}
    Our proof technique parallels that of \Cref{lem:bias-var-mlmc-general}. Starting in a similar fashion, we can bound the variance as shown in \Cref{eq:mlmc_var_explicit}, namely, 
    \begin{align}\label{eq:mlmc_var_explicit_mfm}
        V_t^2 \leq \E_{t-1}\norm{g_t - \nabla_t}^2 \leq 2\E_{t-1}\norm{\widehat{g}_t^{0} - \nabla_t}^2 + 2\sum_{j=1}^{\Jmax}{2^{j}\underbrace{\E_{t-1}\sbrac{\lVert{\widehat{g}_t^{j} - \widehat{g}_t^{j-1}}\rVert^2\mathbbm{1}_{\Ecal_t(j)}}}_{=(\dag)}}\; .
    \end{align}
    Unlike in \Cref{lem:bias-var-mlmc-general}, here we bound $(\dag)$ differently for $t\in\badrounds$ and $t\in\goodrounds$. This is because the bound within the event $\Ecal_t$ in \Cref{eq:event_E_option2} deviates from that in \Cref{eq:event_E_option1} by a factor of $\sqrt{\gamma}$. Alternatively, one could introduce a factor of $\sqrt{\tilde{\gamma}}$ to maintain a similar analysis; however, in doing so, the event would no longer be oblivious to $\delta$, which contradicts one of our objectives in utilizing the AdaGrad learning rate.

    For every $t\in\sbrac{T}$ (including $t\in\badrounds$), it holds that
    \begin{align*}
        \E_{t-1}\!\sbrac{\lVert{\widehat{g}_t^{j} - \widehat{g}_t^{j-1}}\rVert^2\mathbbm{1}_{\Ecal_t(j)}} &= \E_{t-1}\!\sbrac{\lVert{\widehat{g}_t^{j} - \widehat{g}_t^{j-1}}\rVert^2 | \Ecal_t(j)}\underbrace{\prob(\Ecal_t(j))}_{\leq 1} \leq \frac{(1+\sqrt{2})^2 c_{\Ecal}^2 C^2\V^2}{2^{j}}\leq \frac{54\tilde{C}^2\V^2}{2^{j}}\; ,
    \end{align*}
    where in the final inequality, we utilize the constraint on $\lVert{ \widehat{g}_t^{j} - \widehat{g}_t^{j-1}\rVert}$, conditioned on the event $\Ecal_t (j)$ (cf. \Cref{eq:event_E_option2}). However, considering $t\in\goodrounds$ (static rounds), we employ a more careful analysis. Specifically, \Cref{cor:mse_bound_mfm} implies that
    \begin{align*}
        \E_{t-1}\!\sbrac{\lVert{\widehat{g}_t^{j} - \widehat{g}_t^{j-1}}\rVert^2\mathbbm{1}_{\Ecal_t(j)}} \leq \E_{t-1}\lVert{\widehat{g}_t^{j} - \widehat{g}_t^{j-1}}\rVert^2 &\leq 2\E_{t-1}\lVert{\widehat{g}_t^{j} - \nabla_t}\rVert^2 + 2\E_{t-1}\lVert{\widehat{g}_t^{j-1} - \nabla_t}\rVert^2 \\ &\leq 2\brac{\frac{\tilde{C}^2\V^2\tilde{\gamma}}{2^j} + \frac{\norm{\nabla_t}^2}{8mT}} + 2\brac{\frac{\tilde{C}^2\V^2\tilde{\gamma}}{2^{j-1}} + \frac{\norm{\nabla_t}^2}{8mT}} \\ &= 2\brac{\frac{3\tilde{C}^2\V^2\tilde{\gamma}}{2^{j}} + \frac{\norm{\nabla_t}^2}{4mT}} \\ &\leq \frac{6\tilde{C}^2\V^2\tilde{\gamma} + \frac{1}{4}\!\norm{\nabla_t}^2}{2^{j}}\; ,
    \end{align*}
    where the last inequality follows from $2^{j}\leq T$ and $m\geq 2$. We can thus conclude that
    \begin{align*}
        \E_{t-1}\sbrac{\lVert{\widehat{g}_t^{j} - \widehat{g}_t^{j-1}}\rVert^2\mathbbm{1}_{\Ecal_t(j)}} \leq \frac{1}{2^{j}}\cdot \begin{cases}
            54\tilde{C}^2\V^2, & t\in\badrounds \\
            6\tilde{C}^2\V^2\tilde{\gamma} + \frac{1}{4}\!\norm{\nabla_t}^2, &t\in\goodrounds
        \end{cases}\; .
    \end{align*}
    In addition, \Cref{cor:mse_bound_mfm} implies that $\E_{t-1}\lVert \widehat{g}_t^{0}-\nabla_t \rVert^2\leq \tilde{C}^2 \V^2\tilde{\gamma} + \frac{\norm{\nabla_t}^2}{8mT}$. Plugging these bounds back into \Cref{eq:mlmc_var_explicit_mfm} establishes the variance bound. Specifically, for $t\in\badrounds$,
    \begin{align*}
        V_t^2 \leq \E_{t-1}\norm{g_t - \nabla_t}^2 &\leq 2\brac{\tilde{C}^2 \V^2\tilde{\gamma} + \frac{\norm{\nabla_t}^2}{8mT}} + 2\sum_{j=1}^{\Jmax}{2^{j}\cdot\frac{54\tilde{C}^2\V^2}{2^{j}}} \\ &= 2\tilde{C}^2 \V^2\tilde{\gamma} + \frac{\norm{\nabla_t}^2}{4mT} + 108\tilde{C}^2 \V^2 \Jmax \\ &\leq 125\tilde{C}^2\V^2\log{T} + \frac{\!\norm{\nabla_t}^2}{4mT}\; ,
    \end{align*}
    where in the last inequality we used $\tilde{\gamma}\leq 8.5$ for $m\geq 2$, and $1\leq\Jmax\leq\log{T}$. 
    On the other hand, for $t\in\goodrounds$, it holds that 
    \begin{align*}
        V_t^2 \leq \E_{t-1}\norm{g_t - \nabla_t}^2 &\leq 2\brac{\tilde{C}^2 \V^2\tilde{\gamma} + \frac{\norm{\nabla_t}^2}{8mT}} + 2\sum_{j=1}^{\Jmax}{2^{j}\cdot\frac{6\tilde{C}^2\V^2\tilde{\gamma} + \frac{1}{4}\!\norm{\nabla_t}^2}{2^{j}}} \\ &= 2\tilde{C}^2 \V^2\tilde{\gamma} + \frac{\norm{\nabla_t}^2}{4mT} + \brac{12\tilde{C}^2 \V^2\tilde{\gamma} + \frac{1}{2}\!\norm{\nabla_t}^2} \Jmax \\ &\leq 14\tilde{C}^2\V^2\tilde{\gamma}\log{T} + \log{T}\norm{\nabla_t}^2\; ,
    \end{align*}
    where the last inequality follows from $1\leq\Jmax\leq \log{T}$ and $\frac{1}{2mT}\leq \log{T}$.

    Moving forward, we proceed to establish a bounds on the squared bias, following a similar approach as demonstrated in the proof of \Cref{lem:bias-var-mlmc-general}. For $t\in\badrounds$, we trivially bound the bias by the square root of the MSE as follows:
    \begin{equation}\label{eq:bias_in_bad_rounds}
        \norm{b_t} = \norm{\E_{t-1} g_t - \nabla_t} \leq \sqrt{\E_{t-1}\!\norm{g_t - \nabla_t}^2} \leq \sqrt{125\tilde{C}^2\V^2\log{T} + \frac{\norm{\nabla_t}^2}{4mT}} \leq \sqrt{125\log{T}}\tilde{C}\V + \frac{\norm{\nabla_t}}{2\sqrt{mT}}\; .
    \end{equation}
    For $t\in\goodrounds$, we repeat the steps from the proof of \Cref{lem:bias-var-mlmc-general}, leading to the derivation of \Cref{eq:bias-mlmc-good-rounds}, i.e., 
    \begin{equation}\label{eq:bias-mlmc-good-rounds-mfm}
        \norm{b_t} \leq \lVert\E_{t-1}[\widehat{g}_t^{\Jmax} - \nabla_t]\rVert + \norm{\E_{t-1} y_t} \leq \underbrace{\lVert\E_{t-1}[\widehat{g}_t^{\Jmax} - \nabla_t]\rVert}_{=(A)} + \underbrace{\E_{t-1}\!\norm{y_t}}_{=(B)}\; ,
    \end{equation}
    where $y_t = -\sum_{j=1}^{\Jmax}{z_t^{j}}$, and $z_t^{j}= \E_{t-1}\sbrac{\widehat{g}_t^{j} - \widehat{g}_t^{j-1} | \Ecal_t(j)^c}\prob_{t-1}(\Ecal_t(j)^{c})$. 

    \paragraph{Bounding $(A)$: } Utilizing \Cref{cor:mse_bound_mfm} and Jensen's inequality, we get
    \[
        \lVert\E_{t-1}[\widehat{g}_t^{\Jmax} - \nabla_t]\rVert \leq \sqrt{\E_{t-1}\lVert{\widehat{g}_t^{\Jmax} - \nabla_t\rVert}^2} \leq \sqrt{\frac{\tilde{C}^2\V^2\tilde{\gamma}}{2^{\Jmax}}+ \frac{\norm{\nabla_t}^2}{8mT}} \leq \tilde{C}\V\sqrt{\frac{2\tilde{\gamma}}{T}} + \frac{\norm{\nabla_t}}{\sqrt{8mT}}\; ,
    \]
    where we used $2^{\Jmax}\geq T/2$.

    \paragraph{Bounding $(B)$: } By the triangle inequality and Jensen's inequality, we have
    \[
        \norm{y_t} \leq \sum_{j=1}^{\Jmax}{\lVert{z_t^j\rVert}} \leq \sum_{j=1}^{\Jmax}{\E_{t-1}\!\Big[\lVert{\widehat{g}_t^{j} - \widehat{g}_t^{j-1}\rVert} | \Ecal_t(j)^{c}\Big]\prob_{t-1}(\Ecal_t(j)^{c})}\; .
    \]
    By item 1 of \Cref{lem:core_lemma_mfm}, for every $j=0,\ldots,\Jmax$, we have that $\lVert{\widehat{g}_t^{j} - \nabla_t}\rVert \leq \sqrt{20C^2 \V^2 + \norm{\nabla_t}^2 } \leq \sqrt{20}C\V + \norm{\nabla_t}$. Therefore, it holds that 
    \[
        \lVert{\widehat{g}_t^{j} - \widehat{g}_t^{j-1}}\rVert\leq \lVert{\widehat{g}_t^{j} - \nabla_t}\rVert + \lVert{\widehat{g}_t^{j-1} - \nabla_t}\rVert \leq 2\brac{\sqrt{20}C\V + \norm{\nabla_t}}\; . 
    \]
    In addition, by \Cref{lem:Ecal_hp_option2}, we have for every $j=0,\ldots,\Jmax$ that $\prob_{t-1}(\Ecal_t(j)^{c})\leq \frac{1}{4mT}$. Combining these bounds, we conclude that 
    \begin{align*}
        \norm{y_t} &\leq \sum_{j=1}^{\Jmax}{2\brac{\sqrt{20}C\V + \norm{\nabla_t}}\cdot\frac{1}{4mT}} = \frac{(\sqrt{20}C\V + \norm{\nabla_t})\Jmax}{2mT} \leq \frac{(\sqrt{20}C\V + \norm{\nabla_t})\log{T}}{2mT}\; .
    \end{align*}
    Substituting the bounds on $(A)$ and $(B)$ back into \Cref{eq:bias-mlmc-good-rounds-mfm} implies that for every $t\in\goodrounds$:
    \begin{align*}
        \norm{b_t} \leq \tilde{C}\V\sqrt{\frac{2\tilde{\gamma}}{T}} + \frac{\norm{\nabla_t}}{\sqrt{8mT}} + \frac{(\sqrt{20}\tilde{C}\V + \norm{\nabla_t})\log{T}}{2mT} &= \tilde{C}\V\sqrt{\frac{2\tilde{\gamma}}{T}} + \frac{\sqrt{5}\tilde{C}\V\log{T}}{mT} + \frac{\norm{\nabla_t}}{2}\brac{\frac{1}{\sqrt{2mT}} + \frac{\log{T}}{mT}} \\ &\leq \tilde{C}\V\sqrt{\frac{2\tilde{\gamma}}{T}} + \frac{\sqrt{5}\tilde{C}\V\log{T}}{mT} + \frac{\norm{\nabla_t}}{2\sqrt{2m}}\brac{\frac{1}{\sqrt{T}} + \frac{\log{T}}{T}} \\ &\leq \tilde{C}\V\sqrt{\frac{2\tilde{\gamma}}{T}} + \frac{\sqrt{5}\tilde{C}\V\log{T}}{mT} + \frac{(1+\sqrt{2})\norm{\nabla_t}}{2\sqrt{2mT}}\; , 
    \end{align*}
    where the second inequality follows from $\frac{1}{m}\leq \frac{1}{\sqrt{2m}}$, which holds for every $m\geq 2$; and in the last inequality we used $\log{T}\leq \sqrt{2T}$, which holds $\forall T\geq 1$. 
\end{proof}

Similarly to the approach employed in \Cref{app:dynamic_general}, we now utilize the established bias and variance bounds to derive convergence guarantees for \Cref{alg:method-new} with \textcolor{purple}{\textbf{Option 2}}. We use the following notations, as in \Cref{app:biased-adagrad}:
\begin{align*}
    \mathcal{R}_T\coloneqq\sum_{t\in\sbrac{T}}{(f(x_t) - f(x^*))}, \enskip G_{T}^2\coloneqq\sum_{t\in\sbrac{T}}{\norm{\nablat}^2}, \enskip V_{1:T}^2\coloneqq\sum_{t\in\sbrac{T}}{V_t^2}, \enskip b_{1:T}\coloneqq\sum_{t\in\sbrac{T}}{\norm{b_t}}, \enskip S_{T}^2\coloneqq\sum_{t\in\sbrac{T}}{\norm{b_t}^2}\; .
\end{align*}

\subsection{Convex Case}\label{subapp:convex-proof-mfm}
The following theorem implies convergence in the convex case. For ease of analysis, we assume that $\nabla f(x^*)=0$, which enables using \Cref{lem:self_boundness}; this is the case when $\K$ contains the global minimizer of $f$. To alleviate this assumption, one could consider adopting a more sophisticated \emph{optimistic} approach~\cite{mohri2016accelerating}. Yet, we refrain from doing so to maintain the clarity of our presentation and analysis.

\begin{theorem}\label{thm:convex-adaptive}
Assume $f$ is convex and $x^*$ satisfies $\nabla f(x^*) = 0$. Under Assumptions \ref{assump:bounded-noise} and \ref{assump:bounded-domain}, consider \Cref{alg:method-new} with \textcolor{purple}{\textbf{Option 2}} and the AdaGrad-Norm learning rate specified in \Cref{eq:adagrad}, where $\eta_0=D/2$. For every $T\geq 1$, we have
    \begin{align*}
        \E f(\widebar{x}_T) - f(x^*) &\leq \tilde{C}D\V\brac{14\sqrt{\frac{\tilde{\gamma}\log{T}}{T}} + \frac{24\abs{\badrounds}\sqrt{\log{T}}}{T} + \frac{78\sqrt{\abs{\badrounds}\log{T}}}{T} + \frac{6\log{T}}{mT} + \frac{10\sqrt{\tilde{\gamma}}}{T} + \frac{16\log{T}}{m T^{3/2}}} \hspace{0.1cm}+ \\ &\qquad \frac{392 LD^2\log{T}}{T}\; ,
    \end{align*}
    where $\widebar{x}_T \coloneqq \frac{1}{T}\sum_{t=1}^{T}{x_t}$ and $\tilde{C}\coloneqq 2\sqrt{2}C$.
\end{theorem}
\begin{proof}
    Applying \Cref{lem:convex-adagrad}, we have
    \begin{align}\label{eq:regret_bound_convex}
        \E\mathcal{R}_T &\leq D\sqrt{2\sum_{t\in\sbrac{T}}{\E V_t^2}} + 2D\sqrt{\sum_{t\in\sbrac{T}}{\E\!\norm{b_t}^2}} + D\sum_{t\in\sbrac{T}}{\E\!\norm{b_t}} + 2D\sqrt{\sum_{t\in\sbrac{T}}{\E\!\norm{\nablat}^2}} \nonumber \\ &= D\sqrt{2\E V_{1:T}^2} + D\E b_{1:T} + 2D\sqrt{\E S_{T}^2} + 2D\sqrt{\E G_T^2} \; .
    \end{align}
    \paragraph{Bounding $\E V_{1:T}^2$. } Utilizing the variance bound from \Cref{lem:bias-var-mlmc-mfm}, we get that
    \begin{align}\label{eq:bound_on_V_T}
        \E V_{1:T}^2 &= \sum_{t\in\badrounds}{\E V_t^2} + \sum_{t\in\goodrounds}{\E V_t^2} \nonumber \\ &\leq \sum_{t\in\badrounds}{\brac{125\tilde{C}^2\V^2\log{T} + \frac{\E\!\norm{\nablat}^2}{4mT}}} + \sum_{t\in\goodrounds}{\brac{14\tilde{C}^2\V^2\tilde{\gamma}\log{T} + \log{T}\E\!\norm{\nablat}^2}} \nonumber \\ &\leq 125\tilde{C}^2\V^2\abs{\badrounds}\log{T} + 14\tilde{C}^2\V^2\tilde{\gamma} T\log{T} + \log{T} \E G_T^2\; .
    \end{align}
    \paragraph{Bounding $\E b_{1:T}$. } Employing the bias bound from \Cref{lem:bias-var-mlmc-mfm}, we obtain: 
    \begin{align*}
        \E b_{1:T} &= \sum_{t\in\badrounds}{\E\!\norm{b_t}} + \sum_{t\in\goodrounds}{\E\!\norm{b_t}} \\ &\leq \sum_{t\in\badrounds}{\brac{\sqrt{125\log{T}}\tilde{C}\V + \frac{\E\!\norm{\nablat}}{2\sqrt{mT}}}} + \sum_{t\in\goodrounds}{\brac{\tilde{C}\V\sqrt{\frac{2\tilde{\gamma}}{T}} + \frac{\sqrt{5}\tilde{C}\V\log{T}}{mT} + \frac{(1+\sqrt{2})\E\!\norm{\nablat}}{2\sqrt{2mT}}}} \\ &\leq \sqrt{125\log{T}}\tilde{C}\V\abs{\badrounds} + \tilde{C}\V\sqrt{2\tilde{\gamma}T} + \frac{\sqrt{5}\tilde{C}\V\log{T}}{m} + \frac{1+\sqrt{2}}{2\sqrt{2mT}}\sum_{t\in\sbrac{T}}{\E\!\norm{\nablat}} \\ &\leq \sqrt{125\log{T}}\tilde{C}\V\abs{\badrounds} + \tilde{C}\V\sqrt{2\tilde{\gamma}T} + \frac{\sqrt{5}\tilde{C}\V\log{T}}{m} + \frac{1+\sqrt{2}}{2\sqrt{2m}}\sqrt{\E G_T^2}\; ,
    \end{align*}
    where in the second inequality we used the fact that $\frac{1 + \sqrt{2}}{\sqrt{2}}\geq 1$, and the last inequality arises from the application of the Cauchy-Schwarz inequality and Jensen's inequality, specifically $\sum_{t\in\sbrac{T}}{\E\!\norm{\nablat}}\leq \sqrt{T}\sqrt{\sum_{t\in\sbrac{T}}{\E\!\norm{\nablat}^2}}$.
    
    \paragraph{Bounding $\E S_{T}^2$. } We start by bounding $\norm{b_t}^2$ for every $t\in\sbrac{T}$. From \Cref{eq:bias_in_bad_rounds}, we have for all $t\in\badrounds$ that 
    \[
        \norm{b_t}^2 \leq 125\tilde{C}^2\V^2\log{T} + \frac{\!\norm{\nablat}^2}{4mT}\; .
    \]
    For $t\in\goodrounds$, employing \Cref{lem:bias-var-mlmc-mfm} and using $(a+b+c)^2\leq 3(a^2 + b^2 + c^2)$ gives:
    \begin{align*}
        \norm{b_t}^2 &\leq \brac{\tilde{C}\V\sqrt{\frac{2\tilde{\gamma}}{T}} + \frac{\sqrt{5}\tilde{C}\V\log{T}}{mT} + \frac{(1+\sqrt{2})\norm{\nablat}}{2\sqrt{2mT}}}^2 \leq \frac{6\tilde{C}^2\V^2\tilde{\gamma}}{T} + \frac{15\tilde{C}^2\V^2\log^2{T}}{m^2 T^2} + \frac{9\E\!\norm{\nablat}^2}{4mT}\; .
    \end{align*}
    Thus, we can bound $\E S_{T}^2$ as follows:
    \begin{align}\label{eq:bound_on_S_T}
        \E S_{T}^2 &= \sum_{t\in\badrounds}{\E\!\norm{b_t}^2} + \sum_{t\in\goodrounds}{\E\!\norm{b_t}^2} \nonumber \\ &\leq \sum_{t\in\badrounds}{\brac{125\tilde{C}^2\V^2\log{T} + \frac{\E\!\norm{\nablat}^2}{4mT}}} + \sum_{t\in\goodrounds}{\brac{\frac{6\tilde{C}^2\V^2\tilde{\gamma}}{T} + \frac{15\tilde{C}^2\V^2\log^2{T}}{m^2 T^2} + \frac{9\E\!\norm{\nablat}^2}{4mT}}} \nonumber \\ &\leq 125\tilde{C}^2\V^2\abs{\badrounds}\log{T} + 6\tilde{C}^2\V^2\tilde{\gamma} + \frac{15\tilde{C}^2\V^2\log^2{T}}{m^2 T} + \frac{9}{4mT}\E G_T^2\; .
    \end{align}
    
    Plugging these bound back into \Cref{eq:regret_bound_convex} and rearranging terms yields: 
    \begin{align*}
        \E\mathcal{R}_T &\leq D\sqrt{250\tilde{C}^2\V^2\abs{\badrounds}\log{T} + 28\tilde{C}^2\V^2\tilde{\gamma}T\log{T} + 2\log{T}\E G_T^2} \hspace{0.1cm}+ \\ &\qquad\sqrt{125\log{T}}\tilde{C}D\V\abs{\badrounds} + \tilde{C}D\V\sqrt{2\tilde{\gamma}T} + \frac{\sqrt{5}\tilde{C}D\V\log{T}}{m} + \frac{(1+\sqrt{2})D}{2\sqrt{2m}}\sqrt{\E G_T^2} \hspace{0.1cm}+ \\ &\qquad 2D\sqrt{125\tilde{C}^2\V^2\abs{\badrounds}\log{T} + 6\tilde{C}^2\V^2\tilde{\gamma} + \frac{15\tilde{C}^2\V^2\log^2{T}}{m^2 T} + \frac{9}{4mT}\E G_T^2}\hspace{0.1cm}+ \\ &\qquad 2D\sqrt{\E G_T^2} \\ &\leq (\sqrt{250} + 2\sqrt{125})\tilde{C}D\V\sqrt{\abs{\badrounds}\log{T}} + (\sqrt{28} + \sqrt{2})\tilde{C}D\V\sqrt{\tilde{\gamma}T\log{T}} + \sqrt{125}\tilde{C}D\V\abs{\badrounds}\sqrt{\log{T}}\hspace{0.1cm}+ \\ &\qquad\sqrt{5}\tilde{C}D\V\frac{\log{T}}{m} + 2\sqrt{6}\tilde{C}D\V\sqrt{\tilde{\gamma}} + 2\sqrt{15}\tilde{C}D\V\frac{\log{T}}{m\sqrt{T}} + D\brac{\sqrt{2\log{T}} + \frac{1+\sqrt{2}}{\sqrt{2m}} + \frac{3}{\sqrt{mT}} + 2}\sqrt{\E G_{T}^2} \\ &\leq \underbrace{39\tilde{C}D\V\sqrt{\abs{\badrounds}\log{T}} \!+ 7\tilde{C}D\V\sqrt{\tilde{\gamma}T\log{T}} \!+ 12\tilde{C}D\V\abs{\badrounds}\sqrt{\log{T}} \!+ 3\tilde{C}D\V\frac{\log{T}}{m} \!+ 5\tilde{C}D\V\sqrt{\tilde{\gamma}} \!+ 8\tilde{C}D\V\frac{\log{T}}{m\sqrt{T}}}_{\coloneqq B} + \\ &\qquad  7D\sqrt{2L\log{T}\E\mathcal{R}_{T}}\; ,
    \end{align*}
    where in the last inequality follows from $\sqrt{2} + \frac{1+\sqrt{2}}{\sqrt{2m}} +\frac{3}{\sqrt{mT}} + 2\leq 7$ (for $m\geq 2$), and the application of \Cref{lem:self_boundness}, which holds since we assume $\K$ contains a global minimum, i.e., $\nabla f(x^*)=0$. Employing \Cref{lem:simple_cases_lemma} with $a=\E\mathcal{R}_T$, $b=B$, $c=7D$, and $d=2L\log{T}$, we get
    \begin{align*}
        \E\mathcal{R}_T &\leq 2B + 392LD^2\log{T}\; .
    \end{align*}
    Finally, dividing by $T$ and utilizing Jensen's inequality establishes the result:
    \begin{align*}
        \E f(\widebar{x}_T) - f(x^*) \leq\! \frac{\E\mathcal{R}_T}{T} &\leq\! \tilde{C}D\V\!\brac{\!14\sqrt{\frac{\tilde{\gamma}\log{T}}{T}} + \frac{24\abs{\badrounds}\sqrt{\log{T}}}{T} + \frac{78\sqrt{\abs{\badrounds}\log{T}}}{T} + \frac{6\log{T}}{mT} + \frac{10\sqrt{\tilde{\gamma}}}{T} + \frac{16\log{T}}{m T^{3/2}}\!} +\hspace{0.1cm} \\ &\qquad \frac{392 LD^2\log{T}}{T}\; .
    \end{align*}
\end{proof}

\Cref{thm:convex-adaptive} suggests the following observation holds true.
\begin{corollary}\label{cor:convex-optimal}
    As long as  $\abs{\badrounds}\in\O(\sqrt{\tilde{\gamma}T})$, the first term dominates the convergence rate, implying an asymptotically optimal bound of $\Otilde\big(D\V\sqrt{\brac{\delta^2 + \nicefrac{1}{m}}/T}\big)$.
\end{corollary}

\subsection{Non-convex Case}\label{subapp:nonconvex-proof-mfm}
\begin{customthm}{5.2}
    Suppose \Cref{assump:bounded-noise} holds, with $f$ bounded by $M$ (i.e., $\max_{x}{\abs{f(x)}}\leq M$). Define $\zeta\coloneqq\frac{2M}{\eta_0} + \eta_0 L$, and consider \Cref{alg:method-new} with \textcolor{purple}{\textbf{Option 2}} and the AdaGrad-Norm learning rate. For every $T\geq 3$,
    \begin{align*}
        \frac{1}{T}\sum_{t=1}^{T}{\E\!\norm{\nablat}^2} &\leq 8\zeta\tilde{C}\V\brac{\frac{27\sqrt{\abs{\badrounds}\log{T}}}{T} + 8\sqrt{\frac{\tilde{\gamma}\log{T}}{T}} + \frac{6\log{T}}{m T^{3/2}}}\hspace{0.1cm}+ \\ &\qquad 4\tilde{C}^2\V^2\!\brac{\!\frac{125\abs{\badrounds}\log{T}}{T} + \frac{6\tilde{\gamma}}{T} + \frac{15\log^2{T}}{m^2 T^2}} + \frac{1024 \zeta^2\log{T}}{T}\; .
    \end{align*}
\end{customthm}
\begin{proof}
    Utilizing \Cref{lem:nonconvex-adagrad} gives:
    \begin{equation}\label{eq:regret_bound_nonconvex}
        \E G_T^2 \leq 2\zeta\brac{\sqrt{\E V_{1:T}^2} + \sqrt{2\E S_{T}^2} + \sqrt{2\E G_T^2}} + \E S_{T}^2\; .
    \end{equation}
    Employing the bounds on $\E V_{1:T}^2$ and $\E S_{T}^2$ as given in  \Cref{eq:bound_on_V_T,eq:bound_on_S_T}, respectively, we obtain
    \begin{align*}
        \E G_T^2 &\leq 2\zeta\Bigg(\sqrt{125\tilde{C}^2\V^2\abs{\badrounds}\log{T} + 14\tilde{C}^2\V^2\tilde{\gamma} T\log{T} + \log{T} \E G_T^2}\hspace{0.1cm}+ \\ &\qquad \sqrt{250\tilde{C}^2\V^2\abs{\badrounds}\log{T} + 12\tilde{C}^2\V^2\tilde{\gamma} + \frac{30\tilde{C}^2\V^2\log^2{T}}{m^2 T} + \frac{9}{2mT}\E G_T^2} + \sqrt{2\E G_T^2}\Bigg)\hspace{0.1cm}+ \\ &\qquad 125\tilde{C}^2\V^2\abs{\badrounds}\log{T} + 6\tilde{C}^2\V^2\tilde{\gamma} + \frac{15\tilde{C}^2\V^2\log^2{T}}{m^2 T} + \frac{9}{4mT}\E G_T^2 \\ &\leq 2\zeta\tilde{C}\V\brac{\brac{\sqrt{125} + \sqrt{250}}\sqrt{\abs{\badrounds}\log{T}} + 4\sqrt{\tilde{\gamma}T\log{T}} + 4\sqrt{\tilde{\gamma}} + \frac{6\log{T}}{m\sqrt{T}}}\hspace{0.1cm}+ \\ &\qquad \tilde{C}^2\V^2\brac{125\abs{\badrounds}\log{T} + 6\tilde{\gamma} + \frac{15\log^2{T}}{m^2 T}} +\hspace{0.1cm} \\&\qquad 2\zeta\brac{\sqrt{\log{T}} + \frac{3}{\sqrt{2mT}} + \sqrt{2}}\sqrt{\E G_T^2} + \frac{9}{4mT}\E G_T^2 \\ &\leq 2\zeta\tilde{C}\V\brac{27\sqrt{\abs{\badrounds}\log{T}} + 8\sqrt{\tilde{\gamma}T\log{T}} + \frac{6\log{T}}{m\sqrt{T}}} + \tilde{C}^2\V^2\brac{125\abs{\badrounds}\log{T} + 6\tilde{\gamma} + \frac{15\log^2{T}}{m^2 T}}\hspace{0.1cm}+ \\ &\qquad 8\zeta\sqrt{\log{T}\E G_T^2} + \frac{1}{2}\E G_T^2\; ,
    \end{align*}
    where in the last inequality we used $1+\frac{3}{\sqrt{2mT}} + \sqrt{2}\leq 4$ and $\frac{9}{4mT}\leq \frac{1}{2}$, both of which hold for $mT\geq 6$. Subtracting $\frac{1}{2}\E G_T^2$ and multiplying by $2$ gives:
    \begin{align*}
        \E G_T^2 &\leq \underbrace{4\zeta\tilde{C}\V\brac{27\sqrt{\abs{\badrounds}\log{T}} + 8\sqrt{\tilde{\gamma}T\log{T}} + \frac{6\log{T}}{m\sqrt{T}}} + 2\tilde{C}^2\V^2\brac{125\abs{\badrounds}\log{T} + 6\tilde{\gamma} + \frac{15\log^2{T}}{m^2 T}}}_{\coloneqq B}\hspace{0.1cm}+ \\ &\qquad 16\zeta\sqrt{\log{T} \E G_T^2}\; .
    \end{align*}
    Similarly to the proof of \Cref{lem:convex-adagrad}, we apply \Cref{lem:simple_cases_lemma} with $a = \E G_T^2$, $b=B$, $c=16\zeta$, and $d=\log{T}$ to obtain:
    \begin{equation*}
        \E G_T^2 \leq 2B + 1024\zeta^2\log{T}\; .
    \end{equation*}
    Dividing by $T$ concludes the proof,
    \begin{align*}
        \frac{1}{T}\sum_{t=1}^{T}{\E\!\norm{\nablat}^2} = \frac{\E G_T^2}{T} &\leq 8\zeta\tilde{C}\V\brac{\frac{27\sqrt{\abs{\badrounds}\log{T}}}{T} + 8\sqrt{\frac{\tilde{\gamma}\log{T}}{T}} + \frac{6\log{T}}{m T^{3/2}}}\hspace{0.1cm}+ \\ &\qquad 4\tilde{C}^2\V^2\brac{\frac{125\abs{\badrounds}\log{T}}{T} + \frac{6\tilde{\gamma}}{T} + \frac{15\log^2{T}}{m^2 T^2}} + \frac{1024 \zeta^2\log{T}}{T}\; .
    \end{align*}
\end{proof}

The above convergence bounds implies the subsequent result, mirroring \Cref{cor:nonconvex}.
\begin{corollary}\label{cor:nonconvex-optimal}
    \Cref{thm:nonconvex-adaptive} establishes the following asymptotic convergence rate:
    \[
        \frac{1}{T}\sum_{t=1}^{T}{\E\norm{\nablat}^2}\in\Otilde\brac{\zeta\V\sqrt{\frac{1}{T}\brac{\delta^2 + \frac{1}{m}}} + \V^2\frac{\abs{\badrounds}}{T}}\; .
    \]
    Thus, as long as the number of bad rounds $\abs{\badrounds}$ is $\Otilde(\sqrt{(\delta^2 + \nicefrac{1}{m})T})$ (omitting the dependence on $\eta_0, L, M$, and $\V$), the established convergence rate is asymptotically optimal. 
\end{corollary}

\section{Technical Lemmata}\label{app:technical_lemmata}
In this section, we provide all technical results required for our analysis.

The following result by \citet{pinelis1994optimum} is a concentration inequality for bounded martingale difference sequence.
\begin{lemma}[\citealp{alistarh2018byzantine}, Lemma 2.4]\label{lem:concentration}
    Let $X_1,\ldots,X_T\in\reals^{d}$ be a random process satisfying $\E[X_t|X_1,\ldots,X_{t-1}]=0$ and $\norm{X_t}\leq M$ a.s. for all $t\in\sbrac{T}$. Then, with probability at least $1-p$:
    \[
        \norm{\sum_{t=1}^{T}{X_t}}\leq M\sqrt{2T\log{(2/p)}}\; .
    \]
\end{lemma}

In our convex analysis, we use the following classical result for projected SGD.

\begin{lemma}[\citealp{alistarh2018byzantine}, Fact 2.5]\label{lem:classical_psgd}
    If $x_{t+1} = \proj{\K}{x_t - \eta g_t}\coloneqq \argmin_{y\in\K}{\norm{y - (x_t - \eta g_t)}^2}$, then for every $x\in\K$, we have
    \[
        g_t^\top(x_t - x) \leq g_t^\top(x_t - x_{t+1}) - \frac{\norm{x_t - x_{t+1}}^2}{2\eta} + \frac{\norm{x_t - x}^2}{2\eta} - \frac{\norm{x_{t+1} - x}^2}{2\eta}\; .
    \]
\end{lemma}

Next is a classical result for smooth functions.
\begin{lemma}[\citealp{levy2018online}, Lemma 4.1; \citealp{attia2023sgd}, Lemma 8]\label{lem:self_boundness}
    Let $f:\reals^d\to\reals$ be an $L$-smooth function and $x^*=\argmin_{x\in\reals^d}{f(x)}$. Then,
    \[
        \norm{\nabla f(x)}^2 \leq 2L\brac{f(x) - f(x^*)}, \quad \forall x\in\reals^d\; .
    \]
\end{lemma}

We utilize the following lemma by \citet{auer2002adaptive}, commonly used in the online learning literature, when analyzing our method with the AdaGrad-Norm learning rate.
\begin{lemma}[\citealp{mcmahan2010adaptive}, Lemma 5; \citealp{levy2017online}, Lemma A.1]\label{lem:sum_sqrt_lemma}
    For any sequence $a_1,\ldots,a_n\in\mathbb{R}_{+}$,
    \begin{equation*}
        \sum_{i=1}^{n}{\frac{a_i}{\sqrt{\sum_{j=1}^{i}{a_j}}}} \leq 2\sqrt{\sum_{i=1}^{n}{a_i}}\; .
    \end{equation*}
\end{lemma}

The next two lemmas arise from fundamental probability calculations.
\begin{lemma}\label{lem:expectation_indicator}
    For any random variable $X$ and event $\Ecal$, 
    \begin{equation*}
        \E[X\cdot\mathbbm{1}_{\Ecal}] = \E[X] - \E[X|\Ecal^c]\prob(\Ecal^c)\; .
    \end{equation*}
\end{lemma}
\begin{proof}
    By the law of total expectation:
    \begin{equation*}
        \E[X\cdot\mathbbm{1}_{\Ecal}] = \E[X\cdot\mathbbm{1}_{\Ecal}|\Ecal]\prob(\Ecal) + \underbrace{\E[X\cdot\mathbbm{1}_{\Ecal}|\Ecal^c]}_{=0}\prob(\Ecal^c) = \E[X|\Ecal]\prob(\Ecal) = \E[X] - \E[X|\Ecal^c]\prob(\Ecal^c)\; .
    \end{equation*}
\end{proof}

\begin{lemma}\label{lem:3events}
    For any three events $A,B,C$ satisfying $\prob(A)>0$ and $\prob(B|A)\geq \prob(C|A)$, we have $\prob(B)\geq \prob(C) - \prob(A^c)$.
\end{lemma}
\begin{proof}
    By the law of total probability, we have
    \begin{equation}
        \prob(B) = \prob(B |A)\prob(A) + \underbrace{\prob(B | A^c)\prob(A^c)}_{\geq 0} \geq \prob(C|A)\prob(A)\; .
    \end{equation}
    Again, by the law of total probability, 
    \[
        \prob(C) = \prob(C|A)\prob(A) + \underbrace{\prob(C|A^c)}_{\leq 1}\prob(A^c) \leq \prob(C|A)\prob(A) + \prob(A^c)\; .
    \]
    Since $\prob(A)>0$, we can establish a lower bound for $\prob(C|A)$ as follows: 
    \[
        \prob(C|A) \geq \frac{\prob(C) - \prob(A^c)}{\prob(A)}\; .
    \]
    Substituting this bound back gives:
    \[
        \prob(B) \geq \frac{\prob(C) - \prob(A^c)}{\prob(A)}\cdot\prob(A) = \prob(C) - \prob(A^c)\; .
    \]
\end{proof}

Finally, we utilize the following lemmas in our analysis to establish convergence rates.
\begin{lemma}\label{lem:lr_min_of_2_lrs}
    Let $a,b\geq 0$, $c>0$, and consider $\eta = \min\cbrac{\sqrt{\nicefrac{a}{b}}, \nicefrac{1}{c}}$. Then, 
    \[
        \frac{a}{\eta} + b\eta \leq 2\sqrt{ab} + ac\; .
    \]
\end{lemma}
\begin{proof}
    Assume that $\sqrt{\nicefrac{a}{b}}\leq \nicefrac{1}{c}$. In this case we have $\frac{a}{\eta} + b\eta = 2\sqrt{ab}$. Alternatively, if $\nicefrac{1}{c}\leq\sqrt{\nicefrac{a}{b}}$, then $\frac{a}{\eta} + b\eta = ac + \nicefrac{b}{c}\leq ac + \sqrt{ab}$. Therefore, we always have
    \[
        \frac{a}{\eta} + b\eta \leq \max\cbrac{2\sqrt{ab}, ac+\sqrt{ab}} \leq 2\sqrt{ab} + ac\; .
    \]
\end{proof}

\begin{lemma}\label{lem:simple_cases_lemma}
    Let $a, b, c, d \geq 0$ with $a > 0$. If $a \leq b + c\sqrt{da}$, then $a \leq 2b + 4dc^2$
\end{lemma}
\begin{proof}
    Consider two cases. If $b\geq c\sqrt{d\cdot a}$, then we can bound,
    \[
        a\leq b + c\sqrt{d\cdot a} \leq 2b\; .
    \]
    Otherwise, $b<c\sqrt{d\cdot a}$ and we can bound,
    \[
        a\leq b + c\sqrt{d\cdot a} \leq 2c\sqrt{d\cdot a}\; .
    \]
    Dividing by $\sqrt{a}>0$, we get that $\sqrt{a}\leq 2c\sqrt{d}$, which is equivalent to $a\leq 4dc^2$. We can thus conclude that
    \[
        a\leq\max\cbrac{2b, 4dc^2}\leq 2b + 4dc^2\; .
    \]
\end{proof}
\section{Experimental Details}\label{app:experiments}
In this section, we describe the experimental setup and training details for the image classification tasks in \Cref{sec:experiments}.

\paragraph{Hardware and training times. } We run all experiments on a machine with a single NVIDIA GeForce RTX 4090 GPU. For the MNIST experiments, the runtime of each configuration is $\approx 5$ minutes and for CIFAR-10 it is $\approx 45$ minutes.

\paragraph{Architectures and training details. } We adopt the CNN architectures from \citet{allouah2023fixing}, as detailed in \Cref{tab:training_details}. 

\begin{table}[h!]
    \centering
    \caption{Training details and hyperparameters.}
    \vspace{0.1in}
\begin{tabular}{|l|cc|}
\hline
\textbf{Dataset}                            & \multicolumn{1}{c|}{MNIST}                                                                                                                 & CIFAR-10                                                                                                                                                                                                                                    \\ \hline
\textbf{Architecture} & \multicolumn{1}{c|}{\begin{tabular}[c]{@{}c@{}}Conv(20)-ReLU-MaxPool-\\ Conv(20)-ReLU-MaxPool-\\ FC(500)-ReLU-FC(10)-SoftMax\end{tabular}} & \begin{tabular}[c]{@{}c@{}}Conv(64)-ReLU-BatchNorm-\\ Conv(64)-ReLU-BatchNorm-MaxPool-Dropout(0.25)-\\ Conv(128)-ReLU-BatchNorm-\\ Conv(128)-ReLU-BatchNorm-MaxPool-Dropout(0.25)-\\ FC(128)-ReLU-Dropout(0.25)-FC(10)-Softmax\end{tabular} \\ \hline
\textbf{\# of iterations}                 & \multicolumn{1}{c|}{5000}                                                                                                                  & 8000                                                                                                                                                                                                                                        \\ \hline
\textbf{Learning rate}                      & \multicolumn{1}{c|}{$\times$10 drop after 4000 iterations}                                                                                            & $\times$10 drop after 6000 iterations                                                                                                                                                                                                                  \\ \hline
\textbf{Weight decay}                       & \multicolumn{2}{c|}{$10^{-4}$}                                                                                                                                                                                                                                                                                                                                                           \\ \hline
\textbf{Base batch size}                    & \multicolumn{1}{c|}{32}                                                                                                                    & 64                                                                                                                                                                                                                                          \\ \hline
\textbf{\# of workers ($m$)}                & \multicolumn{1}{c|}{17}                                                                                                                    & 25                                                                                                                                                                                                                                          \\ \hline
\end{tabular}\label{tab:training_details}
\end{table}

In our experiments, we use a base mini-batch of size $B$ ($32$ for the MNIST experiments and $64$ for the CIFAR-10 experiments). That is, in each iteration, each worker observes $B$ samples. This implies that for the baselines we use a fixed mini-batch of size $B$, whereas for our method, which employs the MLMC estimator, in level $J$ we use $B\cdot 2^{J}$ samples for gradient estimation. Following \citet{dorfman2022adapting}, we limit the maximal value of $J$ to be smaller than $\floor{\log{T}}$, specifically $\Jmax=7$, which we found to perform well in practice. For both MNIST and CIFAR-10 we use a learning rate schedule where the initial learning rate is reduced by a factor of $10$ for the final $1000/2000$ iterations, respectively. 

\paragraph{Byzantine Attacks. } Next, we describe the attacks that the Byzantine workers employ in our experiments. Let $\bar{g}$ denote the average of honest messages in some round. We consider the following attacks:
\begin{itemize}
    \item \textbf{Sign-flip (SF,~\citealp{allen2020byzantine}):} each Byzantine worker computes a stochastic gradient and returns its negative.
    \item \textbf{Inner-Product Manipulation (IPM,~\citealp{xie2020fall}):} all Byzantine workers return the negative, re-scaled average of the honest messages, i.e., $-\epsilon \bar{g}$. We follow \citet{karimireddy2021learning} and use $\epsilon=0.1$. 
    \item \textbf{A Little is Enough (ALIE,~\citealp{baruch2019little}):} Let $\sigma$ denote the (element-wise) standard deviation vector of honest messages. All Byzantine workers return the vector $\bar{g} - z\sigma$, where $z\in\reals$ is computed as in \citet{baruch2019little,karimireddy2021learning} (cf. Appendix G in the latter). For the MNIST experiments with $m=17$ and $\delta m=8$, we have $z\approx 1.22$.  
\end{itemize}


\subsection{MNIST Classification with Periodic Identity-Switching Strategy}\label{subapp:mnist_exp}
In \Cref{fig:mnist_sf_cwtm_full}, we show the full training curves on MNIST, corresponding to the final accuracy results presented in \Cref{fig:mnist_sf_cwtm}. That is, we show the test accuracy as a function of the number of observed samples ($\times$ batch size $\times$ \# of workers, i.e., total sample complexity) under the \textbf{Periodic} switching strategy for different values of $K$, where Byzantine workers employ the SF attack and the server uses CWTM. For the momentum and SGD methods, we used an initial learning rate of $\eta=0.01$ and for \methodName{} it is $\eta=0.05$. We observe that our method performs well across all values of $K$. However, when the switch rate $K$ is small (specifically, smaller than the effective momentum horizon of $\frac{1}{1-\beta}$), momentum SGD fails to learn something meaningful and performs similarly to a random guess. When $K$ is larger, momentum performs better, e.g., when $K=100$ momentum with parameter $\beta=0.9$ slightly outperforms our method. 
\begin{figure}[h]
    \centering
    \vspace{-0.5em}
    \includegraphics[trim={0.1cm 0.1cm 0.1cm 0.5cm},clip,width=.75\linewidth]{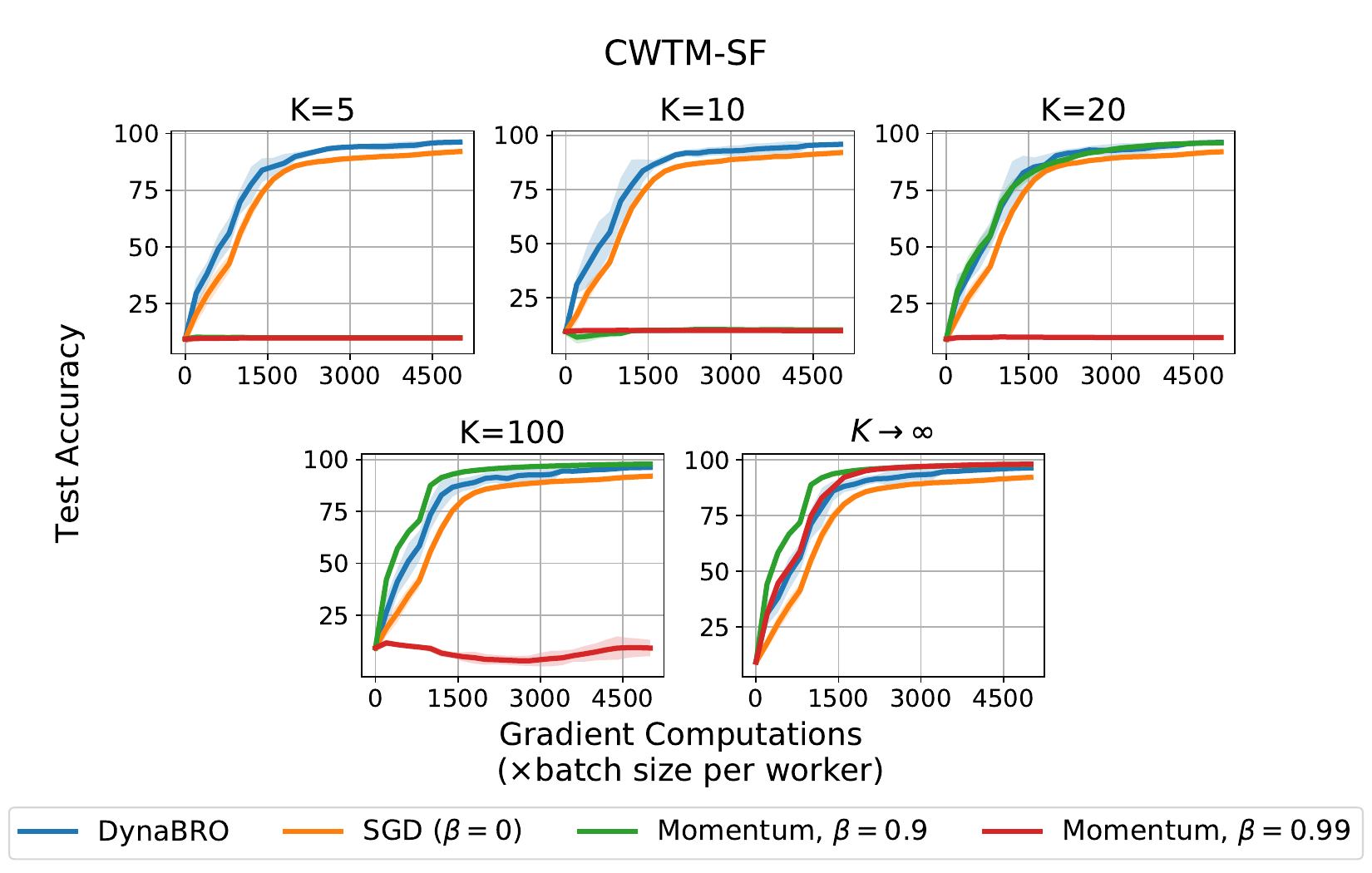}
    \vspace{-0.5em}
    \caption{Test accuracy on MNIST under the \textbf{Periodic($K$)} identity-switching strategy for different values of $K$. Byzantine workers employ SF attack and the server implements CWTM aggregation.}
    \vspace{-1em}
\label{fig:mnist_sf_cwtm_full}
\end{figure}

\begin{figure}[h]
    \centering
    \includegraphics[trim={0.1cm 0.1cm 0.1cm 0.5cm},clip,width=0.55\linewidth]{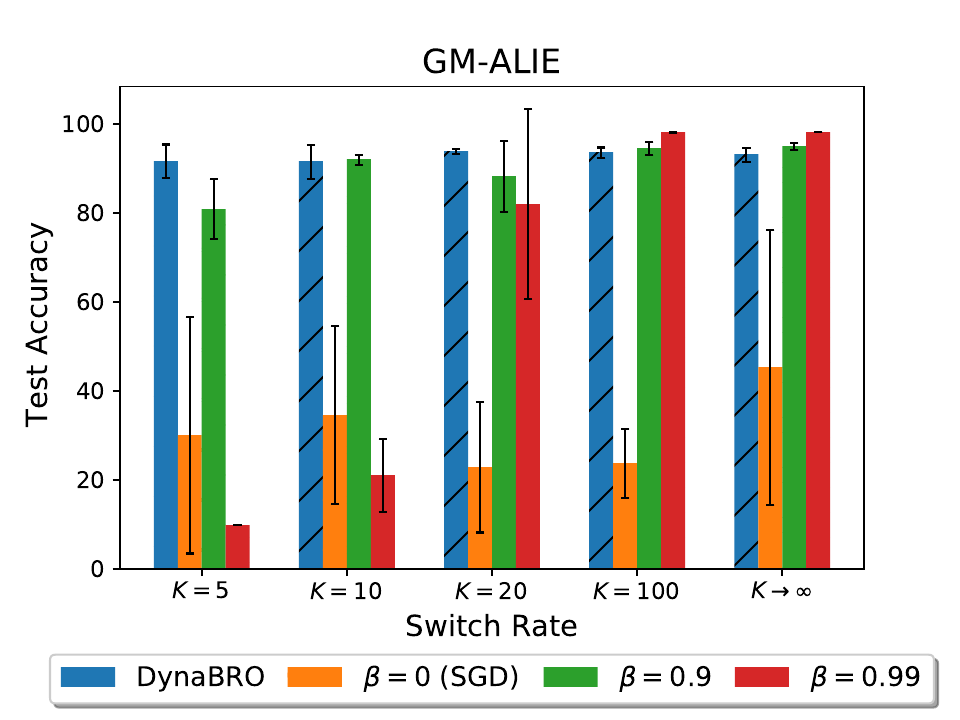}
    \caption{Final test accuracy on MNIST under the \textbf{Periodic($K$)} identity-switching strategy for different values of $K$. Byzantine workers employ ALIE attack and the server implements geometric median (GM) aggregation.}
\label{fig:mnist_alie_gm}
\end{figure}
In \Cref{fig:mnist_alie_gm,fig:mnist_alie_gm_full}, we show the test accuracy on the same MNIST configuration, but with Byzantine workers utilizing the ALIE attack~\cite{baruch2019little} and the server using the geometric median (GM,~\citealp{pillutla2022robust}) aggregator. For this configuration, all methods are trained with the same initial learning rate of $\eta=0.01$. While here momentum with parameter $0.9$ performs reasonably when $K=5$, the general trend is similar -- our method has consistent accuracy across different values of $K$ and momentum improves as $K$ increases. 

\begin{figure}[h]
    \centering
    \includegraphics[trim={0.1cm 0.1cm 0.1cm 0.5cm},clip,width=.75\linewidth]{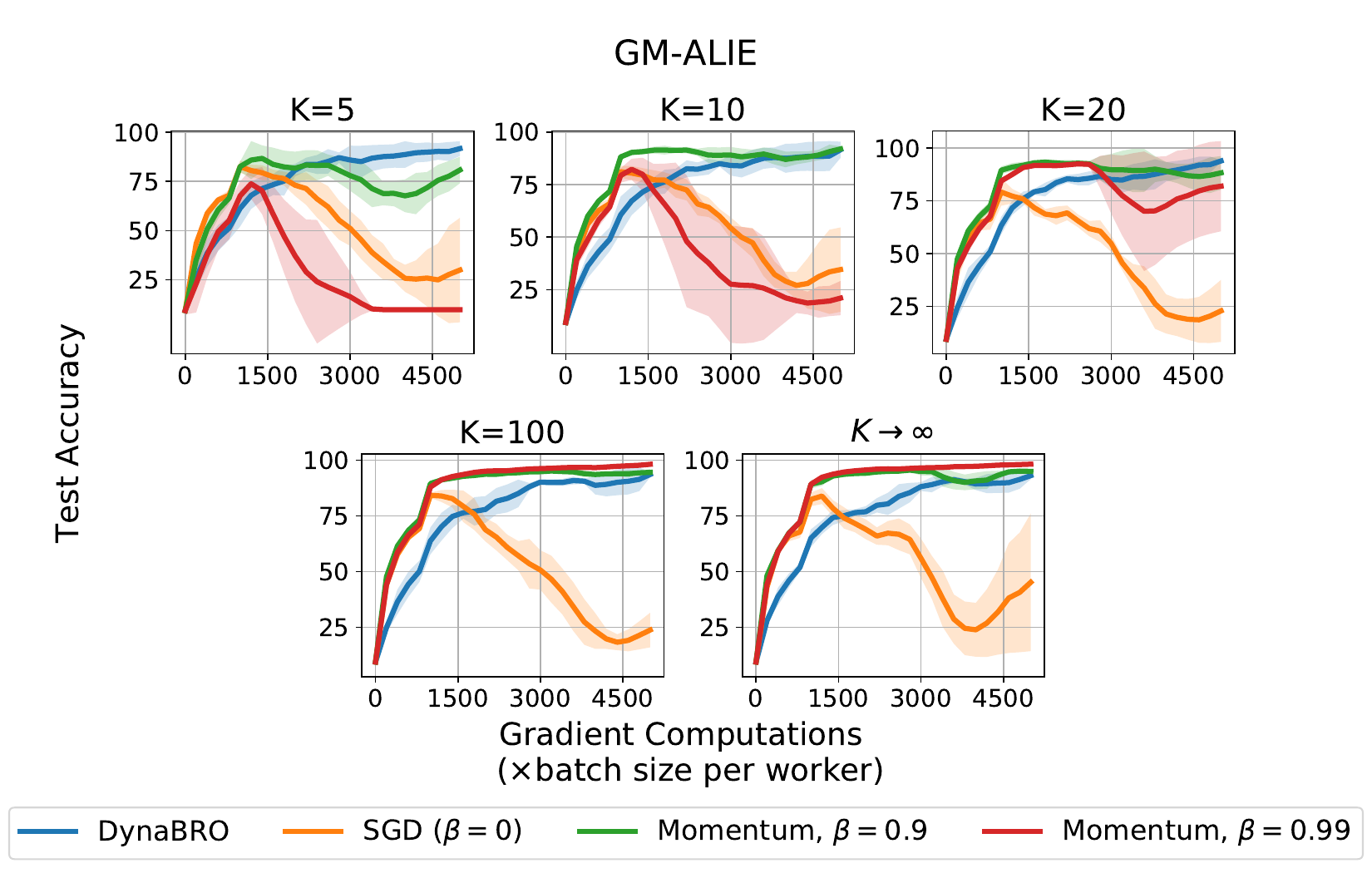}
    \vspace{-1em}
    \caption{Test accuracy curves corresponding to the final results presented in \Cref{fig:mnist_alie_gm}.}
\label{fig:mnist_alie_gm_full}
\end{figure}


\subsection{CIFAR-10 Classification with Bernoulli Identity-Switching Strategy}\label{subapp:cifar_exp}
In this section, we present an additional configuration for the CIFAR-10 task under the \textbf{Bernoulli}$(p, D, \delta_{\max})$ switching strategy, with a different maximum fraction of Byzantine workers. Instead of using $\delta_{\max}=0.72$ (maximum 18 Byzantine workers per iteration), we limit the maximum to 12 Byzantine workers, corresponding to $\delta_{\max}=0.48$. This demonstrates that similar results to those in the main text are observed when there are fewer than half Byzantine workers in all iterations. For both configurations, we used an initial learning rate of $\eta=0.01$ for momentum and SGD, and $\eta=0.05$ for \methodName{}. In \Cref{fig:cifar10_ipm_cwmed_deltamax12}, we show the results when $\delta_{\max}=0.48$, mirroring those in \Cref{fig:cifar10_ipm_cwmed}. 

\begin{figure}[h]
    \centering
    \includegraphics[trim={0.4cm 0.1cm 0.1cm 0.2cm},clip,width=.7\linewidth]{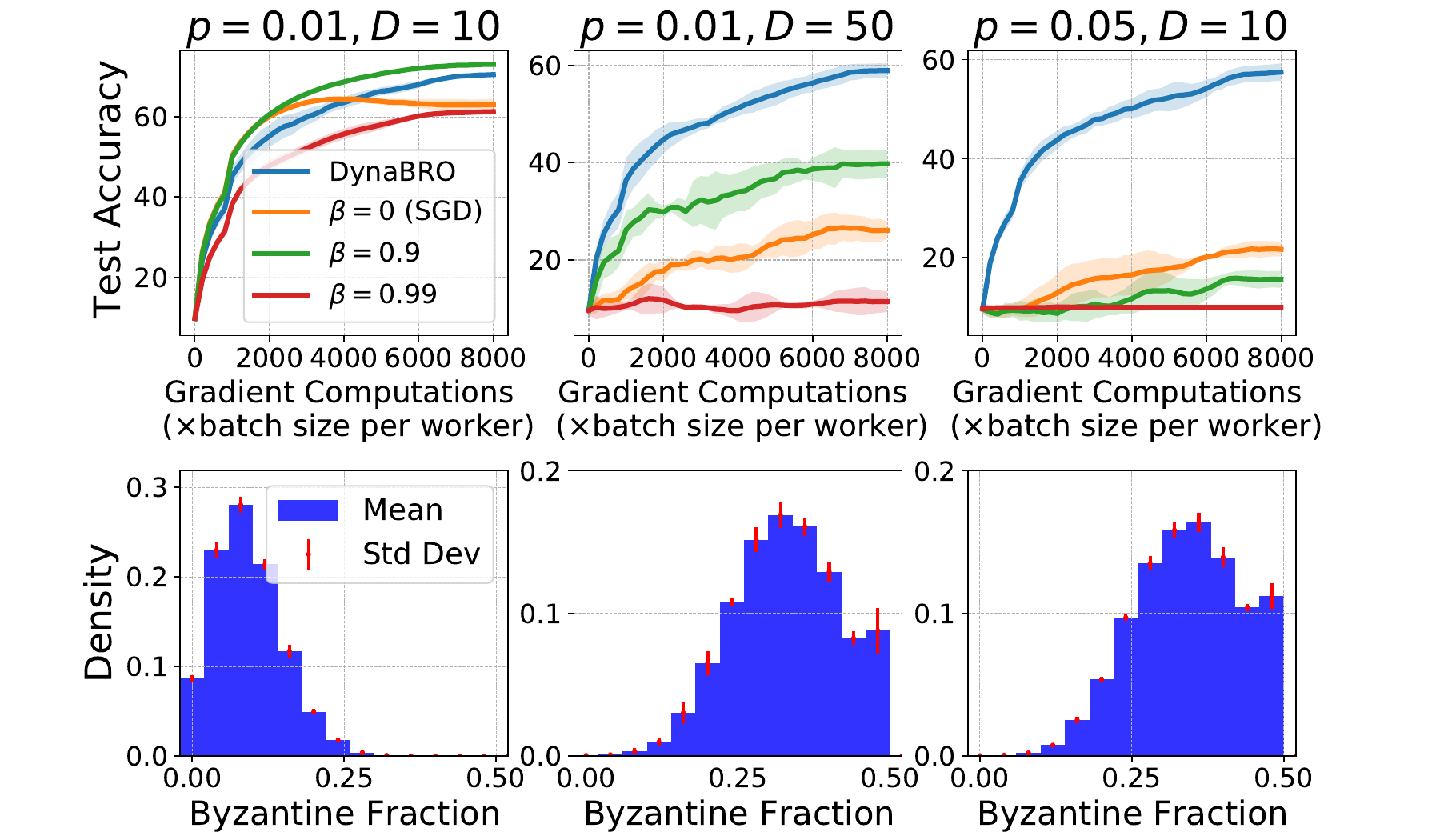}
    \caption{Test accuracy and histogram of the fraction of Byzantine workers on CIFAR-10 under the \textbf{Bernoulli($p, D, \delta_{\max}$)} strategy for different values of $p$ and $D$ with $\delta_{\max}=0.48$. Byzantine workers employ the IPM attack and the server uses CWMed.}
    \label{fig:cifar10_ipm_cwmed_deltamax12}
\end{figure}

\end{document}